\pdfoutput=1
\documentclass[11pt]{article}
\usepackage{caption}
\usepackage{subcaption}
\usepackage{jmlr2e}
\usepackage{amsmath,amssymb,amsfonts,mathrsfs,mathtools,bm,bbm, dsfont}
\usepackage{graphicx}
\usepackage{pifont,booktabs} 
\usepackage{float}      
\usepackage{adjustbox} 

\usepackage{accents,makecell} 

\usepackage[dvipsnames]{xcolor}

\usepackage{comment,enumitem}
\usepackage{algorithm,algpseudocode}
\usepackage{tikz}
\usepackage{multicol} 

\DeclareMathOperator*{\polylog}{\textrm{polylog}}

\newcommand{\beq}{\begin{equation}}
\newcommand{\eeq}{\end{equation}}
\newcommand{\beqa}{\begin{eqnarray}}
\newcommand{\eeqa}{\end{eqnarray}}
\newcommand{\beqan}{\begin{eqnarray*}}
\newcommand{\eeqan}{\end{eqnarray*}}

\newcommand{\trace}{\mbox{\rm Tr}}
\newcommand{\vnorm}[1]{\left\|#1\right\|}

\newcommand{\E}{\mathds{E} }

\newcommand{\prob}{\mathbb{P}}

\newcommand{\Nset}{\mathbb{N}}

\newcommand{\Rset}{\mathbb{R}}

\newcommand{\Acal}{{\cal A}}
\newcommand{\Bcal}{{\cal B}}

\newcommand{\Dcal}{{\cal D}}

\newcommand{\Fcal}{{\cal F}}
\newcommand{\Gcal}{{\cal G}}

\newcommand{\Lcal}{{\cal L}}

\newcommand{\Ncal}{{\cal N}}
\newcommand{\Ocal}{{\cal O}}

\newcommand{\Scal}{{\cal S}}
\newcommand{\Tcal}{{\cal T}}

\newcommand{\argmax}{\mathop{\rm argmax}}

\newcommand{\bone}{\mathbf{1}}

\newcommand{\ve}{\varepsilon}
\newcommand{\ol}[1]{\ensuremath{\overline{{#1}}}}
\newcommand{\ul}[1]{\ensuremath{\underline{{#1}}}}

\newcounter{l1}
\newcounter{l2}
\newcounter{l3}
\newcommand{\bdotlist}{\begin{list}{$\bullet$}{}}
\newcommand{\bboxlist}{\begin{list}{$\Box$}{}}
\newcommand{\bbboxlist}{\begin{list}{\raisebox{.005in}{{\tiny
$\blacksquare$ \ \ }}}{}}
\newcommand{\bdashlist}{\begin{list}{$-$}{} }
\newcommand{\blist}{\begin{list}{}{} }
\newcommand{\barablist}{\begin{list}{\arabic{l1}}{\usecounter{l1}}}
\newcommand{\balphlist}{\begin{list}{(\alph{l2})}{\usecounter{l2}}}
\newcommand{\bAlphlist}{\begin{list}{\Alph{l2}.}{\usecounter{l2}}}
\newcommand{\bdiamlist}{\begin{list}{$\diamond$}{}}
\newcommand{\bromalist}{\begin{list}{(\roman{l3})}{\usecounter{l3}}}

\newtheorem{assumption}{Assumption}
\renewcommand{\tilde}{\widetilde}
\renewcommand{\hat}{\widehat}

\newcommand{\A}{\Acal}

\newcommand{\hmax}{\lambda_{\max}^\Delta}
\newcommand{\hmin}{\lambda_{\min}^\Delta}

\newcommand{\nmin}{\nabla_g^{\min}}
\newcommand{\nmax}{\nabla_g^{\max}}
\newcommand{\co}{\mathsf{c}}

\usepackage{ebgaramond}
\usepackage[OT1]{fontenc}
\usepackage[utf8]{inputenc}

\newcommand{\proj}{{\textrm{proj}}}
\newcommand{\projA}[1]{{\textrm{proj}_{\Acal}\left(#1\right)}}

\newcommand{\probb}[1]{{\prob \left \{#1 \right\}}}

\newcommand{\Ex}[1]{{\E_{s-1} \left[#1 \right]}}

\renewcommand{\mathring}{\accentset{\circ}}

\usepackage{tkz-euclide,tikzscale}
\usetikzlibrary{arrows,circuits.ee.IEC}
\tikzset{ac source/.style={
  circuit symbol lines,
  circuit symbol size = width 2 height 2,
  shape = generic circle IEC,
  /pgf/generic circle IEC/before background={
    \pgfpathmoveto{\pgfpoint{-0.8pt}{0pt}}
    \pgfpathsine{\pgfpoint{0.4pt}{0.4pt}}
    \pgfpathcosine{\pgfpoint{0.4pt}{-0.4pt}}
    \pgfpathsine{\pgfpoint{0.4pt}{-0.4pt}}
    \pgfpathcosine{\pgfpoint{0.4pt}{0.4pt}}
    \pgfusepath{stroke}
  },
  transform shape
}}
\newcommand{\solidblackline}{\raisebox{2pt}{\tikz{\draw[solid,black,line width = 1.5pt](0,0) -- (6mm,0);}}}
\newcommand{\dashedblackline}{\raisebox{2pt}{\tikz{\draw[dashed,black,line width = 1.5pt](0,0) -- (6mm,0);}}}

\begin{document}

\title{Tangential Randomization in Linear Bandits (TRAiL):\\ Guaranteed Inference and Regret Bounds}
\author{\name Arda G\"{u}\c{c}l\"{u} \email aguclu2@illinois.edu \\
       \addr Department of Electrical and Computer Engineering,\\
       Coordinated Science Laboratory,\\
       University of Illinois, Urbana-Champaign,\\
       Urbana, IL 61801, USA
       \AND
       \name Subhonmesh Bose \email boses@illinois.edu \\
       \addr Department of Electrical and Computer Engineering,\\
       Coordinated Science Laboratory,\\
       University of Illinois, Urbana-Champaign,\\
       Urbana, IL 61801, USA}
\editor{N/A}

\maketitle

\begin{abstract}
We propose and analyze TRAiL (Tangential Randomization in Linear Bandits), a computationally efficient regret-optimal forced exploration algorithm for linear bandits on action sets that are sublevel sets of strongly convex functions. TRAiL estimates the governing parameter of the linear bandit problem through a standard regularized least squares and perturbs the reward-maximizing action corresponding to said point estimate along the tangent plane of the convex compact action set before projecting back to it.
Exploiting concentration results for matrix martingales, we prove that TRAiL ensures a $\Omega(\sqrt{T})$ growth in the inference quality, measured via the minimum eigenvalue of the design (regressor) matrix with high-probability over a $T$-length period. We build on this result to obtain an $\Ocal(\sqrt{T} \log(T))$ upper bound on cumulative regret with probability at least $ 1 - 1/T$ over $T$ periods, and compare TRAiL to other popular algorithms for linear bandits. Then, we characterize an $\Omega(\sqrt{T})$ minimax lower bound for any algorithm on the expected regret that covers a wide variety of action/parameter sets and noise processes. Our analysis not only expands the realm of lower-bounds in linear bandits significantly, but as a byproduct, yields a trade-off between regret and inference quality.
Specifically, we prove that any algorithm with an $\Ocal(T^\alpha)$ expected regret growth must have an $\Omega(T^{1-\alpha})$ asymptotic growth in expected inference quality. Our experiments on the $L^p$ unit ball as action sets reveal how this relation can be violated, but only in the short-run, before returning to respect the bound asymptotically. In effect, regret-minimizing algorithms must have just the right rate of inference--too fast or too slow inference will incur sub-optimal regret growth.
\end{abstract}

\section{Introduction}
\label{sec:introduction}
Linear bandits define a generalization of the multi-armed bandit problem where actions have outcomes dependent on a linear function of unknown parameters. An agent sequentially selects actions $a_t \in \Acal \subset \Rset^n$ at time $t$, seeking to maximize the cumulative reward $\sum_{t=1}^T Y_t$ over a time-horizon $T$, where
$Y_t=\theta^{\star \top} a_t+\ve_t
$. The noise $\ve_t$ is zero-mean. The reward-defining parameter $\theta^\star$ is unknown, but assumed to lie in  $\Theta \subset \Rset^n$. A reward-maximizing action at each time is given  by
\begin{align}
    a^\star(\theta^\star) \in \underset{a \in \Acal}\argmax \; \theta^{\star\top} a, \label{linear_bandits_optimization}
\end{align}
for which the best achievable expected reward is $ \left(\theta^{\star}\right)^{\top} a^{\star}\left(\theta^{\star}\right)$. Selecting another action $a \in \Acal$ leads to an expected lost opportunity cost or regret. 
Thus, the cumulative expected regret up to time $T$ is
\begin{align}
    \mathscr{R}_{\theta^\star}(T)  
    := \sum_{t=1}^T r_{\theta^\star}(a_t) 
    := \sum_{t=1}^T \left( {\theta^{\star\top} a^{\star}(\theta^\star)- \theta^{\star\top} a_t} \right).
    \label{eq:regret.def}
\end{align}
The goal of the linear bandit problem is to find a causal action selection policy $\pi$ that maps the sequence of actions and observed rewards up to time $t$ to an action $a_t \in \Acal$ in a way that minimizes the expected cumulative regret, $\mathscr{R}_{\theta^\star}(T)$. 
Linear bandits and their analyses have found applications in recommendation systems in \cite{recomendation_systems}, optimal network routing  in \cite{network_routing}, personalized healthcare in \cite{ph}, and online advertising in \cite{online_advertising}. 

A reasonable goal is to design control policies that guarantee sublinear regret, i.e., $\mathscr{R}_{\theta^\star}(T)/ T \to 0$ as $T\to \infty$. A simple design approach to that end comprises estimation of the parameter $\theta^\star$ from a sequence of observations $Y_1, \ldots, Y_{t-1}$ via a regularized least squares (RLS) and a subsequent selection of a reward-maximizing action corresponding to that estimate, i.e., to solve 
\begin{align}
    \hat{\theta}_{t + 1}:=\min_{\theta}\left[ \sum_{s=1}^{t}\left(\theta^\top a_s-Y_s\right)^2+\lambda\|\theta\|^2\right] = \left(\underbrace{\lambda I + \sum_{s=1}^{t} a_s a_s^\top }_{:= V_t \in \Rset^{n \times n} }\right)^{-1} \left(\sum_{s = 1}^t a_s Y_s \right),
    \label{eq:Vt.def}
\end{align}
for a regularization parameter $\lambda >0$ and then select an action $a_{t+1} = a^{\star}(\hat{\theta}_{t +1})$. Such a point estimate-based control has been known to incur linear regret, where an incorrect estimate of $\theta^\star$ leads to suboptimal actions whose rewards remain consistent with the erroneous estimate; see \cite{incomplete_learning_example} for an example. Therefore, it is necessary to \emph{explore} actions away from those suggested by the current estimate. Too much exploration, however, can accrue large regret. The art lies in  delicately balancing exploration and exploitation. In this work, we present a forced exploration algorithm for linear bandits with convex compact action sets for which we provide high probability results for inference of the underlying parameter $\theta^\star$ and an upper bound on regret. Then, we proceed to show that the result is order-optimal (up to log factors) by proving a lower bound on both.

There are numerous learning algorithms designed to solve the linear bandit problem effectively. These include the upper confidence bound (UCB) algorithms in \cite{improved_algorithms_for_stochastic_bandits, contextual_bandits_with_ucb}, Thompson Sampling (TS) in \cite{thompson_sampling_for_contextual_bandits_with_linear_payoffs, an_information_theoretic_analysis_of_TS}, information directed sampling (IDS) in \cite{learning_to_optimize_with_IDS,IDS_for_linear_partial_monitoring}, and forced exploration in \cite{AbbasiYadkori2009ForcedExplorationBA}. The UCB algorithm operates on the principle of optimism in the face of uncertainty (OFUL), selecting actions that lead to the most favorable outcomes among those suggested by error bounds around parameter estimates obtained from past observations. 
The TS algorithm samples parameter estimates from a distribution and selects the action with the highest reward based on the sampled parameters. It updates the parameter sampling distribution after making an additional observation. 
A more recent addition to this literature is IDS in \cite{learning_to_optimize_with_IDS,IDS_for_linear_partial_monitoring}, where the agent selects an action by optimizing a quantity that depends on both regret and the information gained about the parameter of interest at each step. 
Each of these algorithms and their analyses have gaps that, in collection, has motivated our algorithm design and  analysis in this work. For example, UCB applied to linear bandits from \cite{context_bandits_ucb} requires the solution of a nonconvex optimization problem at each iteration that becomes computationally challenging in high dimensions. Vanilla TS algorithm as presented in \cite{russo2014learning} requires Bayesian updates on distributions over parameters, a task that can be computationally difficult without conjugate distributions. The recent work in \cite{linear_thompson_sampling_revisited} avoids calculation of the posterior, but as will become evident from our numerical simulations, it is far sub-optimal in performance compared to forced exploration algorithms which introduce exploratory perturbations into the action selection policy, a method that, unlike UCB and vanilla TS algorithms, does not rely on the action history, given the current best estimate of $\theta^\star$. This makes forced-exploration algorithms computationally efficient and easily scalable, advantages we also examine in our experimental results. Despite their potential,  this algorithm class has remained under-studied compared to their counterparts.

Perhaps the closest in spirit to our work is the seminal paper on forced exploration in \cite{AbbasiYadkori2009ForcedExplorationBA}, where the authors devise a random exploration strategy in $\Ocal(\sqrt{T})$ time instants over a $T$-period horizon and then implements a myopic control based on the current best estimate of the underlying parameter for the remaining $T - \Ocal(\sqrt{T})$ periods. The Forced Exploration for Linear Bandit Problems (FEL) algorithm, given in \cite{AbbasiYadkori2009ForcedExplorationBA} generates $a_t$'s from a fixed distribution that needs to be tailored to specific action sets. Even though this limits the applicability of the algorithm, FEL is computationally light, and is known to achieve a sublinear regret \emph{in expectation}. This direction has remained so understudied that even the most common high-probability guarantees on cumulative regret, common in the literature on UCB, TS, and IDS algorithms over the last decade, have not been established for such algorithms to our knowledge. Our work addresses this gap by introducing a novel algorithm, TRAiL, which differs not only in its exploration strategy from FEL, but offers an array of theoretical results that have broad implications. TRAiL, presented in Section \ref{sec:algo}, estimates the governing parameter of the linear bandit problem through a regularized least squares problem and perturbs the reward-maximizing action corresponding to the point estimate of that parameter along the \emph{tangent plane} of the convex compact action set before projecting back to it, a technique deeply rooted in our analysis of the action space geometry. We consider action sets $\Acal$ defined as sub-level sets of strongly convex functions which posses many properties as explained in Section \ref{sec:geometric_properties}. TRAiL performs exploration and exploitation simultaneously, and does not separate the two phases as FEL advocates. 
 
Beyond introducing a new algorithm, our work also presents an innovative matrix-martingale-based analysis approach to derive high-probability guarantees for inferring $\theta^\star$--a first of its kind in linear bandits to our knowledge. This result shows $\lambda_{\min}(V_T) \sim \Omega(\sqrt{T})$ with high probability, where $\lambda_{\min}$ computes the minimum eigenvalue of its argument. This \emph{regressor} or \emph{design} matrix, $V_t$, underlies the quality of the inference of $\theta^\star$, as proven by \cite{improved_algorithms_for_stochastic_bandits}. Said succinctly, $\theta^\star$ must lie in an ellipsoid around the the current regularized estimate $\hat{\theta}_t$ with high probability, which itself lies in a ball of radius $[\lambda_{\min}(V_t)]^{-1}$. Thus, a guaranteed growth on $\lambda_{\min}(V_t)$ results in a guaranteed inference of $\theta^\star$. Building on this, we then provide a high-probability guarantee on finite-time cumulative regret. While not the focus of our work, our proof structure generalizes to establish both guaranteed inference and an $\Ocal(\sqrt{T})$ upper bound on cumulative regret with high probability for the FEL algorithm as well. This result, included in Appendix \ref{sec:FEL_hp_proof}, highlights the power of our proof techniques. A separate line of work has examined the growth rate of $\lambda_{\min}(V_T)$ for algorithms such as UCB and TS, however a high-probability lower bound only exists under stringent conditions as discussed in \cite[Section 11]{rich_action_spaces} or with spherical action sets as given by \cite{log_minmax_regret}. 
On the contrary, we obtain these results for a broader class of sets with TRAiL.

In Section \ref{sec:lowerBoundRegret}, we turn to analyze theoretical limits of inference and control across \emph{all} algorithms. In turn, we establish the minimax order-optimality of TRAiL. Prior art in \cite{rusmevichientong2010linearlyparameterizedbandits} provides a foundational $\Omega(\sqrt{T})$ 
lower bound on regret for linear bandits on a unit-circle action set, assuming a multivariate normal prior on the unknown parameter $\theta^\star$ and standard normal distribution for errors. In \cite{bandit_algorithms_book}, the authors produce the same order without requiring a specific prior over $\Theta$, but retaining the assumptions on the action set and the error distribution.  \cite{regret_for_all_balls} extends the analysis in \cite{bandit_algorithms_book} to show minimax $\Omega(\sqrt{T})$ expected regret on $L^p$ balls for $p \in (1, \infty)$, by relying on specific properties of these unit balls. In contrast, we provide a lower bound that applies to rich action sets and general noise distributions that our algorithm targets, proving the minimax order-optimality of TRAiL. To our knowledge such a broad analysis is new for infinite-armed linear bandit problems. We achieve this result using a Bayesian version of Cram\'{e}r-Rao inequality, inspired by the analysis of dynamic pricing in \cite{keskin2014dynamic} and linear quadratic regulators in \cite{ziemann2021uninformative,new_lqr_result_arxiv}.

Our lower-bound analysis offers unique insights into the relationship between rates for inference of the underlying parameter and rates for regret growth. While inference and control technically define different objectives, one anticipates sufficient inference to be a fundamental prerequisite for reliable control performance, not just for linear bandits, but more broadly for adaptive control problems. Such connections between inference and control has been examined for $K$-armed contextual bandits in \cite{estimation_regret} and the adaptive control of linear quadratic regulators in \cite{ziemann2021uninformative,new_lqr_result_arxiv}. Our analysis shows a similar connection in linear bandits. Informally, we show that 
\begin{align}
\text{Inference Quality} \lesssim \text{Cumulative Regret},
\label{eq:informal.1}
\end{align}
implying that one cannot infer too fast without sacrificing cumulative regret. The second observation of our analysis is that
\begin{align}
\left( \text{Cumulative Regret}\right) \left(\text{Inference Quality}\right)\gtrsim T
\label{eq:informal.2}
\end{align}
holds on average. 
That is, one cannot achieve an $\Ocal(\sqrt{T})$ regret without inference quality of the same order, measured through the growth rate of $\E^\varrho[\lambda_{\min}(\E^\pi[V_T])] \sim \Omega(\sqrt{T})$, where the outer expectation is taken with respect to a prior over $\Theta$ and the actions are selected using policy $\pi$. These two (informal) inequalities produce a universal $\Ocal(\sqrt{T})$ minimax lower bound on cumulative regret for a wide array of linear bandit problems. As a byproduct, it also implies that inference cannot lag behind regret for minimax regret-optimal algorithms, at least asymptotically. This result aligns with observations made in \cite{rich_action_spaces} for action sets  with locally constant Hessians.

Inspired by \cite{rich_action_spaces}, we examine $L^p$ unit norm balls as action spaces in Section \ref{sec:Lp.balls} that define smooth approximations to the $L^\infty$ unit norm ball. For  $p>2$, their boundaries are characterized by variable Hessians; \cite{rich_action_spaces} provides $\Omega({T}^s)$ lower bounds on $\lambda_{\min}(\E[V_T])$ with $s\in(0, 1/2]$. They even  demonstrate an example where a regret-optimal algorithm produces $\approx T^{2/5}$ growth rate for the minimum eigenvalue of the design matrix with such action sets. We show that in the Bayesian setting, such a behavior can only be short-lived. How long such behavior lasts depends on how close the action set gets to breaking the assumptions of our analysis. Asymptotically, rates for inference quality and regret growth satisfy \eqref{eq:informal.2}, which for $\Ocal(\sqrt{T})$ regret growth forces an $\Omega(\sqrt{T})$ asymptotic growth rate on inference quality. To conclude, our work provides a unifying lens to study inference guarantees and lower bounds on cumulative regret in the asymptotic regime.

\subsection{Summary of Contributions}
 \begin{itemize}[leftmargin=*]
     \item  \emph{Forced Exploration Revisited:} We propose a novel forced exploration algorithm for the linear bandit problem that enjoys a high-probability inference guarantee on the unknown parameter, $\theta^\star$, and an order-optimal upper bound on cumulative regret.

    \item \emph{A Universal Lower-Bound:} We propose a lower bound on the expected cumulative regret for any algorithm whose analysis covers a wide variety of action sets and noise distributions.

    \item \emph{Connection between Inference and Control:} Our analysis showcases the intricate connection between inference of the underlying parameter and control performance to infinite-armed linear bandits, showing that order-optimal control cannot exist with too fast or too slow inference. 

    \item \emph{Illusion of Low-Cost Control:} For certain action sets, order-optimal control does not seem to necessitate the same quality of inference as more well-behaved action sets demand. We show that when the action sets adhere to certain regularity conditions, such behavior can only be transient and must vanish asymptotically.
 \end{itemize}

\subsection{Paper Organization}
We present the algorithm in Section \ref{sec:algo} and provide a roadmap of the main results in Section \ref{sec:main_results}. Then, we investigate properties of the action set $\Acal$ in Section \ref{sec:geometric_properties} that prove useful in our analyses throughout the paper. In Section \ref{sec:inference}, we derive the high probability result on guaranteed inference quality of the underlying parameter. The upper bound on regret is presented in Section \ref{sec:regret} and the universal lower bound on both inference and regret are derived in Section \ref{sec:lowerBoundRegret}. Performance comparison of the algorithm with other methods in the literature are demonstrated in Section \ref{sec:numerics}. Then in Section \ref{sec:Lp.balls}, we delve deeper into the connection between inference and regret and present experiments on the unit $L^p$-norm balls as action sets. The paper concludes in Section 10.
\section{The Algorithm}
\label{sec:algo}
We propose an action selection policy that plays a perturbed version of the optimal action based on the current RLS estimate. The algorithm is illustrated in Figure \ref{fig_problem_definition} and formally presented in Algorithm \ref{alg:FELP}.

At time $t$, we calculate $\hat{\theta}_t$ from past actions/observations $\{a_1,Y_1,...,a_{t-1}, Y_{t-1}\}$ using \eqref{eq:Vt.def}, and then compute the reward-maximizing action, $a^{\star}(\hat{\theta}_t)$ by solving the convex optimization problem in \eqref{linear_bandits_optimization} with $\theta^\star$ replaced by its current estimate $\hat{\theta}_t$. Since the cost is linear, this action lies on the boundary of the convex set $\Acal$. Then, we find the plane tangent to the boundary of $\Acal$ at $a^\star(\hat{\theta}_t)$. We randomly perturb away from $a^\star(\hat{\theta}_t)$ on this tangent plane, and project the perturbed action back on the boundary of $\Acal$. We adopt the convention that $\mu_t^1$ is a unit-norm vector, normal to the surface of $\Acal$ at $a^\star(\hat{\theta}_t)$. The perturbations are generated along $\{\mu_t^2,...,\mu_t^n\}$ that define an orthonormal basis for the tangent plane of $\Acal$ at $a^\star(\hat{\theta}_t)$. To generate the perturbations, we sample $\nu^i_t$ from a common distribution $\mathcal{D}_t$, sampled independently for each orthonormal direction $\mu^i_t$ and independently of the history. Then, we construct $a_t$ as
\begin{align}
    a_t := \projA{a^\star(\hat{\theta}_t) + \sum_{i=2}^n \nu^i_t \mu^i_t},
\end{align}
where $\proj$ is the projection operator.  Knowledge about the underlying parameter grows with time, and naturally, the perturbations are scaled to vanish over time.

\begin{figure}[t]
    \centering
    \includegraphics[width=0.4\textwidth]{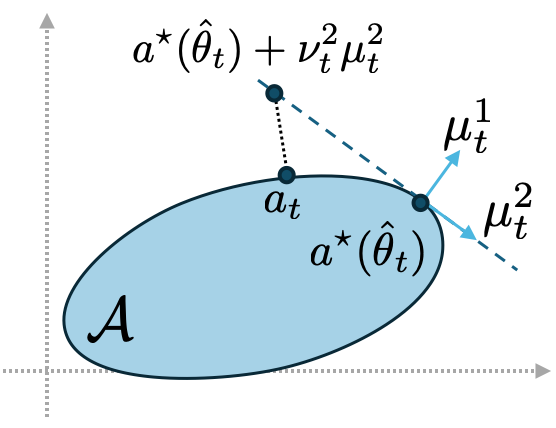}
    \caption{Visualization of action formation at time $t$ via $a_{t} = \projA{a^\star(\hat{\theta}_t) + \nu^2_t \mu^2_t}$ for a 2-dimensional convex action set $\Acal$.}
    \label{fig_problem_definition}
\end{figure}
The intuition behind TRAiL lies in the geometric properties of the action set that we catalog in Section \ref{sec:geometric_properties}. Choosing an action in a direction allows us to gather information about the reward available in that direction as it reveals how well that action aligns with the governing parameter $\theta^\star$. Indeed, if one chooses a single action $a$ for a number of rounds, the sample average of the rewards will provide an estimate of the alignment between that action and $\theta^\star$. The myopic choice $a^\star(\hat{\theta}_t)$ thus yields information about the alignment of $a^\star(\hat{\theta}_t)$ with $\theta^\star$, albeit in a noisy fashion. As long as possible governing parameters and their reward-maximizing actions are somewhat aligned, playing $a^\star(\hat{\theta}_t)$ then also provides information about how well $\hat{\theta}_t$ aligns with $\theta^\star$.
The goal of exploration is to move away from that action, but not too much, lest it accrues regret. In TRAiL, we explore \emph{tangentially} away from $a^\star(\hat{\theta}_t)$ by a random amount. As
Lemma \ref{lemma_nab} will reveal, the normal $\mu^1_t$ to the surface of $\Acal$ at $a^\star(\hat{\theta}_t)$ is $\propto \hat{\theta}_t$. Moving tangentially from  $a^\star(\hat{\theta}_t)$ loosely permits us to collect information away from $\hat{\theta}_t$. This movement, however, renders the perturbed action infeasible, forcing us to project back to $\Acal$. As one might expect, the projection must neither hinder exploration generated from the walk on the tangent plane, nor deviate too much in performance from that obtained with $a^\star(\hat{\theta}_t)$. Our assumptions on $\Acal$, essentially encode these intuitions that facilitate the derivation of a collection of geometric properties pertaining to $\Acal$ in Section \ref{sec:geometric_properties}. Being the foundation of our intuition behind the design of TRAiL, these properties play a vital role in its ensuing analysis. 

To present the result, we must impose a collection of assumptions on the action set $\Acal$, the candidate underlying parameter set $\Theta$, and $\Dcal_t$.
The following notation will prove useful in stating the assumptions. 
Let $\Bcal_{a}(m)$ denote the ball centered at $a$ with radius $m$.  Consider the filtration ${\Fcal}_{t=0}^T$, where $\Fcal_{t-1}$ as the $\sigma$-algebra generated by the past actions/observations before time $t$, with $\Fcal_0$ being the empty set. Define the notation $\E_{t-1}[\cdot] = \E[\cdot | \Fcal_{t-1}]$. Let $\Acal^\star$ be the set of candidate maximizing actions, i.e., $\Acal^\star := \left\{ \argmax_{a\in\Acal} \; \theta^{\top} a |  \theta \in \Theta \right\}$, and let $    \Acal^\star(m) := \{x \in \Bcal_{a}(m) | a \in \Acal^\star\}$. For any vector $z\neq 0$, let $\mathring{z}:=z/\vnorm{z}$.

\begin{algorithm}
\caption{TRAiL (Tangential Randomization in Linear Bandits) Algorithm}
\begin{algorithmic}
\State \textbf{Data:} $\hat{\theta}_1$, $T$, $\lambda$
\State $V_0 \gets \lambda I$; $t \gets 1 $ 
\While{$t \leq T$}
\State $a^\star(\hat{\theta}_t) \gets \arg\max_{a \in \Acal} \ \hat{\theta}_t^\top a$  
\State Find an orthonormal basis $\{\mu_t^1,...,\mu_t^n \}$ of $\mathbb{R}^{n}$ with $\mu_t^1 = \mathring{\nabla}g(a^\star(\hat{\theta}_t))$
\State $a_t \gets \projA{a^\star(\hat{\theta}_t)+ \sum_{i = 2}^n \nu_t^i \mu_t^i}$, where $\nu_t^i \sim \mathcal{D}_t$, $i \in \{2,3,...,n\}$ 
\State Observe $Y_t$
\State $V_{t} \gets V_{t-1} + a_t a_t^\top$
\State $\widehat{\theta}_{t+1} \gets V_{t}^{-1} \sum_{s=1}^{t} a_{s} Y_{s}$
\State $t \gets t + 1$
\EndWhile
\end{algorithmic}
\label{alg:FELP}
\end{algorithm}

\begin{assumption}[Parameter Space]
\label{assumption_of_parameter_set}
$\theta^\star  \in \Theta$, where $\Theta$ is bounded,
$\theta_{\min} := \min_{\theta \in \Theta}\| \theta\| > 0$, and $
    \theta_{\max} := \max_{\theta \in \Theta}\| \theta\| < \infty$.
\end{assumption}

\begin{assumption}[Noise Process]
\label{assumption_noise}
$\varepsilon_t$ is zero-mean $M$-subgaussian for some $M > 0$ that satisfies $   \E_{t-1}\left[Y_{t}\right]= \theta^{\star\top} a_t.$
\end{assumption}

\begin{assumption}[Action Set]
\label{assumption_of_action_set}
 $\Acal \subset \Rset^n$ is compact and is defined by $\Acal := \left\{x \in \mathbb{R}^{n} |  g(x) \leq 0\right\}$
for a 3- times differentiable strongly convex function $g:\Rset^n \to \Rset$ such that

\begin{enumerate}[leftmargin=*,label=(\alph*)]
    \item There exists $m'>0$, $\nabla_g^{\min} > 0$, and $\lambda_{\max}^\Delta > \lambda_{\min}^\Delta > 0$,
    \begin{gather}
    \label{assumption_of_g}   
    \nabla_g^{\max} \geq \|\nabla g(a)\| \geq \nabla_g^{\min}, 
    \;
    \lambda_{\min}^\Delta \leq \lambda_{\min}(\Delta g(a)),
    \;
    \lambda_{\max} (\Delta g(a)) \leq \lambda_{\max}^\Delta
    \end{gather}
    for all $a \in \Acal^\star(m')$ where $\nabla g$ and $\Delta g$ denote the gradient and the Hessian of $g$, respectively. 
    \item There exists $a_{\max} \geq  a_{\min} > 0$ such that 
\begin{align}
    a_{\min} \leq \|a\| \leq a_{\max} \; \text{  for all }  a \in \Acal^\star(m'). 
\end{align}
 \item There exists $\alpha_\Acal \in (0, \pi/2)$, such that 
    \begin{align}
    0 < \arccos(\mathring{a}^{',\top} \mathring{\nabla}g(a)) \leq \alpha_\Acal <\arctan\left(\frac{a_{\min}}{n a_{\max}} \right) \text{ for all } a \in \Acal^\star(m'), {a'} \in \mathcal{B}_{m'}(a).
    \end{align}
\end{enumerate}
\end{assumption}

To state our assumptions on the noise process, we need the following notation.
\begin{gather}
\varphi_g := \nabla_g^{\min} / \hmax, 
\;
m_\varphi^2 := {\min\left\{\frac{\varphi_g^2}{2}, \sqrt{ \frac{1}{4} \varphi_g^4 + (m')^2 \varphi_g ^2} - \frac{1}{2} \varphi_g^2\right\}}\label{m_varphi_definition},
\\
\gamma := \frac{\pi}{4} + \frac{1}{2} \arccos\left( \min \left \{ \frac{\sqrt{3}\varphi_g}{8 n a_{\max}} , \frac{\sqrt{3}}{2}\right \}   \right), 
\;
\overline{m}_{\varphi} := \min \left\{ m_\varphi, a_{\min} \sin\left(\frac{\pi}{4} - \frac{\gamma}{2}\right) \right\}. 
\end{gather}
Then, by definition, $\gamma \in [\pi/3, \pi/2)$. 

\begin{assumption}[Perturbation]
\label{assumption_perturbation}
For $1 \leq t \leq T$, $\xi  \sim \Dcal_t$ where $\Dcal_t$ is a $D/\sqrt{t}$- subgaussian distribution symmetric around the origin and $\xi$ satisfies 
$
\E_{t-1}[\xi^2] = D/\sqrt{t}
$
where 
\begin{align}
D \leq \min \left\{a_{\min} \cos(\alpha_{\A}), 2 a_{\min}^2 \cos(\gamma),  \frac{ \overline{m}_{\varphi}^2}{2n \log(2n)} \right\}.
\end{align}
\end{assumption}
Assumptions \ref{assumption_of_parameter_set} and \ref{assumption_noise} are standard in the linear bandit literature, where the noise process is commonly assumed to be subgaussian, and both the action and the parameter spaces are bounded.\footnote{We remark that $\vnorm{a} \geq a_{\min}$ might appear to rule out popular action set choices such as the unit circle. However, we only need $\Acal^\star(m')$ to be bounded away from the origin, and not $\Acal$.} 
The latter ensures a finite worst-case expected regret, which is essential for the analysis of many algorithms including ours. By $\varepsilon_t$ being $M$-subgaussian, we mean that its tail is dominated by that of $\Ncal(0,M)$--a normal distribution with zero mean and variance $M$.
Our lower bounds in Section \ref{sec:lowerBoundRegret} are derived  with even less restrictive noise processes. 

Assumption \ref{assumption_of_action_set} defines the nature of strong convexity of the function whose sub-level set defines $\Acal$. Similar assumptions have appeared in \cite{AbbasiYadkori2009ForcedExplorationBA, IDS_for_linear_partial_monitoring, rich_action_spaces} in the bandit setting and the online optimization literature, e.g., see \cite{curvature_offline_online_optimization,projection_free_optimization_on_uniformly_convex_sets,JMLR:v18:17-079}. The last among these assumptions ensures that the actions on the surface and the normals to the surface of $\Acal$ at said actions are aligned enough that information generated along the tangent plane from an action defines directions that are sufficiently rich for exploration away from that action. Assumption \ref{assumption_perturbation} is natural in that the perturbation that drives exploration must be designed based on how smooth the surface of $\Acal$ is, and such noise should vanish with time. It should allow larger explorations with higher $D$ for smoother surfaces, for which a larger $\varphi_g$ acts as a proxy. In the empirical performance evaluation presented in Section \ref{sec:numerics}, we disregard the bound on $D$ at Assumption \ref{assumption_perturbation} and instead tune the hyper-parameter $D$. The empirical performance attests that for practical applications, Assumption \ref{assumption_perturbation} can be bypassed. It streamlines the theoretical derivation, however.

\section{Theoretical Results}
\label{sec:main_results}

\begin{itemize}[leftmargin=*]
    \item TRAiL works by adding perturbations that vanish at a certain rate over time in a way that it balances between  explorations and exploitation of the RLS-estimate. Our first main result, Theorem \ref{corollary_inference},  provides a matrix- martingale based methodology to obtain a high probability result on TRAiL's inference. This approach builds on a modified matrix Freedman inequality from \cite{tropp}, which may be of independent interest. Specifically, we show that  for all $(\delta,T)$ such that $\log^2(1/\delta) \leq T$ and $\delta \in (0,1)$, 
\begin{align}
\probb{\bigcap_{t \sim \log^2(1/\delta)}^{T} \left \{\vnorm{\hat{\theta}_{t + 1} - \theta^\star}^2 \lesssim \frac{ \log(t/\delta)}{\sqrt{t}} \right\}} \geq 1 - \frac{2\delta}{3}. \label{inference_of_theta.begin}
\end{align}
Here, we use the notation $\chi_1(t) \lesssim \chi_2(t)$ for two positive increasing functions of $t$ to indicate that a $t$-independent constant dominates  $\chi_1(t)/\chi_2(t)$ for  sufficiently large $t$. 
\item The nature of TRAiL's tangential exploration, combined with the inference result from Theorem \ref{corollary_inference}, then allows us to deduce a high-probability bound on expected regret in Theorem
\ref{theorem_regret} , stating 
$\probb{\mathscr{R}_{\theta^{\star}}(T) \lesssim  \sqrt{T}  \log(T) } \geq 1 - \frac{1}{T}
$
for a sufficiently large $T$.

\item To obtain the lower bounds, we delve deeper into the connections between inference and regret. First, we expand the set of considered noise processes and drop the subgaussianity restriction. Then, we prove in Theorem \ref{theorem_lower_bound} that one cannot achieve an arbitrarily good inference without hindering performance. In particular, we show that for any control policy $\pi$,  $  \E_{\theta^\star}^\pi\left[\mathscr{R}_{\theta^{\star}}(T)\right]\gtrsim \lambda_{\min}(\E_{\theta^\star}^\pi[V_T])$, where the expectations are conditioned on $\theta^\star$ and the policy $ \pi$.
 
\item Our analysis shows that the inference generated, examined by $\E^{\varrho}[\lambda_{\min}(\E^\pi[V_T])]$, and the expected regret $\E^{\varrho,\pi}\left[\mathscr{R}_{\theta^{\star}}(T)\right]$ are deeply connected, where $\theta^\star$ is sampled from a prior $\varrho$ over $\Theta$ and actions are played according to policy $\pi$. We show that if an algorithm produces $\E^{\varrho,\pi}\left[\mathscr{R}_{\theta^{\star}}(T)\right] \lesssim \Ocal(T^\alpha)$, it must be accompanied with $\E^{\varrho}[\lambda_{\min}(\E^\pi[V_T])] \gtrsim \Omega(T^{1- \alpha})$ asymptotically.

\item By utilizing the last result, we provide the lower bound, $\sup_{\theta^\star \in \Theta}\E_{\theta^\star}^\pi\left[\mathscr{R}_{\theta^{\star}}(T)\right] \gtrsim \sqrt{T/\mathscr{I}_{\ve}}$, where $\mathscr{I}_{\ve}$ captures a Fisher information-type quantity associated with the noise process. This result captures the well-known $\Omega(\sqrt{T})$ lower-bound on regret, but we achieve the same for general action sets and noise distributions with finite $\mathscr{I}_{\ve}$.
\end{itemize}
\section{Geometric Properties of the Action Set}
\label{sec:geometric_properties}

The key to perturbing and projecting without sacrificing exploitation of the RLS-estimate, while simultaneously gathering sufficient information to identify the underlying parameter $\theta^\star$ lies in the action set geometry. 
In this section, we identify these geometric properties that play a vital role in the ensuing analysis.
Lemma \ref{lemma_nab} establishes the  collinearity of $\nabla g(a^{\star}(\theta^{\star}))$, characterizing the direction of the normal to $\Acal$ at the reward-maximizing action. Lemma \ref{quadratic_upperbound} proves that the per-iteration regret from playing an action close to the optimum is of the same order as the squared distance between that action and the optimum. This result allows us to relate per-step regret to squared distances on the surface of $\Acal$.
In Lemma \ref{lemma_lower}, we demonstrate that for a tangential perturbation of length $m$ away from an action on the surface of $\Acal$, the length of the projection back to $\Acal$ is $\Ocal(m^2)$; this lemma ultimately translates to the order-wise insignificance of the projection operation to regret analysis for TRAiL. Lemma \ref{perturbation_upperbound} establishes the Lipschitz continuity of the reward-maximizing action map $a^{\star}$--a result that relates changes in the underlying parameter to its optimal action. Lipschitz continuity yields almost everywhere differentiability of $a^\star$. Assuming differentiability, we provide a closed-form solution to $\nabla a^\star(\theta^\star)$ in Lemma \ref{corollary_eigenvalues}.

\begin{lemma}
\label{lemma_nab}
Suppose Assumptions \ref{assumption_of_parameter_set} and \ref{assumption_of_action_set} hold. For all $\theta \in \Theta$, the reward-maximizing action $a^\star(\theta)= \argmax_{a\in\Acal} \ \theta^{\top} a$ is unique and $\nabla g(a^\star(\theta)) = \omega^\star(\theta ) \theta$ for some $ \omega^\star(\theta) \in \mathbb{R}^+$.
\end{lemma}
\begin{proof}
The Karush-Kuhn-Tucker optimality conditions for the convex optimization problem $\max_a \theta^\top a$, subject to ${g(a) \leq 0}$ yield
\begin{align}
    g(a^\star(\theta)) \leq 0,
    \;
    -\theta + \omega' \nabla g(a^\star(\theta)) = 0,
    \;
    \omega' \geq 0.
\end{align}
Recall that $\theta \neq 0$ as $0 \notin \Theta$. Then, the guaranteed existence of $\nabla_g^{\min}>0$ from Assumption \ref{assumption_of_parameter_set} implies that $\omega' > 0$. 

For some $\theta \in \Theta$, suppose $
\{a_1,a_2\} \in {\argmax}_{a \in \Acal} \  \theta^{\top} a$. The above analysis guarantees the existence of $\omega'_1 > 0$ such that $\nabla g(a_1) = \theta/\omega'_1$ and
\begin{align}
    \nabla g(a_1)^\top (a_2 - a_1) = \frac{1}{\omega'_1} \theta^\top(a_2 - a_1) = \frac{1}{\omega'_1} (\theta^\top a_2 - \theta^\top a_1) = 0.
\end{align}
Similarly, $\nabla g(a_2)^\top (a_1 - a_2) = 0$.
From the strict convexity of $g$, we then have
\begin{align}
    a_1 \geq a_2 + \nabla g(a_2)^\top (a_1 - a_2) = a_2, 
    \;
    a_2 \geq a_1 + \nabla g(a_1)^\top (a_2 - a_1) = a_1,
\end{align}
yielding $a_1 = a_2$. This completes the proof.
\end{proof}

This lemma combined with the strong convexity of $g$ gives rise to the following quadratic bounds on regret for actions on the surface of $\Acal$ that are close to the optimum.

\begin{lemma} \label{quadratic_upperbound}
Suppose Assumptions \ref{assumption_of_parameter_set} and \ref{assumption_of_action_set} hold. Let $\theta \in \Theta$, and  $x \in \Bcal_{a^\star(\theta)}(m') \cap \Acal$ such that $g(x) = 0$. Then, 
\begin{align}
    \dfrac{\theta_{\min}\hmin}{2 \nmax}\|a^\star(\theta) - x\|^2
    \leq
    \theta^\top a^\star(\theta) - \theta^\top  x \leq  \frac{\theta_{\max}}{2\varphi_g} \|a^\star(\theta) - x\|^2.
\end{align}
\end{lemma}
\begin{proof}
Taylor's expansion of $g$ around $a^\star(\theta)$ yields
\begin{align}
\underbrace{g(x)-g(a^\star(\theta))}_{=0} = \nabla g(a^\star(\theta))^\top (x - a^\star(\theta)) + \dfrac{1}{2}(x-a^\star(\theta))^\top \Delta g(\xi) (x - a^\star(\theta)).
\label{eq:taylor.x.a}
\end{align}
for some $\xi$ on the line segment joining $a^\star(\theta)$ and $x$. Since $x \in B_{a^\star(\theta)}(m')$, then $\xi \in B_{a^\star(\theta)}(m')$. Using Assumption \ref{assumption_of_action_set} and $\mathring{\theta} = \mathring{\nabla}g(a^\star(\theta))$ from Lemma \ref{lemma_nab}, we 
infer from \eqref{eq:taylor.x.a} that
\begin{align}
0 \leq \mathring{\theta}^\top ( a^\star(\theta) - x) \leq \dfrac{\hmax}{2 \nmin} \|x-a^\star(\theta)\|^2.
\end{align}
The upper bound then follows from the relation 
$\theta^\top ( a^\star(\theta) - x) \leq \theta_{\max} \mathring{\theta}^\top ( a^\star(\theta) - x)$ and the definition of $\varphi_g$. The lower bound follows similarly and is omitted.
\end{proof}

Recall that $\Acal$ and $\Theta$ are compact. Thus, $\|\nabla g(a)\|$ and $\|\theta\|$ attain maximum and minimum values over $\Acal$ and $\Theta$, respectively. Since $0 \notin \Theta$, the optimal Lagrange multiplier $\omega^\star(\theta)$ from Lemma \ref{lemma_nab} then admit upper and lower bound as $\theta$ varies over $\Theta$. Said succinctly, $\omega^\star( \theta ) \in [\ul{\omega}, \ol{\omega}] \subset \Rset^+$ for all $\theta \in \Theta$, where 
\begin{align}
\underline{\omega} = 
\frac{\nmin }{\theta_{\max}}, 
\;
\overline{\omega} = \frac{\nmax}{\theta_{\min}}.
\end{align}

Our algorithm advocates moving along a vector in the tangent plane of $a^\star(\hat{\theta}_t)$ and then projecting it back on $\Acal$ at each time $t$, where $\hat{\theta}_t$ is the estimate of $\theta^\star$ at time $t$. As explained, key to our analysis therefore, is understanding the effect of moving on the tangent plane and projecting it back to the action set on metrics relevant to inference and control in the sequel. Next, we present a sequence of such results.

\begin{lemma}\label{lemma_lower} 
Suppose Assumptions \ref{assumption_of_parameter_set} and \ref{assumption_of_action_set} hold. Let $\varphi_g := \nabla_g^{\min} / \hmax$. Then, for $a \in \Acal^{\star}$, a unit vector $\mu \perp \nabla g(a)$, and  $m \leq m_\varphi$,
\begin{gather}
   \left \|\projA{a + m\mu} - a - m \mu \right \| \leq m^2/{\varphi_g}. 
   \label{eq:statement_1_in_lemma}
\end{gather}
\end{lemma}
\begin{proof} Let $x := a + m \mu - \eta \mathring{\nabla}g(a)$ with $\eta = \frac{1}{\varphi_g} m^2$. Then, $\vnorm{x-a}^2 = m^2 + \eta^2$. Taylor's expansion of $g$ around $a$ then gives
\begin{align}
\begin{aligned}
    g(x) & = g(a) + [\nabla g(a)]^\top (m \mu - \eta \mathring{\nabla}g(a)) + \frac{1}{2}(m \mu - \eta \mathring{\nabla}g(a))^\top \Delta g(\xi) (m \mu - \eta \mathring{\nabla}g(a))\\
    & =  -\eta \|\nabla g(a)\| + \frac{1}{2}(m \mu - \eta \mathring{\nabla}g(a))^\top \Delta g(\xi) (m \mu - \eta \mathring{\nabla}g(a))
\end{aligned} \label{proof_add_1}
\end{align}
for some $\xi$ that lies on the line segment between $a$ and $x$, since $g(a) = 0$ as $a \in \Acal^{\star}$. Given the choice of $m, \eta$, we have
\begin{align}
    m^2 + \eta^2 - (m')^2 = \frac{1}{\varphi_g^2}\left({m^2} + \frac{\varphi_g^2}{2} \right)^2 - \frac{\varphi_g^2}{4} - (m')^2 \leq 0.
\end{align}
Thus, both $x$ and $\xi$ lie in $\in B_{a}(m')$ with $\vnorm{x-a}^2 = m^2 + \eta^2$ and hence from \eqref{proof_add_1}, 
\begin{align}
\begin{aligned}
    g(x) 
     \leq -\eta \nmin + \frac{\hmax}{2}(m^2 + \eta^2) \leq 0,
\end{aligned}
\end{align}
which follows from \eqref{m_varphi_definition}. This implies that $x = a + m\mu - \eta \mathring{\nabla}g(a) \in \Acal$ and 
\begin{align}
    \vnorm{\projA{a + m\mu} - a - m\mu } \leq \vnorm{x - a - m\mu} = \eta.
\end{align}
This completes the proof.
\end{proof}
A key result enabled by the geometry of the action space is the local Lipschitz continuity of $a^{\star}$ that will play a key role in our regret analysis.

\begin{lemma}\label{perturbation_upperbound}

Suppose Assumptions \ref{assumption_of_parameter_set} and \ref{assumption_of_action_set} hold. Then, for $\theta_1 \in \Theta$ and $\theta_2 \in \Bcal_{\theta_1}(m_\theta) \cap \Theta$ for 
$m_\theta < \min\left\{\theta_{\min}, \dfrac{\hmin \theta_{\min}^2}{2 \nmax \theta_{\max}} m' \right\}$,
\begin{align}
    \vnorm{a^\star(\theta_1) - a^\star(\theta_2)} \leq 
    \frac{2 \nmax\theta_{\max}}{\hmin\theta^2_{\min}}  \vnorm{\theta_1 - \theta_2}. 
\end{align}
\end{lemma}
\begin{proof}
Within this proof, we simplify notation to $a_i := a^\star(\theta_i)$ for $i=1, 2$. We prove the result under the assumption that $\|a_2 - a_1\| \leq m'$ and then will design $m_\theta$ in such a way that our assumption holds for $\theta_1 \in \Theta$ and $\theta_2 \in B_{\theta_1}(m_\theta)$.
When $\theta_1 \propto \theta_2$, we have $a_2 = a_1$ and the result trivially holds. Consider henceforth the case where $\theta_1$ and $\theta_2$ are not collinear. For any $\xi$ on the line segment between $a_1$ and $a_2$, $\Delta g(\xi) \succeq \hmin I$, and hence, Taylor's expansion of $g$ around $a_1$ and Assumption \ref{assumption_of_action_set} allow us to conclude 
\begin{align}
g(a_2) \geq  g(a_1) + \nabla g(a_1)^\top(a_2 - a_1) + \dfrac{1}{2} \hmin \|a_2 - a_1\|^2.
\end{align}
Since $a_1$ and $a_2$ lie on the surface of $\A$, $g(a_1) = g(a_2) = 0$. Thus, we have
\begin{align}
 \nabla g(a_1)(a_2 - a_1) +  \dfrac{1}{2} \hmin \vnorm{a_2 - a_1}^2 \leq 0.
\end{align}
Since by our hypothesis $a_2 \in \Bcal_{a_1}(m')$, Assumption \ref{assumption_of_action_set} allows us to infer
\begin{align}
\mathring{\nabla} g(a_1)^\top (a_2 - a_1)     
+ \underbrace{\dfrac{\hmin}{2 \nmax} }_{:=\eta > 0} \vnorm{a_2 - a_1}^2 \leq 0. 
\end{align}
Recall from Lemma \ref{lemma_nab} that $\theta \propto \nabla g(a^\star(\theta))$ for all $\theta \in \Theta$, and hence, 
\begin{align} \label{first_constraint}
    \mathring{\theta}_1^\top (a_2 - a_1) + \eta \left \|a_2 - a_1 \right\|^2 \leq 0. 
\end{align}
Furthermore, convexity of $g$ and Lemma \ref{lemma_nab} give 
\begin{align} \label{second_constraint}
\nabla g(a_2)^\top (a_1 - a_2) \leq 0 \implies \mathring{\theta}_2^\top (a_1 - a_2) \leq 0.
\end{align}
From \eqref{first_constraint} and \eqref{second_constraint}, $a_2$ is feasible in the optimization problem
\begin{align}
\begin{aligned}
&\underset{a}{\text{maximize}} &&  \vnorm{a-a_1}^2,
\\
&\text{such that} && \mathring{\theta}_1^\top (a - a_1) +  \eta \vnorm{a - a_1}^2 \leq 0, 
\quad
\mathring{\theta}_2^\top(a - a_1) \geq 0.
\end{aligned}
\end{align}
Replace $a-a_1$ with $x$, associate Lagrange multipliers $\zeta_1, \zeta_2 \geq 0$ for the two constraints. Using the Karush-Kuhn-Tucker optimality conditions, we will bound $\vnorm{x^\star}$ from above
and use $\vnorm{a_2 - a_1} \leq \vnorm{x^\star}$ from the feasibility of $a_2$ to deduce the result, where $z^\star$ denotes any variable $z$ at optimality. Specifically, consider the Lagrangian function,
\begin{align}
    \Lcal(x, \zeta_1, \zeta_2) 
    := -\vnorm{x}^2 + \zeta_1(\eta \vnorm{x}^2 + \mathring{\theta}_1^\top x) - \zeta_2 \mathring{\theta}_2^\top x.
\end{align}
Setting $\partial \Lcal/\partial x$ to zero at $x^\star, \zeta_1^\star, \zeta_2^\star$, we get
\begin{align}
    2x^\star (1 - \zeta_1^\star \eta) = \zeta_1^\star \mathring{\theta}_1 - \zeta_2^\star \mathring{\theta}_2.
    \label{eq:x.grad.cond}
\end{align}
If $\zeta_1^\star = 1/\eta$, then the right-hand side of \eqref{eq:x.grad.cond} vanishes, but that requires $\mathring{\theta}_1$ and $\mathring{\theta}_2$ to be collinear, which they are not. Hence, we infer
\begin{align}
 x^\star  = \dfrac{\zeta_1^\star}{2(1 - \zeta_1^\star \eta)} \mathring{\theta}_1 -  \dfrac{\zeta_2^\star}{2(1 - \zeta_1^\star \eta)}\mathring{\theta}_2.   
\end{align}
The feasibility of $x^\star$ gives
\begin{gather}
\mathring{\theta}_1^\top x^\star +  \eta \vnorm{x^\star}^2 \leq 0 
\implies \mathring{\theta}_1^\top x^\star \leq 0
\implies \dfrac{\zeta_1^\star}{1 - \zeta_1^\star \eta} - \dfrac{\zeta_2^\star}{1 - \zeta_1^\star \eta} \mathring{\theta}_1^\top \mathring{\theta}_2 \leq 0,
\label{eq:lhs_1}
\\
-\mathring{\theta}_2^\top x^\star \leq 0
\implies -\dfrac{\zeta_1^\star}{1 - \zeta_1^\star \eta}\mathring{\theta}_1^\top \mathring{\theta}_2 + \dfrac{\zeta_2^\star}{1 - \zeta_1^\star \eta} \leq 0. \label{eq:lhs_2}
\end{gather}
We argue that $\mathring{\theta}_1^\top \mathring{\theta}_2 > 0$ for $m_\theta < \theta_{\min}$. To that end, we have
\begin{align}
    \vnorm{\theta_1} - m_\theta \leq \vnorm{\theta_2} \leq \vnorm{\theta_1} + m_\theta 
\end{align}
for $\theta_2 \in B_{\theta_1}(m_\theta)$. Then, the law of cosines yields
\begin{align}
\begin{aligned}
    \mathring{\theta}_1^\top \mathring{\theta}_2 
    &= \frac{1}{ 2 \vnorm{\theta_1}\vnorm{\theta_2}} \left( \vnorm{\theta_1}^2 + \vnorm{\theta_2}^2 - \vnorm{\theta_1 - \theta_2}^2\right), 
    \\
    &\geq 
    \frac{ \vnorm{\theta_1}^2 + (\vnorm{\theta_1}-m_\theta)^2  - m_\theta^2}{2 \vnorm{\theta_1}(\vnorm{\theta_1} + m_\theta) }
    \\
    &= \frac{ \vnorm{\theta_1} - m_\theta}{\vnorm{\theta_1} + m_\theta}
    \\
    & \geq \frac{ \theta_{\min} - m_\theta}{\theta_{\min} + m_\theta},
\end{aligned}
\end{align}
for $\theta_1 \in \Theta$ and $\theta_2 \in \Bcal_{\theta_1}(m_\theta) \cap \Theta$, proving the positivity of $\mathring{\theta}_1^\top \mathring{\theta}_2$. Next, we argue that both constraints of the optimization problem are active at $x^\star$.  
Towards that goal, assume to the contrary that $\zeta_2^\star = 0$. Then necessarily, $\zeta_1^\star = 0$ from \eqref{eq:lhs_1} and \eqref{eq:lhs_2}. In that scenario, $x^\star = 0$, implying that $\vnorm{a_1 - a_2} \leq \vnorm{x^\star} = 0$. Since $\theta_1$ and $\theta_2$ are not collinear, we gather $a_1 \neq a_2$ from Lemma \ref{lemma_nab} to arrive at a contradiction. Thus, $\zeta_2^\star > 0$. We multiply \eqref{eq:lhs_1} with $\mathring{\theta}_1^\top\mathring{\theta}_2$ and add it to \eqref{eq:lhs_2} to obtain
\begin{align}
\label{result_1}
\dfrac{\zeta_2^\star}{1 - \zeta_1^\star \eta} \left( 1 - (\mathring{\theta}_1^\top \mathring{\theta}_2 )^2 \right) \leq 0. 
\end{align} 
Again, since $\mathring{\theta}_1$ and $\mathring{\theta}_2$ are not collinear and $\zeta_2^\star >0$, the above relation implies $\zeta_1^\star > 1/\eta > 0$. The positivity of $\zeta_1^\star, \zeta_2^\star$ and the complementary slackness conditions for optimality of $x^\star$ give that both constraints of the optimization problem are active, i.e.,
\begin{align}
\mathring{\theta}_1^\top x^\star +  \eta \vnorm{x^\star}^2 = 0, 
\quad
\mathring{\theta}_2^\top x^\star = 0.\label{eq:constraint_equality}
\end{align}
Since arc-cosine defines a metric over a unit sphere, according to \cite[Equation (6)]{schubert2021triangle}, we deduce
\begin{align}
\arccos \left( \mathring{\theta}_1^\top {\mathring{x}^\star} \right) 
\leq \arccos \left( \mathring{\theta}_1^\top \mathring{\theta}_2 \right) + \arccos \left(\mathring{\theta}_2^\top {\mathring{x}^\star} \right) 
= \arccos \left( \mathring{\theta}_1^\top \mathring{\theta}_2 \right) + \dfrac{\pi}{2},
\end{align}
which together with \eqref{eq:constraint_equality} gives
\begin{align}
0 = \mathring{\theta}_1^\top \left(\dfrac{x^\star}{\|x^*\|}\right) +  \eta \vnorm{x^\star} 
& \geq \cos\left(\arccos \left( \mathring{\theta}_1^\top \mathring{\theta}_2 \right) + \dfrac{\pi}{2} \right) +  \eta \vnorm{x^\star}.
\end{align}
Thus, we deduce the following characterization of $x^\star$,
\begin{align}
\vnorm{x^*} \leq \dfrac{1}{\eta} \sin\left(\arccos \left( \mathring{\theta}_1^\top \mathring{\theta}_2 \right)  \right) = \frac{1}{\eta}\sqrt{1-(\mathring{\theta}_1^\top \mathring{\theta}_2)^2}.
\label{eq:xstar.eta}
\end{align}
Recall that $\theta_1 \in \Theta$ and $\theta_2 \in \Bcal_{\theta_1}(m_\theta) \cap \Theta$. For $m_\theta < \theta_{\min}$, 
the law of cosines yields
\begin{align}
\begin{aligned}
    {1 - (\mathring{\theta}_1^\top \mathring{\theta}_2)^2} 
    &= 1 - \frac{ \left(\vnorm{\theta_1}^2 + \vnorm{\theta_2}^2 - \vnorm{\theta_1 - \theta_2}^2\right)^2}{ 4 \vnorm{\theta_1}^2 \vnorm{\theta_2}^2} 
    \\
    &= 
    \frac{1}{ 4 \vnorm{\theta_1}^2 \vnorm{\theta_2}^2} \left( \vnorm{\theta_1 - \theta_2}^2 - (\vnorm{\theta_1} - \vnorm{\theta_2})^2 \right) 
    \left( (\vnorm{\theta_1} + \vnorm{\theta_2})^2  - \vnorm{\theta_1 - \theta_2}^2\right)
    \\
    &\leq  
    \frac{1}{ 4 \vnorm{\theta_1}^2 \vnorm{\theta_2}^2}  \vnorm{\theta_1 - \theta_2}^2  
    (\vnorm{\theta_1} + \vnorm{\theta_2})^2
    \\
    &\leq  
    \frac{\theta_{\max}^2}{  \theta_{\min}^4}  \vnorm{\theta_1 - \theta_2}^2. 
\end{aligned}
\end{align}
This bound in \eqref{eq:xstar.eta}, combined with $\vnorm{a_2 - a_1} \leq \vnorm{x^\star}$ yields the required result,
\begin{align}
    \vnorm{a_2 - a_1} 
    \leq \vnorm{x^\star} 
    \leq  {    \frac{\theta_{\max}}{  \eta \theta_{\min}^2} \vnorm{\theta_1 - \theta_2}}.
\end{align}
We choose  $m_\theta < \eta m'  \theta_{\min}^2/\theta_{\max}$ to ensure that $\vnorm{a_2 - a_1} \leq m'$ as needed for our analysis. This completes the proof.
\end{proof}

Lemmas \ref{quadratic_upperbound} and \ref{perturbation_upperbound} give rise to the expected regret of playing $a^\star(\theta)$ instead of $a^\star(\theta^{\star})$ to be in the order of $\vnorm{\theta^{\star} - \theta}^2$ when $\theta$ is in the locality of $\theta^{\star}$--an observation that enabled the analysis in \cite{AbbasiYadkori2009ForcedExplorationBA} and ours. The final result of this section gives a closed form solution to $\nabla a^\star$, crucial to our universal lower bounds derived in Section \ref{sec:lowerBoundRegret}. Recall the definition of $\omega^{\star}(\theta)$ from Lemma \ref{lemma_nab}. 
\begin{lemma}\label{corollary_eigenvalues}
 Suppose Assumptions  \ref{assumption_of_parameter_set} and \ref{assumption_of_action_set} hold. If $a^\star(\theta)$, $\omega^\star(\theta)$, and $\Delta_{\theta} := \Delta g(a^\star(\theta))$ are differentiable over $\Theta$, then
\begin{align}
\nabla a^{\star}(\theta ) = \omega^\star(\theta)\left(\Delta_{\theta}^{-1} -  \frac{\Delta_{\theta}^{-1} \theta  \theta^{\top} \Delta_{\theta}^{-1}}{\theta^{\top} \Delta_{\theta}^{-1} \theta}\right). 
\end{align}
\end{lemma}
\begin{proof} In this proof, all derivatives are taken with respect to $\theta$. 
 The stationary condition of the optimization problem \eqref{linear_bandits_optimization} yields Lemma \ref{lemma_nab} where $\theta := \omega(\theta)^{-1} \nabla g(a^\star(\theta))$, and the primal feasibility condition gives $g(a^\star(\theta)) = 0$. Differentiating the stationarity condition reveals
\begin{align}
\nabla \theta = I_{n \times n} =  \omega(\theta)^{-1} \Delta_\theta \nabla a^\star(\theta)  + \nabla g(a^\star(\theta)) \nabla \omega(\theta)^{-1}, \label{differentiated_stationarity}
\end{align}
whereas differentiating the primal feasibility condition yields \begin{align}\nabla g(a^\star(\theta))^\top \nabla a^\star(\theta) = 0. \label{differentiated_primary}
\end{align}
Throughout, $\omega(\theta) \geq \underline{\omega} > 0$ and $\Delta_{\theta}$ is invertible.
Multiplying \eqref{differentiated_stationarity}  with $\nabla g(a^\star(\theta))^\top \Delta_{\theta}^{-1}$ gives
\begin{subequations}
\begin{align}
\nabla g(a^\star(\theta))^\top  \Delta_{\theta}^{-1} & = \omega(\theta)^{-1} \nabla g(a^\star(\theta))^\top  \Delta_{\theta}^{-1} \Delta_{\theta} \nabla  a^\star(\theta)+ \|\nabla g(a^\star(\theta))\|_{\Delta_{\theta}^{-1}}^2 \nabla  \omega(\theta)^{-1}
\\ 
& = \|\nabla g(a^\star(\theta))\|_{\Delta_{\theta}^{-1}}^2 \nabla  \omega(\theta)^{-1},
\end{align}
\end{subequations}
where we have used \eqref{differentiated_primary}. Thus, Lemma \ref{lemma_nab} implies
\begin{subequations}
\begin{align}
&\omega(\theta) \theta^\top  \Delta_{\theta}^{-1} = \omega(\theta)^2 \|\theta\|_{\Delta_{\theta}^{-1}}^2 \nabla \omega(\theta)^{-1}
\notag
\\
& \implies \nabla \omega(\theta)^{-1} = \frac{\omega(\theta)^{-1} \theta^\top  \Delta_{\theta}^{-1}}{\theta^\top \Delta_{\theta}^{-1}\theta}\\
& \implies \nabla g(a^\star(\theta)) \nabla \omega(\theta)^{-1} = \frac{\theta \theta^\top \Delta_{\theta}^{-1}}{\theta^\top \Delta_{\theta}^{-1} \theta}.\label{to_be_substituted}
\end{align}
\end{subequations}
Substituting \eqref{to_be_substituted} in \eqref{differentiated_stationarity} finalizes the proof.
\end{proof}

\section{Analysis of Guaranteed Inference}
\label{sec:inference}

At the heart of adaptive control algorithms for linear bandit problems lies the question of guaranteed inference of where the underlying governing parameter $\theta^{\star}$ is. The algorithm must explore \emph{enough} in order to ensure that confidence in $\theta^\star$ grows fast enough with time, without of course, compromising reward accumulation in the process. 
The goal of the present section is to prove the following result on the quality of the RLS estimate $\hat{\theta}_t$ through time. 
\begin{theorem}
\label{corollary_inference}
Suppose Assumptions  \ref{assumption_of_parameter_set}, \ref{assumption_noise}, \ref{assumption_of_action_set}, and \ref{assumption_perturbation} hold. Define
\begin{subequations}
\begin{align}
    c_0 &:= \max\left\{\frac{1}{2} - \frac{\sqrt{3}}{4}, \frac{1}{2} - \frac{\sqrt{3}n a_{\max}}{\varphi_g}\right \},
    \label{eq:c0.def}
    \\
    \rho_t\left({\delta}\right)
    &:= M\sqrt{ n\log \left [\frac{3 + 3t a_{\max}^2/\lambda }{\delta} \right]}+\sqrt{\lambda} \theta_{\max},
    \label{eq:rhoT.def}
    \\
    F(\delta) &:= \frac{16 a_{\max}^4}{9 D^2 c_0^2 }\log\left(\frac{3n}{\delta}\right)^2  \left( \frac{768  n}{c_0} + 2 \right)^2.
    \label{eq:Fdelta.def}
\end{align}
\end{subequations}
Consider Algorithm \ref{alg:FELP} with $\lambda = \frac{8 n D^2}{\overline{m}_{\varphi}^{2} } (c_0 n + 1)$,
For all $(\delta,T)$ such that $F(\delta) \leq T$ and $\delta \in (0,1)$, 
\begin{align}
\probb{\bigcap_{t = F(\delta)}^{T} \left \{\vnorm{\hat{\theta}_{t + 1} - \theta^\star}^2 \leq \frac{ 2\rho_t(\delta)^2}{D c_0 \sqrt{t}} \right\}} \geq 1 - \frac{2\delta}{3}. \label{inference_of_theta}
\end{align}
\end{theorem}
In the rest of this section, we prove this result through a sequence of lemmas that we provide a roadmap for.
Recall that the estimate of $\theta^\star$ at time $t+1$ is calculated as $\hat{\theta}_{t+1} = V_t^{-1}(\sum_{s=1}^t a_s Y_s)$. The quality of the estimate then invariably depends on properties of the regressor matrix $V_t$. Prior literature on adaptive control of linear bandits abounds in results that characterize confidence ellipsoids around RLS estimates whose radius depends on the eigenvalues of $V_t$, particularly the rate of growth of its smallest eigenvalue, $\lambda_{\min}(V_t)$. Thus, through lemmas \ref{lemma_lower_bound_for_variation}, \ref{lemma_upper_bound_for_variation}, and \ref{lemma:probability_Ct}, we establish a probabilistic lower bound of the form $\lambda_{\min}(V_t)  \gtrsim \Omega(\sqrt{t})$. We then appeal to Lemma \ref{lemma:probability_Et} from \cite{improved_algorithms_for_stochastic_bandits} on said confidence ellipsoids to conclude a probabilistic bound on the quality of the RLS estimate required in Theorem \ref{corollary_inference}. To delve further into our proof technique, recall that $V_t = \lambda I + \sum_{s=1}^t a_s a_s^\top$. This is a stochastic matrix-valued process. In Lemma \ref{lemma_lower_bound_for_variation}, we establish a lower bound on the eigenvalues of the one-step \emph{expected} increment of the type $\lambda_{\min}(\Ex{a_s a_s^\top}) \gtrsim \Omega(1/\sqrt{s})$. Then, we upper bound the sum of the matrix ``variance'' of these one-step increments in Lemma \ref{lemma_upper_bound_for_variation} to ensure that these increments are close to their expected values. We then adapt a version of the matrix Freedman inequality, a concentration result for matrix-valued stochastic processes, from \cite{tropp} in Lemma \ref{matrix_freedman_improved} to arrive at a probabilistic bound of the form $\lambda_{\min}(V_t) \gtrsim \Omega(\sqrt{t})$ in Lemma \ref{lemma:probability_Ct}.

Before we present the formal proof, we include here the intuition behind the result through a simple exercise--we show that if we ignore the projection operation in deducing $a_{t}$ from $a^\star(\hat{\theta}_t)$, then $\lambda_{\min}(\E[V_t])$ indeed grows as $\Omega(\sqrt{t})$. To that end, notice that
\begin{align}
\begin{aligned}
\lambda_{\min}\left(\E[V_t]\right) & = \lambda_{\min}\left(\E\left[\sum_{s = 1}^t a_sa_s^\top\right] \right) + \lambda
\\
& \geq \lambda_{\min}\left(\sum_{s = 1}^t \E[ a_sa_s^\top]\right)
\\
& = \lambda_{\min}\left(\E\left[\sum_{s = 1}^t \E_{s-1}[ a_sa_s^\top]\right]\right)
\\
& \geq \E\left[\lambda_{\min}\left(\sum_{s = 1}^t \E_{s-1}[ a_sa_s^\top]\right)\right],
\end{aligned}
\end{align}
which utilizes the law of total expectation and Jensen's inequality. If we ignore the projection step, then 
\begin{align}
\E_{s-1}[a_s a_s^\top] \approx
 a^\star(\hat{\theta}_s) [a^\star(\hat{\theta}_s)]^\top +  \sum_{i = 2}^n \Ex{(\nu_s^i)^2} \mu_s^i \mu_s^{i,\top},
\end{align}
as the components of $\nu_s$ are independent. Hence, $\E_{s-1}[a_s a_s^\top]$ can be expressed as the summation of $n$ matrices, the last $n-1$ of which have orthogonal column spaces. The span of $\mu_s^2, \ldots, \mu_s^{n}$ forms an orthogonal complement to $\mu_s^1$ in $\Rset^n$. Provided $[a^\star(\hat{\theta}_s)]^\top \mu_s^1$ is sufficiently larger than $\Ex{(\nu_s^i)^2}$, we expect
\begin{align}
    \lambda_{\min}(\E_{s-1}[a_s a_s^\top]) \sim \Omega(\Ex{(\nu_s^i)^2}) = \Omega(1/\sqrt{s}),
\end{align}
where the last equality follows from the design of our perturbations.
Consequently, we expect its summation over $s=1$ to $s=t$ to exhibit an $\Omega(\sqrt{t})$ growth as we demand from $\lambda_{\min}(V_t)$. This derivation of course is premised on conflating $V_t$ with $\E[V_t]$ and intentionally overlooking the projection operation in each iteration. Nevertheless, this simple calculation unravels the intuition behind our design of the algorithm and why  estimation succeeds. The rest of the section develops the formal proof of Theorem \ref{corollary_inference}.

\begin{lemma}\label{lemma_lower_bound_for_variation}
Suppose Assumptions \ref{assumption_of_parameter_set}, \ref{assumption_of_action_set}, and \ref{assumption_perturbation} hold.  Let $c_0 := \max\left\{\frac{1}{2} - \frac{\sqrt{3}}{4}, \frac{1}{2} - \frac{\sqrt{3}n a_{\max}}{\varphi_g}\right \}$ and $c_2 :=  \frac{8 n D^2}{\overline{m}_{\varphi}^{2} }  (c_0 n + 1)$. Then,
\begin{align}
\sum_{s=1}^{t} \lambda_{\min }\left(\mathbb{E}_{s-1}\left[a_{s} a_{s}^{\top} \right]\right)  \geq Dc_0\sqrt{t} - c_2  .
\end{align}
\end{lemma}
\begin{proof}
Define the event $\Gcal_{s}$ as where $|\nu_{s}^{i}|\leq {\overline{m}_{\varphi}}/{\sqrt{n}}$ for  $i=1,\ldots,n$ and $\Gcal_{s}^{c}$ as its complement. Then, we have 
\begin{subequations}
\begin{align}
\lambda_{\operatorname{min}}\left(\E_{s-1}\left[a_{s} a_{s}^{\top}\right]\right)
& =\lambda_{\min }\left(\E_{s-1}^{\Gcal}\left[a_{s} a_{s}^{\top}\right] \probb{\Gcal_{s}|\Fcal_{s-1}}+\E_{s-1}^{\Gcal^{c}}\left[a_{s} a_{s}^{\top}\right] \probb{\mathcal{G}_{s}^{c}|\Fcal_{s-1}}\right)
\\
& \geq \lambda_{\min }\left(\E_{s-1}^{\Gcal}\left[a_{s} a_{s}^{\top}\right]\right) \probb{\Gcal_{s}|\Fcal_{s-1}}
\\
& \geq \left( \underset{z \in \Rset^n, \vnorm{z} = 1}{\min} z^\top  \E_{s-1}^\Gcal[a_s a_s^\top] z \right) \left[ 1-2n \exp\left(\frac{-\overline{m}_{\varphi}^{2} \sqrt{s}}{2 n D}\right) \right], \label{lem_single_stage_pri}
\end{align}
\end{subequations} 
where we have used the positive semi-definiteness of $\E_{s-1}^{\Gcal^\co}[a_s a_s^\top]$, union and Chernoff bounds on $\probb{\Gcal_{s}|\Fcal_{s-1}}$, and the Rayleigh-Ritz theorem. For convenience, define $c_1:= \frac{\overline{m}_{\varphi}^{2}}{2 n D}
$.
The crux of the proof then lies in showing that
\begin{align}
z^\top  \E_{s-1}^\Gcal[a_s a_s^\top] z 
\geq \frac{D c_0}{\sqrt{s}}  -  \frac{2D}{\sqrt{s}}\exp\left(-c_1 \sqrt{s} \right) 
\label{eq:lemma10.main}
\end{align}
for all unit norm vectors $z \in \Rset^n$.  We first show how \eqref{eq:lemma10.main} in \eqref{lem_single_stage_pri} proves the result. To that end, notice that $1-2n \exp\left(-c_1 \sqrt{s}\right) \geq 0$ for each $s \in \Nset$ in \eqref{lem_single_stage_pri}, owing to Assumption \ref{assumption_perturbation}, which then implies
\begin{align}
\begin{aligned}
\sum_{s=1}^{t} \lambda_{\min }\left(\mathbb{E}_{s-1}\left[a_{s} a_{s}^{\top} \right]\right) 
& \geq \sum_{s=1}^{t} \left( \frac{D c_0}{ \sqrt{s}} -  \frac{2D}{\sqrt{s}}\exp\left(-c_1 \sqrt{s} \right) \right)\left[ 1-2n \exp\left(-c_1\sqrt{s}\right) \right]
\\
& \geq  D c_0 \sqrt{t} - 2 D (c_0 n + 1) \int_{0}^{t} \frac{1}{\sqrt{s}}\exp\left(-c_1 \sqrt{s}\right)ds 
\\
& \geq  D c_0 \sqrt{t} - \frac{4 D}{c_1}{ (c_0 n + 1)}, 
\end{aligned}
\end{align}
completing the proof. In what follows, we prove the bound in \eqref{eq:lemma10.main} via the following steps.
\begin{itemize}[leftmargin=*]
    \item \emph{Step 1:} We show that for a unit vector $z \in \Rset^n$,
    \begin{align}
    z^\top \E_{s-1}^\Gcal [a_s a_s^\top ]z \geq 
    \underbrace{\frac{D}{2\sqrt{s}}}_{:=Y_1} - \underbrace{2 \left( \underset{ a \in \Scal^\Gcal_s}{\max}  |a^\top z| \right) \frac{D n}{\varphi_g \sqrt{s}} 
     }_{:=Y_2} - \frac{2D}{\sqrt{s}}\exp(-c_1 \sqrt{s}), \label{step_1_statement}
    \end{align}
where $\Scal_s^\Gcal := a_s + \sum_{i=2}^n [-\overline{m}_\varphi/\sqrt{n},\overline{m}_\varphi/\sqrt{n}]\mu_s^i$. The summation is interpreted in the of Minkowski (set) sum sense.
When $Y_2 \leq Y_1 (1-2c_0)$, \eqref{step_1_statement} implies \eqref{eq:lemma10.main}.

    \item  \emph{Step 2:} We prove that  
    \begin{align}
       Y_2 > Y_1 (1 - 2c_0) \implies 
       \underset{ \substack{a \in \Scal_s^\Gcal,\\ z:Y_2 > Y_1(1-2c_0)}}{\min}  {|\mathring{a}^\top z|} \geq \cos(\gamma). \label{lem_step_2_result}
    \end{align}
\end{itemize}

Upon completing these steps, the rest  follows from the observation that $\gamma \in [\pi/3,\pi/2)$ implies 
    \begin{align}
        \begin{aligned}
        z^\top \E_{s-1}^\Gcal[a_s a_s^\top ]z 
        & = \E_{s-1}^\Gcal\left[\vnorm{a_s}^2 \left( \mathring{a}_s^\top z\right)^2\right]
         \\
        & \geq a_{\min}^2 \left(  \underset{ \substack{a \in \Scal_s^\Gcal,\\ z:Y_2 > Y_1(1-2c_0)}}{\min}  {|\mathring{a}^\top z|} \right)^2
        \\
        & \geq a_{\min}^2  \cos^2 (\gamma)
        \\
        &\geq  \frac{D}{2\sqrt{s}}
         \\
        &\geq  \frac{Dc_0}{\sqrt{s}}, \label{lem_pri_3}
        \end{aligned}
    \end{align}
where we have used $D \leq 2 a_{\min}^2 \cos^2(\gamma)$ per Assumption \ref{assumption_perturbation} and $c_0 < 1/2$ by construction.

$\bullet$ \textit{Proof of Step 1:}
Recall that $a_s =\projA{a^\star(\hat{\theta}_s) + \sum_{i = 2}^n \nu_s^i \mu_s^i}$. Resolving $a^\star(\hat{\theta}_s)$ along the orthonormal directions $\{\mu_s^1, \ldots, \mu_s^n\}$ as $a^\star(\hat{\theta}_s): = \sum_{i = 1}^n \ell_s^i \mu_s^i$, we get 
    \begin{align}
        {a}^\nu_s := a^\star(\hat{\theta}_s) + \sum_{i = 2}^n \nu_s^i \mu_s^i =  \sum_{i = 1}^n (\ell_s^i + \nu_s^i) \mu_s^i, 
    \end{align}
with the notation $\nu_s^1 = 0$. Also, define $\beta_s=a_s - {a}^\nu_s$. Then, we have 
    \begin{align}
        \E_{s-1}^\Gcal[a_s a_s^\top] = \E_{s-1}^\Gcal[a_s^\nu a_s^{\nu,\top}] + \E_{s-1}^\Gcal[{a}^\nu_s\beta_s^\top + \beta_s {a}_s^{\nu,\top}]  + \E_{s-1}^\Gcal[\beta_s \beta_s^\top].
        \label{lemma_pri_x0}
    \end{align}
Then, we have $\E_{s-1}^\Gcal[a_s^\nu a_s^{\nu,\top}]= \left(\mu_s^1 |  \ldots | \mu_s^n \right) \E_{s-1}^\Gcal[A_s^\nu]\left(\mu_s^1 |  \ldots | \mu_s^n \right)^\top$, where $A_s^\nu \in \Rset^{n\times n}$ is such that its $ij^{\text{th}}$ element is $(l_s^i + \nu_s^i)(l_s^j + \nu_s^j)$. This implies 
    \begin{align}
    \begin{aligned}
        z^\top \E_{s-1}^\Gcal[a_s a_s^\top] z 
        & \geq \lambda_{\min}(\E_{s-1}^\Gcal[A_s^\nu] ) - |z^\top \E_{s-1}^\Gcal[a_s^\nu\beta_s^\top + \beta_s a_s^{\nu,\top}] z|
        \\
        & \geq \lambda_{\min}(\E_{s-1}^\Gcal[A_s^\nu] ) - 2\E_{s-1}^\Gcal[|(z^\top a_s^\nu) (\beta_s^\top z) | ] 
        \\
        & \geq \lambda_{\min}(\E_{s-1}^\Gcal[A_s^\nu] ) - 2 \left( \underset{ a \in \Scal^\Gcal_s}{\max} \  |a^\top z| \right) \E_{s-1}^\Gcal[\vnorm{\beta_s} ]. 
    \end{aligned}
    \label{eq:lemma_pri_step_1_1}
    \end{align}
We now bound $\lambda_{\min}(\E_{s-1}^\Gcal[A_s^\nu] )$ from below and $\E_{s-1}^\Gcal[\vnorm{\beta_s} ]$ from above to complete the proof of step 1. 

\emph{Lower-bounding $\lambda_{\min}(\E_{s-1}^\Gcal[A_s^\nu] )$:}
Assumption \ref{assumption_perturbation} gives 
    \begin{align}
    \E_{s-1}^\Gcal[A_s^\nu] := \left(\ell_s^1, \ldots, \ell_s^n\right)^\top \left(\ell_s^1, \ldots, \ell_s^n\right)  +  \text{diag}(0, \underbrace{\E_{s-1}^\Gcal[{(\nu_s^1)}^2]}_{:=v},..., \underbrace{\E_{s-1}^\Gcal[{(\nu_s^n)}^2]}_{=v} ), \label{lemma_pri_8}
    \end{align}
which is a rank-one update to a diagonal matrix. Following \cite[Section 5]{golub1973some}, the eigenvalues of $A_s^\nu$ are the roots of $\xi(\lambda)=0$, where
    \begin{align}
    \begin{aligned}
    \xi(\lambda) 
    & := -(v - \lambda)^{n-1}\lambda - \left(\sum_{i = 2}^n(\ell_s^i)^2 \right) (v - \lambda)^{n-2} \lambda + (\ell_s^1)^2 (v - \lambda)^{n-1}
    \\
    & = (v - \lambda)^{n-2} \left( \lambda^2 - \lambda ( v +  \vnorm{\ell_s}^2) +  (\ell_s^1)^2 v  \right).
    \end{aligned}
    \end{align}
    The polynomial has $n-2$ roots at $v$. The remaining two roots are
    \begin{align}
    \lambda^{\pm} := \frac{1}{2}\left\{ (v  +
 \vnorm{\ell_s}^2) \pm \sqrt{(v + \vnorm{\ell_s}^2)^2 - 4 (\ell_s^1)^2 v  }\right \}
    \end{align}
    with $\lambda^+ \geq \lambda^-$. Elementary algebra reveals
    \begin{align}
        \lambda^- \geq \frac{v}{2}
        \iff
        v \leq 4 (\ell_s^1)^2  - 2\vnorm{\ell_s}^2.
        \label{eq:v.lambda.minus}
    \end{align}
    According to Assumption \ref{assumption_of_action_set}, we have
    \begin{align}
        \ell_s^1 = a^\star(\hat{\theta}_s)^\top \mu_s^1 
        = \vnorm{\ell_s} \mathring{a^\star}(\hat{\theta}_s)^\top \mathring{\nabla}g(a^\star(\hat{\theta}_s)
        \geq \vnorm{\ell_s} \cos(\alpha_\A),
    \end{align}
    which due to Assumption \ref{assumption_perturbation}  implies
    \begin{align}
        4 {(\ell_s^1)}^2 - 2\vnorm{\ell_s}^2 
        \geq 2 \vnorm{\ell_s}^2 \cos(2\alpha_\A) 
        \geq a_{\min}^2 \cos(2\alpha_\A) 
        \geq \E[(\nu_s^i)^2] 
         \geq \E^\Gcal[(\nu_s^i)^2] = v
    \end{align}
    for all $s \in \Nset$ and $i =1,\ldots,n$. This verifies \eqref{eq:v.lambda.minus}. Given the other roots of $\xi(\lambda)$, we conclude that $\lambda_{\min}(\E^\Gcal[A_s^\nu]) \geq \frac{1}{2} \E^\Gcal[(\nu_s^i)^2]$, where the right-hand side is the variance of the perturbation truncated at $\overline{m}_\varphi/\sqrt{n}$.
    \begin{subequations}
    \begin{align}
    \E^\Gcal[(\nu_s^i)^2] 
    & = \frac{1}{\prob\{|\nu_s^i| \leq \overline{m}_\varphi/\sqrt{n}\}}\left(\E\left[(\nu_s^i)^2\right] - \E_{s-1}\left[(\nu_s^i)^2 \bone\{|\nu_s^i| > \overline{m}_\varphi/\sqrt{n} \}\right]\right)
    \\
    & \geq \frac{D}{\sqrt{s}} - \int_{ \overline{m}_{\varphi}^2/n}^{\infty}  \probb{(\nu_s^i)^2 > x} dx
    \\
    & = \frac{D}{\sqrt{s}} - 4\int_{ \overline{m}_{\varphi}/\sqrt{n}}^{\infty}  y \probb{\nu_s^i > y} dy
    \\
    & \geq  \frac{D}{\sqrt{s}} - 4\int_{ \overline{m}_{\varphi}/\sqrt{n}}^{\infty}  y  \exp\left(\frac{- y^2 \sqrt{s}}{2 D} \right)dy
    \label{eq:subG.UB}
    \\ 
    &  = \frac{D}{\sqrt{s}} -  \frac{4D}{\sqrt{s}}\exp\left(\frac{-\overline{m}^2_{\varphi} \sqrt{s}}{2Dn}\right).
    \end{align}
    \label{eq:variance.G}
    \end{subequations}
    In \eqref{eq:subG.UB}, we have used the tail bound for subgaussian random variables in \cite[Proposition 2.5.2]{vershynin2018high}. 

    \emph{Upper Bounding $\E_{s-1}^\Gcal[\vnorm{\beta_s} ]$:} From Lemma \ref{lemma_lower}, we deduce
    \begin{subequations}
    \begin{align}
        \E_{s-1}^\Gcal[\vnorm{\beta_s} ] 
        & \leq \E_{s-1}^\Gcal\left[\frac{1}{\varphi_g} \sum_{i = 2}^n (\nu_s^i)^2\right]
        \\
        & \leq \frac{1}{\varphi_g} \sum_{i = 2}^n \E_{s-1}\left[(\nu_s^i)^2\right] \label{lemma_pri_4}
        \\
        & \leq \frac{n D}{\varphi_g \sqrt{s}} \label{eq:lemma_pri_step_1_2}.
    \end{align}
\end{subequations}
The bound on $\lambda_{\min}(\E_{s-1}^\Gcal[A_s^\nu] )$ as implied by \eqref{eq:variance.G} and that on $\E_{s-1}^\Gcal[\vnorm{\beta_s} ]$ from \eqref{eq:lemma_pri_step_1_2} in \eqref{eq:lemma_pri_step_1_1} gives the desired result for step 1.

$\bullet$ \textit{Proof of Step 2:} Assume that $Y_2 > Y_1 (1 - 2c_0)$. Then, $z \in \Rset^n$ satisfies 
    \begin{align}
    \underset{ a \in \Scal^\Gcal_s}{\max} \  |a^\top z|  > \min\left\{\frac{\sqrt{3}\varphi_g}{8 n}, \frac{\sqrt{3}a_{\max}}{2}\right \}
\implies 
    \underset{a \in \Scal^\Gcal_s}{\max}  \ {|\mathring{a}^\top z|}
    > \underbrace{\min \left \{ \frac{\sqrt{3}\varphi_g}{8 n a_{\max}} , \frac{\sqrt{3}}{2}\right \}}_{:=\cos(\alpha_1)}, \label{lemma_pri_step_2_1}
    \end{align}   
where $\alpha_1 \in [\pi/6, \pi/2)$, by construction. Since $\Scal^\Gcal_s$ is closed, let the maximum in \eqref{lemma_pri_step_2_1} be reached at $a_0$. Then, $| \mathring{a}_0^\top z | \geq \cos(\alpha_1)$. Notice that
\begin{align}
    \min_{\substack{a\in \Scal^\Gcal_s, \\ |\mathring{a}_0^\top z| \geq \cos(\alpha_1)}} | \mathring{a}^\top z| 
    = \min\left\{ \min_{\substack{a\in \Scal^\Gcal_s, \\ \mathring{a}_0^\top z \geq \cos(\alpha_1)}} | \mathring{a}^\top z|, 
    \min_{\substack{a\in \Scal^\Gcal_s, \\ \mathring{a}_0^\top z \leq -\cos(\alpha_1)}} | \mathring{a}^\top z|
    \right\} 
    = \min_{\substack{a\in \Scal^\Gcal_s, \\ \mathring{a}_0^\top z \geq \cos(\alpha_1)}} | \mathring{a}^\top z|.
    \label{eq:min.two}
\end{align}
Arc-cosine defines a metric over a unit sphere, per \cite[Equation (6)]{schubert2021triangle}, and hence, 
\begin{align}
    \arccos(\mathring{a}^\top z) \leq \arccos(\mathring{a}^\top \mathring{a}_0) + \arccos(\mathring{a}_0^\top z) \leq \arccos(\mathring{a}^\top \mathring{a}_0) + \alpha_1.
    \label{eq:angle.ineq}
\end{align}
Since $\Scal^\Gcal_s \subset \Bcal_{a^\star(\hat{\theta}_s)}(\overline{m}_{\varphi})$, the angle between any two vectors from $\Scal^\Gcal_s$ is smaller than that subtended by $\Bcal_{a^\star(\hat{\theta}_s)}(\overline{m}_{\varphi})$ at the origin, which is given by $2\arcsin(\overline{m}_{\varphi}/a^\star(\hat{\theta}_s)) \leq 2\arcsin(\overline{m}_{\varphi}/a_{\min}) $.
From \eqref{eq:angle.ineq} and Assumption \ref{assumption_perturbation}, we then deduce
\begin{align}
\arccos(\mathring{a}^\top z) \leq \gamma < \frac{\pi}{2}
\end{align}
Hence, \eqref{eq:min.two} implies that $\min | \mathring{a}^\top z|$ over $a \in \Scal^\Gcal_S$ and $z$ for which $Y_2 > Y_1(1-2c_0)$ is bounded below by $\cos(\gamma)$,
(proving \eqref{lem_step_2_result}), and completing the proof of the lemma.
\end{proof}

To achieve a similar result for $V_t$ with high  probability, we require $a_t a_t^\top$ to be sufficiently close to $\E_{t-1} [a_t a_t^\top]$, which is captured in the bounded per-step quadratic variation that we derive next. 

\begin{lemma}\label{lemma_upper_bound_for_variation}
    Suppose Assumptions \ref{assumption_of_action_set} and \ref{assumption_perturbation} hold. Then, 
\begin{align}
    \left\|\sum_{s=1}^{t} \Ex{\left(\Ex{a_{s} a_{s}^{\top}}-a_{s} a_{s}^{\top}\right)^2 }\right\| \leq 128 a_{\max}^2 n D \sqrt{t}.
\end{align}
\end{lemma}

\begin{proof}
Define $e_s := a_s - x_s$, where $x_s := a^\star(\hat{\theta}_s)$. Then, we have 
\begin{subequations}
\begin{align}
&\E_{s-1}\left[\left(a_{s} a_{s}^{\top}-\E_{s-1}\left[a_{s} a_{s}^{\top}\right]\right)^2\right]
\notag
\\
&=\E_{s-1}\left[\left\{\left(x_{s}+e_{s}\right)\left(x_{s}+e_{s}\right)^{\top}-\E_{s-1}\left[\left(x_{s}+e_{s}\right)\left(x_{s}+e_{s}\right)^{\top}\right]\right\}^{2}\right] 
\\
&=\E_{s-1}\left[\left\{\left(x_{s} x_{s}^{\top}+x_{s} e_{s}^{\top}+e_{s} x_{s}^{\top}+e_{s} e_{s}^{\top}\right)-\E_{s-1}\left[x_{s} x_{s}^{\top}+x_{s} e_{s}^{\top}+e_{s} x_{s}^{\top}+e_{s} e_{s}^{\top}\right]\right\}^{2}\right]
\\ 
&=\E_{s-1}\left[\left(e_{s} e_{s}^{\top}-\E_{s-1}\left[e_{s} e_{s}^{\top}\right]+x_{s} e_{s}^{\top}+e_{s} x_{s}^{\top}-\E_{s-1}\left[x_{s} e_{s}^{\top}+e_{s} x_{s}^{\top}\right]\right)^{2}\right].
\end{align}
\label{eq:lemma.9.1}
\end{subequations}
Notice that $\vnorm{x_s} \leq a_{\max}$ and $\vnorm{e_s} \leq 2 a_{\max}$ almost surely.  From triangle inequality and sub-multiplicativity of the 2-norm, we get $\|(A_1 + A_2)^2\| \leq \left( \|A_1\| + \|A_2\| \right)^2$. Using this and Jensen's inequality in \eqref{eq:lemma.9.1}, we get
\begin{subequations}
\begin{align}
& \vnorm{\E_{s-1}\left[\left(a_{s} a_{s}^{\top}-\E_{s-1}\left[a_{s} a_{s}^{\top}\right]\right)^2\right] }
\notag
\\
& \leq \E_{s-1}\left[\left(\left\|e_{s} e_{s}^{\top}\right\| + \left\|\E_{s-1}\left[e_{s} e_{s}^{\top}\right]\right\|+\left\|x_{s} e_{s}^{\top}+e_{s} x_{s}^{\top} \right\| + \left\| \E_{s-1}\left[x_{s} e_{s}^{\top}+e_{s} x_{s}^{\top}\right]\right\|\right)^{2}\right] \label{lemma_10_3}
\\
& \leq \E_{s-1}\left[\left(2 a_{\max}\left\|e_{s} \right\| + 2 a_{\max}\E_{s-1}\left[\left\|e_{s} \right\|\right]+2a_{\max}\left\|e_s \right\| +  2a_{\max}\E_{s-1}\left[\left\|e_s\right\|\right]\right)^{2}\right] \label{lemma_9_7}\\
& = \E_{s-1}\left[\left( 4 a_{\max}\left\|e_{s} \right\| + 4 a_{\max}  \E_{s-1}\left[\left\|e_s\right\|\right]\right)^{2}\right] 
\\
& \leq 64 a_{\max}^2   \E_{s-1}\left[\left\|e_{s} \right\|^2 \right]  \\
& \leq 64 a_{\max}^2   \E_{s-1}\left[\left\|\sum_{i = 2}^n \nu_s^i \mu_s^i  \right\|^2 \right]  
\label{eq:lemma_9_8} 
\\
& =  64n a_{\max}^2   \frac{D}{\sqrt{s}}.
\end{align}
\end{subequations}
In the derivation of \eqref{eq:lemma_9_8}, we used the non-expansiveness of projection to deduce $\|e_s\| \leq \|\sum_{i = 2}^n \nu_s^i\mu_s^i\|$. The last step follows from Assumption \ref{assumption_perturbation}.
Finally, triangle inequality and the integral bound for summations yield
\begin{align}
\begin{aligned}
\left\|\sum_{s=1}^{t} \Ex{\left(\Ex{a_{s} a_{s}^{\top}}-a_{s} a_{s}^{\top}\right)^2 }\right\| & \leq \sum_{s=1}^{t}\left\| \Ex{\left(\Ex{a_{s} a_{s}^{\top}}-a_{s} a_{s}^{\top}\right)^2 }\right\|
\\
& \leq 64 a_{\max}^2 n D \int_{s=0}^{t} \frac{1}{\sqrt{s}} ds
\\
& = 128 a_{\max}^2 n D \sqrt{t},
\end{aligned}
\end{align}
proving the result.
\end{proof}

Our plan is to leverage a concentration inequality in conjunction with the above two lemmas to show that $\lambda_{\min}(V_t)$ behaves similar to $\lambda_{\min}(\E[V_t])$, which has been shown to grow as $\Omega(\sqrt{t})$. We now state said concentration inequality on matrix martingales that is adapted from the matrix Freedman inequality from \cite{tropp}. The proof requires minor variations to the derivation in \cite{tropp} and are omitted.

\begin{lemma}\label{matrix_freedman_improved}
Consider a discrete-time matrix martingale $\sum_{\tau=0}^{s} X_s$ of self-adjoint matrices in $\Rset^{n\times n}$ with ${X}_0  = 0$ that satisfies $
\lambda_{\max}( {X}_s ) \leq R$ almost surely for some $R > 0$ for all $s$. 
Then, for any $ 1\leq \Breve{T} \leq T$, and monotonically non-decreasing, non-negative sequence $\{\psi_s\}_{s= 1}^{T}$ and $C>0$, we have
\begin{align}
    \begin{aligned}
        & \probb{\exists s \in [\Breve{T}, T] : \lambda_{\max}\left(\sum_{\tau=0}^{s} X_s \right) \geq \psi_s, \ 
	\vnorm{\sum_{\tau=1}^s \E_{j-1} \left({X}_j^2\right)} \leq C \psi_s }
 \\
	&\quad \leq n \exp \left( \frac{- 3   \psi_{\Breve{T}}}{6C  + 2 R}  \right).
    \end{aligned}
\end{align}
\end{lemma}
Using this inequality, we now prove a probabilistic bound on how much the minimum eigenvalue of $V_t$ can deviate from that of its expectation. In effect, we choose $\psi_t \sim \sqrt{t}$ in the last result to show that with high probability, $\lambda_{\min}(V_t)$ itself grows as $\Omega(\sqrt{t})$--the same way it grows for its expectation in Lemma \ref{lemma_lower_bound_for_variation}.


\begin{lemma}\label{lemma:probability_Ct}
Suppose Assumptions  \ref{assumption_of_parameter_set}, \ref{assumption_of_action_set}, and \ref{assumption_perturbation} hold. 
Then, for all $(\delta,T)$ such that $F(\delta) \leq T$, 
\begin{align}
\probb{ \exists t \in [F(\delta), T] : \lambda_{\min }\left(V_t\right) \leq \frac{D c_0 \sqrt{t}}{2}}\leq \frac{\delta}{3}. \label{union_1}
\end{align}
\end{lemma}
\begin{proof}
The random matrices $X_{s}:=\E_{s-1}\left[a_{s} a_{s}^{\top} \right]-a_{s} a_{s}^{\top}$ form an adapted sequence of zero-mean random, self-adjoint matrices with $\vnorm{X_s} \leq 2a_{\max}^2$ almost surely. Recall the definition of $F(\delta)$ from \eqref{eq:Fdelta.def}. Upon choosing  $\psi_t := \frac{D c_0}{2}\sqrt{t}$, $C := \frac{256}{c_0} a_{\max}^2 n$, $R := 2 a_{\max}^2$, and $\Breve{T} = F(\delta)$, Lemma \ref{matrix_freedman_improved} yields
\begin{align}
\probb{ \exists t \in [F(\delta), T] : \lambda_{\max}\left(\sum_{s = 1}^t {X}_s\right) \geq \psi_t, 
	\vnorm{\sum_{s = 1}^t ({X}_s)^2} \leq C \psi_t }
	\leq \frac{\delta}{3}.
\label{eq:mFI.Vt}
\end{align}
In addition, Lemma \ref{lemma_upper_bound_for_variation} allows us to infer
\begin{align}
\vnorm{\sum_{s = 1}^t ({X}_s)^2} = \vnorm{\sum_{s = 1}^t \left(\E_{s-1}\left[a_{s} a_{s}^{\top} | \Fcal_{s-1}\right]-a_{s} a_{s}^{\top}\right)^2} \leq  C \psi_t
\end{align}
almost surely, which then allows us to conclude
\begin{align}
\probb{ \exists t \in [F(\delta), T] : \lambda_{\max}\left(\sum_{s = 1}^t {X}_s\right) \geq \frac{D c_0}{2} \sqrt{t}}
	\leq \frac{\delta}{3}.
 \label{eq:mFI.Vt.2}
\end{align}
Using Weyl's inequality from \cite[Theorem 4.3.1]{horn2012matrix}, Lemma \ref{lemma_lower_bound_for_variation}, and $\lambda=c_2$, we get
\begin{align}
\begin{aligned}
&\lambda_{\max }\left(\sum_{s = 1}^t {X}_s \right) \geq \frac{D c_0}{2} \sqrt{t} 
\\
&\iff \lambda_{\max }\left(-V_t + c_2 I + \sum_{s=1}^{t} \mathbb{E}_{s-1}\left[a_s a_s^{\top}\right]\right) \geq \frac{D c_0}{2} \sqrt{t}  
\\
&\Longleftarrow \lambda_{\max }\left(-V_t\right) + c_2 + \sum_{s=1}^{t} \lambda_{\min }\left(\mathbb{E}_{s-1}\left[a_s a_s^{\top}\right]\right) \geq \frac{D c_0}{2}\sqrt{t}   
\\
&\Longleftarrow \lambda_{\min }\left(V_t\right) \leq \frac{D c_0}{2} \sqrt{t}. 
\end{aligned}
\end{align}
This observation in \eqref{eq:mFI.Vt.2} completes the proof.
\end{proof}

We remark that the examination of step-wise deviations through the lens of martingale difference sequences with matrix concentration results along the lines of \cite{stern2020dynamic} offers substantial benefits in producing  bounds on $\lambda_{\min}(V_t)$ which are less frequently studied in the literature compared to bounds on $\lambda_{\min}(\E[V_t])$.

Our high-probability result for $\lambda_{\min}(V_t)$ now seamlesly plugs into the following result on confidence ellipsoids from \cite{improved_algorithms_for_stochastic_bandits}, reproduced below, without proof. Let $\vnorm{z}_Q := \sqrt{z^\top Q z}$ for $x \in \Rset^n$ and a positive semi definite matrix $Q \in \Rset^{n \times n}$. 
\begin{lemma} \label{lemma:probability_Et} Suppose Assumptions \ref{assumption_of_parameter_set}, \ref{assumption_noise}, and \ref{assumption_of_action_set} holds Let $\rho_t\left({\delta}\right) := M\sqrt{ n\log \left [\frac{1 + t a_{\max}^2/\lambda }{\delta} \right]}+\sqrt{\lambda} \theta_{\max}$. Then, for any $\delta \in(0,1)$, and any action sequence $\left\{a_{t}\right\}_{t = 1}^\infty$,
\begin{align}
   \probb{ \exists t \geq 1, \vnorm{\hat{\theta}_{t + 1} - \theta^{\star}}_{V_t} \geq \rho_{t}\left({\delta}\right)  }\leq {\delta} \label{union_2}
\end{align}
with respect to the noise sequence $\left\{\ve_{t}\right\}_{t = 1}^T$.
\end{lemma}
Simple algebra produces
\begin{align}
    \vnorm{\hat{\theta}_{t + 1} - \theta^{\star}}^2_{V_t} \geq \vnorm{\hat{\theta}_{t + 1} - \theta^{\star}}^2 \lambda_{\min}(V_t).
\end{align}
This last observation in conjunction with Lemmas \ref{lemma:probability_Ct} and  \ref{lemma:probability_Et} concludes the proof of Theorem \ref{corollary_inference}.

We comment that algorithms such as UCB and TS in multi-agent settings have been investigated in \cite{distributed_cooperative_decision_making_ucb,MAMAB_ucb,cooperative_MAMAB_ucb, MYS_wind_farm} and they require the agents to communicate over a peer-to-peer or a hub-and-spokes network to collectively solve non-convex optimization problems and/or sample from intricately designed distributions. Forced exploration algorithms such as that in \cite{stern2020dynamic} place less burdens on agents to communicate. We anticipate that our algorithm and its results on guaranteed inference will prove useful in algorithm design and analysis for the multi-agent cooperative variants of the linear bandit problem.

\section{Regret Analysis}
\label{sec:regret}

Having derived guarantees on the quality of inference on the estimation of the underlying parameter in the last section, we now turn to analyze the expected regret of our algorithm. 
\begin{theorem}
\label{theorem_regret} Suppose Assumptions  \ref{assumption_of_parameter_set}, \ref{assumption_noise}, \ref{assumption_of_action_set}, and \ref{assumption_perturbation} hold. With $\lambda = \frac{8 n D^2}{\overline{m}_{\varphi}^{2} } (c_0 n + 1)$ for Algorithm \ref{alg:FELP},
$\probb{\mathscr{R}_{\theta^{\star}}(T) \leq  C \sqrt{T}  \log(T) } \geq 1 - \frac{1}{T}
$
for sufficiently large $T$, where $C$ does not depend on $T$.
\end{theorem}
As our information-theoretic lower bound in the next section will reveal, this $\sqrt{T}\log(T)$ regret bound is order optimal up to a log factor. The key steps in our proof are as follows. We essentially ignore the regret from the first $\sqrt{T}$ samples. By this time, Theorem \ref{corollary_inference} provides high-probability guarantees on the quality of the RLS estimate. The quality is controlled in the large $T$ regime, so as to allow the application of Lemma \ref{perturbation_upperbound} that proves the Lipshitz nature of $a^\star$, which we then utilize to bound $\vnorm{a^\star(\hat{\theta}_t) - a^\star(\theta^\star)}^2$. We translate that to a bound on $\vnorm{a_t - a^\star(\theta^\star)}^2$ by Chernoff-style concentration on the perturbations and projections that contribute to the difference between $a^\star(\hat{\theta}_t)$ and $a_t$. The last step relates the above difference to the per-round regret post $t=\sqrt{T}$ via Lemma \ref{quadratic_upperbound}.

\begin{proof}
Projection is nonexpansive, and hence, 
    \begin{align}
    \begin{aligned}
    \vnorm{a_{t } - a^\star(\theta^\star)}^2 & = \vnorm{\projA{a^\star(\hat{\theta}_t) + \sum_{i = 2}^n\nu_t^i \mu_t^i} - a^\star(\theta^\star)}^2\\
    & \leq \vnorm{a^\star(\hat{\theta}_t) + \sum_{i = 2}^n\nu_t^i \mu_t^i - a^\star(\theta^\star)}^2\\
    & \leq 2\vnorm{a^\star(\hat{\theta}_t) - a^\star(\theta^\star)}^2 + 2\vnorm{\sum_{i = 2}^n\nu_t^i \mu_t^i}^2.
    \end{aligned}
    \label{rewriting_the_action}
    \end{align}
For large $T$, we have $F(1/T) \sim \Ocal(\log^2(T)) < \sqrt{T}$,
where $F$ is as defined in \eqref{eq:Fdelta.def}. Hence, for a sufficiently large $T$, Theorem \ref{corollary_inference} allows us to infer
\begin{align}
\probb{ \bigcap_{t=\lfloor \sqrt{T} \rfloor + 1}^T \vnorm{\hat{\theta}_{t + 1} - \theta^\star}^2 \leq \frac{ 2\rho_t^2(1/T)}{D c_0 \sqrt{t}}} \geq 1 - \frac{2}{3T},
\label{eq:good.estimation}
\end{align}
where $\rho_t$ is as defined in \eqref{eq:rhoT.def}. Furthermore, $\rho_T(1/T) \sim \Ocal(\sqrt{\log(T^2)})$, and hence,
\begin{align}
    \sup_{t \in [\sqrt{T}, T]} \frac{2\rho_t^2(1/T)}{D c_0 \sqrt{t}} 
    \leq \frac{2\rho_T^2(1/T)}{D c_0 T^{1/4}} \to 0,
\end{align}
implying that for large enough $T$, the event in \eqref{eq:good.estimation} is such that 
\begin{align}
  t \geq \sqrt{T} 
  \implies \frac{2\rho_t^2\left(1/T\right)}{D c_0 \sqrt{t}} \leq \frac{m' }{4} \left(\frac{2 \nmax\theta_{\max}}{\hmin\theta^2_{\min}} \right)^{-2} \leq m_{\theta}. \label{R1_property_1}
\end{align}
On this event, we apply Lemma \ref{perturbation_upperbound} to obtain
\begin{align}
\probb{\bigcap_{t=\lfloor \sqrt{T} \rfloor + 1}^T \vnorm{a^{\star}(\hat{\theta}_{t + 1}) - a^{\star}(\theta^\star)}^2 \leq \underbrace{\left(\frac{2 \nmax\theta_{\max}}{\hmin\theta^2_{\min}} \right)^2 \frac{ 2\rho_t(1/T)^2}{D c_0 \sqrt{t}}}_{=: \tilde{\kappa}_t \leq m'/4}} \geq 1-\frac{2}{3T}.
\label{action_deviation_part_1}
\end{align}
 To bound the second term in \eqref{rewriting_the_action}, 
define $\kappa_t := \frac{2 n D}{\sqrt{t}}\log \left( 12n t T\log(T)\right)$ and use Chernoff bound,
\begin{align}
\begin{aligned}
     \prob\left\{{ \vnorm{\sum_{i=2}^n \nu^i_t \mu^i_t}^2 > \kappa_t } \right\}
        \leq \sum_{i=2}^n \prob\left\{ | \nu^i_t | > \sqrt{\frac{\kappa_t}{n}}  \right\}
    = \frac{1}{6tT\log(T)}.\label{probability_of_oversampling}
\end{aligned}
\end{align}
Union bound then gives
\begin{align}
\begin{aligned}
\probb{{ \bigcup_{t = \lfloor \sqrt{T} \rfloor + 1}^T \left \{\vnorm{\sum_{i=2}^n \nu^i_t \mu^i_t}^2 > \kappa_t  \right \}} }
& \leq \frac{1}{6T\log(T)}\sum_{t = 2}^T  t^{-1} +  \frac{1}{6T\log(T)}
\\
& \leq \frac{1}{6T\log(T)} \int_{1}^T t^{-1}dt +  \frac{1}{6T\log(T)}
\\
& \leq \frac{1}{3T}.\label{action_deviation_part_2}
\end{aligned}
\end{align}
For large $T$, we have $\sup_{t\in[\sqrt{T},T]} \kappa_t \leq \Ocal(\polylog(T))/\sqrt{T} \to 0$, implying that for large enough $T$, 
\begin{align}
t \geq \sqrt{T}
\implies 
\kappa_t \leq \frac{m'}{4} 
\label{eq:kappa.t.m'}
\end{align}
Combining \eqref{action_deviation_part_1}, \eqref{action_deviation_part_2}, and \eqref{eq:kappa.t.m'} in \eqref{rewriting_the_action}, we get
\begin{align}
    \probb{ \bigcap_{t=\lfloor \sqrt{T} \rfloor + 1}^T \vnorm{a_{t + 1} - a^\star(\theta^\star)}^2 \leq m' }
     \geq 1-\frac{1}{T}.
    \label{final_upper_bound_on_action}
    \end{align}
Over this high-probability event, Lemma \ref{quadratic_upperbound} applies and gives an upper bound on the regret from $\sqrt{T}$ onwards as
\begin{align}
    \sum_{t = \lfloor \sqrt{T} \rfloor + 1}^{T} \theta^{\star\top} (a^\star - a_{t})
    \leq \frac{\theta_{\max}}{2\varphi_g}  \sum_{t = \lfloor \sqrt{T} \rfloor + 1}^{T} \vnorm{a^\star(\theta^\star) - a_{t}}^2
    \leq  \frac{\theta_{\max}}{2\varphi_g}  \sum_{t = \lfloor \sqrt{T} \rfloor + 1}^{T} \left(\kappa_t + \tilde{\kappa}_t\right),
    \label{eq:R1T.beyond}
\end{align}
where the last summation itself can be upper bounded as
\begin{align}
\begin{aligned}
    \sum_{t = \lfloor \sqrt{T} \rfloor + 1}^{T} \kappa_t + \tilde{\kappa}_t
    &\leq  \left [\left(\frac{2\nmax\theta_{\max}}{\hmin\theta^2_{\min}} \right)^2 \frac{ 2\rho_T^2(1/T)}{D c_0 }  +  {2 n D\log \left(12n T^2\log(T)\right)} \right] \sum_{t=1}^T \frac{1}{\sqrt{t}}
    \\
    & \lesssim  \sqrt{T} \log(T). 
\end{aligned}
\label{eq:R1T.beyond.bound}
\end{align}
The rest then follows from plugging the bound from \eqref{eq:R1T.beyond.bound} in \eqref{eq:R1T.beyond} and bounding the per-round regret up to $t=\lfloor \sqrt{T} \rfloor$ by a constant, appealing to the compactness of $\Acal$ and $\Theta$.
\end{proof}


\section{A Universal Lower-Bound on Regret}
\label{sec:lowerBoundRegret}

Are there algorithms that can achieve lower than $\Omega(\sqrt{T})$ regret on convex compact action sets of the type we considered? In what follows, we answer that question in the negative, making our algorithm order-wise optimal. We achieve this by establishing a deep connection between inference quality and regret accumulation. We start by showing that regardless of the control policy one uses, the order of regret accumulation of an action sequence also limits the ability to infer the underlying governing parameter $\theta^\star$. In this section, $\pi$ defines a causal control policy adapted to the filtration $\{\Fcal_t\}_{t=0}^T$. Let $\E_{\theta^\star}^\pi$ stand for expectation conditioned on the action selection policy $\pi$, when the underlying governing parameter is $\theta^\star$. 

\begin{theorem} \label{inference_regret_connaction} Suppose Assumptions  \ref{assumption_of_parameter_set} and \ref{assumption_of_action_set}
hold.\footnote{We remark that the upper bound on $\alpha_\Acal$ in Assumption \ref{assumption_of_action_set} is irrelevant to this section.} Define 
\begin{align}
    c_3 := \min \left\{\dfrac{\theta_{\min}\hmin}{2 \nmax}, \frac{c_{\min} }{4 a_{\max}^2} \right\},
\end{align}
where $c_{\min}:= \min_{\theta \in \Theta} \min_{a \in \Acal \setminus \Bcal_{a^{\star}(\theta)}(m'')} r_{\theta}(a)$ for a sufficiently small $m'' \in (0, m')$. Then, for any $z \perp a^\star(\theta^\star)$ with $\vnorm{z} \leq 1$,
\begin{align}
z^\top V_T z \leq c_3^{-1}  \mathscr{R}_{\theta^{\star}}(T) \quad \text{almost surely}.
\end{align}
\end{theorem} 
\begin{proof}
 Due to Assumption \ref{assumption_of_action_set}, $c_{\min}$ is well defined and positive since $r_{\theta^{\star}}(a)$ is positive for all  $a \in \Acal$ such that $ a \neq a^{\star}(\theta^{\star})$. Denote by $\mathds{1}_t^m$, the indicator function of the event that  $ \vnorm{a^\star(\theta^{\star}) - a_t} \leq {m''}$. First, we characterize regret on this event.
 Taylor's expansion of $g$ around $a^\star(\theta^{\star})$ then yields
\begin{align}
g(a_t)-g(a^\star(\theta^{\star})) = \nabla g(a^\star(\theta^{\star}))^\top (a_t - a^\star(\theta^{\star})) + \dfrac{1}{2}(a_t-a^\star(\theta^{\star}))^\top \Delta g(\xi) (a_t - a^\star(\theta^{\star}))
\label{eq:taylor.x.a.lower_bound}
\end{align}
for some $\xi$ on the line segment joining $a^\star(\theta^{\star})$ and $a_t$. Since $a_t \in \Bcal_{a^\star(\theta^{\star})}(m'') \in \Bcal_{a^\star(\theta^{\star})}(m')$, then $\xi \in \Bcal_{a^\star(\theta^{\star})}(m')$. Using Assumption \ref{assumption_of_action_set} and $\mathring{\theta^\star} = \mathring{\nabla}g(a^\star(\theta^\star))$ from Lemma \ref{lemma_nab}, we 
infer from \eqref{eq:taylor.x.a.lower_bound} that
\begin{align}
\begin{aligned}
 \mathring{\theta}^{\star\top} ( a^\star(\theta^\star) - a_t) & \geq \dfrac{\hmin}{2 \nmax} \|a_t-a^\star(\theta^\star)\|^2 - \underbrace{\left[g(a_t) - g(a^\star(\theta^\star)) \right]}_{\leq 0}\\
 & \geq \dfrac{\hmin}{2 \nmax} \|a_t-a^\star(\theta^\star)\|^2,
\end{aligned}
\end{align}
where in the last line we have used that $a^\star(\theta^\star) \in \Acal^\star$ and $a_t \in \Acal$, i.e., $g(a^\star(\theta^\star)) = 0$ and $g(a_t) \leq 0$, respectively. Outside the event $\mathds{1}_t^m$, the regret is lower bounded by $c_{\min}$.
Thus, we obtain
\begin{align}
\begin{aligned}
r_{\theta^\star}(a_t) &  =  r_{\theta^\star}(a_t) \mathds{1}_t^m +r_{\theta^\star}(a_t) (1 - \mathds{1}_t^m)
\\
& \geq  \dfrac{\theta_{\min}\hmin}{2 \nmax}\|a^\star(\theta^{\star}) - a_t\|^2\mathds{1}_t^m + c_{\min} (1 - \mathds{1}_t^m)
\\
& \geq  \dfrac{\theta_{\min}\hmin}{2 \nmax}\|a^\star(\theta^{\star}) - a_t\|^2\mathds{1}_t^m + \frac{c_{\min}}{4a_{\max}^2}\|a^\star(\theta^{\star}) - a_t\|^2 (1 - \mathds{1}_t^m)
\\
& \geq c_3 \|a^\star(\theta^{\star}) - a_t\|^2. \label{regret_to_quadratic_variation}
\end{aligned}
\end{align}
For any unit vector $z \perp a^\star(\theta^\star)$, we have $z^\top a_t = z^\top (a_t-a^\star(\theta)^\star)$, and hence,
\begin{align}
\begin{aligned}
z^\top (V_T - \lambda) z 
& = \sum_{t = 1}^T z^\top a_t a_t^\top z
\\
& = \sum_{t = 1}^T z^\top (a_t - a^\star(\theta^\star))(a_t - a^\star(\theta^\star))^\top z 
\\
&  \leq \sum_{t = 1}^T \vnorm {a_t - a^\star(\theta^\star)}^2
\\
& \leq c_3^{-1} \mathscr{R}_{\theta^{\star}}(T),
\end{aligned}
\end{align}
owing to \eqref{regret_to_quadratic_variation} in the last line. All relations in this derivation hold almost surely.
\end{proof}

Using this relation between inference and regret accumulation, we now develop a universal lower bound on regret growth, taking a Bayesian Cramer-Rao style analysis that requires certain assumptions. We mildly increase the requirement  on the action set, but relax the requirements on the noise model.  In the sequel,  $\lambda_{n-1}$ denotes the second smallest eigenvalue of a matrix.

\begin{assumption}[Regularity Conditions]\label{assumption_of_action_set_additional}

$\Acal$ and $\varepsilon$ must satisfy the following.
\begin{enumerate}[leftmargin=*,label=(\alph*)]
    \item There exists a $\lambda_C > 0$ such that 
    $\lambda_{n-1}(\nabla a^\star(\theta)) \geq \lambda_C$     
    for all $\theta \in \Theta$.

\item The eigenvalues $\{\lambda(\theta)\}_{i = 1}^n$ and the corresponding eigenvectors $\{u(\theta)\}_{i = 1}^n$ of $\nabla a^\star(\theta)$ can be selected so that they define a differentiable function on $\Theta$. 

\item The probability density function $p$ of the noise process $\ve$ is such that 
\begin{align}
    \mathscr{I}_{\ve} := \E\left[ \left(p'/p\right)^2\right]
    \label{eq:I.eps}
\end{align}
is finite and positive.
\end{enumerate}
\end{assumption}

The first assumption in part (a) ensures that a change in $\theta$ will have a tangible impact on its reward-maximizing action $a^\star(\theta)$, except possibly in one direction that defines the kernel  of $\nabla a^\star(\theta)$. Part (b) requires smoothness of eigenspaces of the parametric Jacobian $\nabla a^\star(\theta)$, a property we conjecture is true for parametrized families of symmetric rank-one perturbations under mild conditions, as discussed in \cite{kato1980perturbation}. Part (c) is a relaxation and replacement of Assumption \ref{assumption_noise} and encompasses a wide variety of noise models.

To achieve the lower bounds, we utilize the multivariate van Trees inequality with origins in \cite{GillLevit1995}. We present a variation of the result quoted in \cite{ziemann2021uninformative}. The notation $\trace$ computes the trace of its matrix argument.

\begin{lemma}\label{modified_van_trees}[Van Trees Inequality]
Let $\theta$ be a random parameter over a compact set $\Theta \subset \Rset^n$ with prior density $\varrho$, $C:\Theta \to \Rset^{n \times n}$ be a differentiable map, and $X \in \Rset^n$ be a random vector with density $p(x|\theta)$. Define
\begin{align}
\mathscr{I}(\theta) &:= \int \left( \frac{\nabla_\theta p(x|\theta)}{p(x|\theta)} \right) \left( \frac{\nabla_\theta p(x|\theta)}{p(x|\theta)} \right)^\top p(x|\theta) \, dx,
\\
\mathscr{J}(\varrho) &:= \int C(\theta)\left( \frac{\nabla_\theta \varrho(\theta)}{\varrho(\theta)} \right) \left( \frac{\nabla_\theta \varrho(\theta)}{\varrho(\theta)} \right)^\top C(\theta)^\top\varrho(\theta) \, d\theta.
\end{align}
If $\psi, p(x|\theta), \varrho$ are differentiable on $\Theta$, where $\varrho$  vanishes on the boundary of $\Theta$, the support of $p(x|\theta)$ does not depend on $\theta$, and both $\mathscr{I}(\theta)$ and $\mathscr{J}(\varrho)$ are finite\footnote{One way to ensure $\mathscr{J}(\varrho)$ to be finite is to assume that $\E[\nabla_\theta \log \varrho(\theta) \nabla_\theta \log \varrho(\theta)^\top]$ is finite and $\vnorm{C}$ remains bounded over $\Theta$.}, then for any $X$-measurable function $\hat{\psi}$,
\begin{align}
\begin{aligned}
\E \left[ \vnorm{\hat{\psi}(X) - \psi(\theta)}^2\right]  
\geq \frac{\left(\trace \left\{\E[\nabla_\theta \psi(\theta) C(\theta)^\top] \right\}\right)^2}{n \left( \trace \left\{ \E \left[ C(\theta) \mathscr{I}(\theta) C(\theta)^\top\right]\right\} + \trace\left\{\mathscr{J}(\varrho)\right\}\right)}.
\end{aligned}
\end{align}
\end{lemma}

\begin{proof}
In the proof structure of \cite[Appendix D]{ziemann2021uninformative}, we select $v_1 = \hat{\psi}(X) - \psi(\theta)$ and $v_2 = C(\theta)\nabla_\theta \log[p(x|\theta)\lambda(\theta)]$. Following their arguments yields
\begin{align}
\begin{aligned}
&\E \left[ (\hat{\psi}(X) - \psi(\theta))(\hat{\psi}(X) - \psi(\theta))^\top \right] \\
& \geq \E \left[\nabla_\theta \psi(\theta) C(\theta)^\top \right]  \left( \E \left[ C(\theta) \mathscr{I}(\theta) C(\theta)^\top\right] + \mathscr{J}(\varrho) \right)^{-1} \E \left[C(\theta) \nabla_\theta \psi(\theta)^\top\right]
\\
& =: \Phi_1 \Phi_2^{-1} \Phi_1^\top.
\end{aligned} \label{van_trees_their_result}
\end{align}
Trace is submultiplicative over positive semidefinite matrices, and hence,
\begin{align}
\trace\{\Phi_2\} \trace\{\Phi_1^\top \Phi_2^{-1} \Phi_1\} 
= \trace\{\Phi_2\} \trace\{\Phi_2^{-1} \Phi_1 \Phi_1^\top\}
\geq \trace\{\Phi_1 \Phi_1^\top\} 
\geq \frac{1}{n}[\trace\{\Phi_1\}]^2,
    \label{eq:trace.matrices}
\end{align}
where the last step uses Jensen's inequality. Using trace on both sides of  \eqref{van_trees_their_result}, and utilizing \eqref{eq:trace.matrices} finalizes the proof, given that $\E$ and $\trace$ commute.
\end{proof}

We are now ready to present our universal lower-bound on regret.

\begin{theorem}\label{theorem_lower_bound}
Suppose Assumptions  \ref{assumption_of_parameter_set}, \ref{assumption_of_action_set}, and \ref{assumption_of_action_set_additional} hold. 
Then, for any policy $\pi$ and a large enough $T$,  $ \sup_{\theta^{\star} \in \Theta}  \E_{\theta^\star}^\pi\left[\mathscr{R}_{\theta^{\star}}(T)\right]\gtrsim  \sqrt{T/\mathscr{I}_{\ve}}$.
\end{theorem}
\begin{proof}
Within this proof, we lighten notation and denote $\theta^\star$ by $\theta$, assuming that it is random with a prior $\varrho$ on $\Theta$ that satisfies the conditions of Lemma \ref{modified_van_trees}. We will apply Lemma \ref{modified_van_trees} with $\psi(\theta) = a^\star(\theta)$ and $\hat{\psi}(X) = a_t$. Then, $\nabla_\theta\psi(\theta) = \nabla_\theta a^\star(\theta)$, whose expression is given in Lemma \ref{corollary_eigenvalues}. Next, we identify $C(\theta)$.

By direct calculation, it follows that $\nabla a^\star(\theta)$ is a positive semidefinite matrix with $\theta$ in its null space. Per Assumption \ref{assumption_of_action_set_additional}, other eigenvalues are positive for each $\theta \in \Theta$. Choose $u:\Theta \to \Rset^n$ as a differentiable map of eigenvectors corresponding to $\lambda:\Theta \to \Rset$, a continuous map of positive eigenvalues of $\nabla a^\star$. Define 
\begin{gather}
\begin{gathered}
        z(\theta) := \proj_{a^\star(\theta)^\perp} [u(\theta)] = u(\theta) - \left[ u(\theta)^\top \mathring{a^\star}(\theta) \right] \mathring{a^\star}(\theta),\\
    C(\theta) := u(\theta) z(\theta)^\top,
\end{gathered}
\end{gather}
where $a^\star(\theta)^\perp$ denotes the orthogonal complement of $a^\star(\theta)$.\footnote{This construction ensures that $\vnorm{C}$ is bounded over $\Theta$.}

Next, we simplify and lower bound $({\trace \{\nabla_\theta \psi(\theta) C(\theta)^\top\}})^2$. To that end, we have
\begin{subequations}
\begin{align}
\frac{1}{\lambda^2(\theta) }\left(\trace\{C(\theta) \nabla a^{\star}(\theta)^\top\}\right)^2 & =  
\frac{1}{\lambda^2(\theta) }\left(\trace\{u(\theta) z(\theta)^\top \nabla a^{\star}(\theta)\}\right)^2
\label{rhs_numerator_0}
\\
& = \left[ z(\theta)^\top u(\theta) \right]^2
\label{rhs_numerator_1}\\
&=  1 - \left[u(\theta)^\top \mathring{a^\star}(\theta) \right]^2
\label{rhs_numerator_2}\\
&= 1 - \left[u(\theta)^\top \left( \mathring{a^\star}(\theta)^\top \mathring{\theta} \mathring{\theta}+ \mathring{a^\star}(\theta)^\top \mathring{\theta}^\perp  \mathring{\theta}^\perp \right) \right]^2
\label{rhs_numerator_3}\\
&= 1 - {\left[u(\theta)^\top \mathring{\theta}^\perp \right]^2} 
\left[  \mathring{a^\star}(\theta)^\top \mathring{\theta}^\perp \right]^2
\label{rhs_numerator_4}\\
&= 1 - \underbrace{\vnorm{u(\theta)}^2 \vnorm{\mathring{\theta}^\perp}^2}_{=1} 
\left[  \mathring{a^\star}(\theta)^\top \mathring{\theta}^\perp \right]^2
\label{rhs_numerator_5}\\
&\geq \left[  \mathring{a^\star}(\theta)^\top \mathring{\theta} \right]^2
\label{rhs_numerator_6}\\
&\geq \cos^2(\alpha_\Acal)\label{rhs_numerator_7}.
\end{align}
\end{subequations}
where \eqref{rhs_numerator_0} uses the symmetry of $\nabla a^\star$, \eqref{rhs_numerator_2} follows from the definition of $u, z$ and \eqref{rhs_numerator_3} follows from resolving $a^\star(\theta)$ along $\mathring{\theta}$ and its complement $\mathring{\theta}^\perp$. By construction, $u(\theta) \perp a^\star(\theta)$ that implies \eqref{rhs_numerator_4}. Then, \eqref{rhs_numerator_5} uses Cauchy-Schwarz inequality. The last line follows from Assumption \ref{assumption_of_action_set} and observing that $\mathring{\theta} = \mathring{\nabla}g(a^\star(\theta))$, according to Lemma \ref{lemma_nab}. Thus, from Assumption \ref{assumption_of_action_set_additional}, we deduce
\begin{align}
    \left(\trace\{\E[C(\theta) \nabla a^{\star}(\theta)^\top]\}\right)^2 
    \geq \lambda_C^2 \cos^2(\alpha_\Acal).\label{van_trees_1}
\end{align}

To leverage Lemma \ref{modified_van_trees}, we are left to calculate and bound $\trace \left\{ \E \left[ C(\theta) \mathscr{I}_t(\theta) C(\theta)^\top\right]\right\}$, where $\mathscr{I}_t(\theta)$ is the Fisher information matrix generated by $\{a_1, Y_1,...,a_{t-1},Y_{t-1}\}$. Since Fisher information for sequential data follows a chain rule of the type in \cite[Appendix C]{ziemann2021uninformative}, we obtain
\begin{align}
\begin{aligned}
\mathscr{I}_t(\theta) 
& = \sum_{s = 1}^t \E_{\theta}^{\pi} \left[ \nabla  \log\left( p(Y_s - \theta^{\top} a_s)\right) \nabla \log\left( p(Y_s - \theta^{\top}a_s))\right)^\top  \right]
\\
& = \sum_{s = 1}^t \E_{\theta}^{\pi} \left[ a_s a_s^\top \left( \frac{p'(Y_s - \theta^{\top} a_s)}{p(Y_s - \theta^{\top} a_s)} \right)^2   \right]
\\
& = \sum_{s = 1}^t \E_{\theta}^{\pi} \left[ a_s a_s^\top\right]  \E_{\theta}^{\pi}\left[ \left( \frac{p'(\ve_s)}{p(\ve_s)} \right)^2 \right] 
\\
& = \E_{\theta}^{\pi} \left[ V_t - \lambda I \right] \mathscr{I}_{\ve}.
\end{aligned} \label{IE_equality}
\end{align}
where we have used the fact that $a_s$ and $\varepsilon_s$ are independent. 
This allows us to write
\begin{align}
\begin{aligned}
\trace \left\{ \E \left[ C(\theta) \mathscr{I}_t(\theta) C(\theta)^\top\right]\right\}
& = \E \left[ \trace \left\{  u(\theta) z(\theta)^\top \left(\E_{\theta}^{\pi} \left[ V_t \right] - \lambda  I \right) z(\theta) u(\theta)^\top \right\} \right]\\
& = \mathscr{I}_{\ve}\E \left[   z(\theta)^\top \left(\E_{\theta}^{\pi} \left[ V_t \right] - \lambda  I \right)z(\theta) \right]\\
& \leq \mathscr{I}_{\ve}\E \left[   z(\theta)^\top\E_{\theta}^{\pi} \left[ V_t \right] z(\theta) \right]
,\label{van_trees_2}
\end{aligned}
\end{align}
where we have used $\vnorm{u(\theta)}=1$.
Plugging \eqref{van_trees_1} and \eqref{van_trees_2} in Lemma \ref{modified_van_trees} yields
\begin{align}
\begin{aligned}
  \E \left[ \vnorm{\hat{\psi}(X) - \psi(\theta)}^2\right] &  = \int_\Theta \E_{\theta}^{\pi} \left [\vnorm{ a_t - a^{\star}(\theta) }^2 \right]\varrho(\theta) d\theta \\
  & \geq \frac{\lambda_C^2 \cos^2(\alpha_\Acal) }{ n \mathscr{I}_{\ve}  \int_{\Theta}z(\theta)^\top \E_{\theta}^{\pi} \left[ V_t \right] z(\theta) \varrho(\theta)d\theta+ n\trace \left\{ \mathscr{J}(\varrho)\right\}}.
\end{aligned}
\end{align}
Summing the inequalities for $t = 1$ to $T$, 
\begin{subequations}
\begin{align}
  &\int_\Theta \sum_{t = 1}^T\E_{\theta}^{\pi} \left [\vnorm{ a_t - a^{\star}(\theta) }^2 \right]\varrho(\theta) d\theta 
  \notag
  \\
  & \geq \sum_{t = 1}^T  \frac{\lambda_C^2 \cos^2(\alpha_\Acal) }{ n \mathscr{I}_{\ve}  \int_{\Theta}z(\theta)^\top \E_{\theta}^{\pi} \left[ V_t \right]z(\theta) \varrho(\theta)d\theta+ n\trace \left\{ \mathscr{J}(\varrho)\right\}}\\
 & \geq   \frac{\lambda_C^2 \cos^2(\alpha_\Acal) T}{ n \mathscr{I}_{\ve}  \int_{\Theta}z(\theta)^\top \E_{\theta}^{\pi} \left[ V_T \right]z(\theta) \varrho(\theta)d\theta+ n\trace \left\{ \mathscr{J}(\varrho)\right\}},\label{summed_van_trees}
\end{align}
\end{subequations}
where in the last line, we have used that $V_T$ dominates $V_t$ in the positive semidefinite order. Notice that \eqref{regret_to_quadratic_variation} from the proof of Lemma \ref{inference_regret_connaction} produces $
\frac{1}{c_3} \E_{\theta}^\pi\left[\mathscr{R}_{\theta}(T)\right] \geq  \sum_{t = 1}^T\E_{\theta}^{\pi} \left [\vnorm{ (a_t - a^{\star}(\theta)) }^2 \right]
$, which in the last inequality, yields
\begin{align}
 \frac{1}{c_3}  \int_\Theta \E_{\theta}^\pi\left[\mathscr{R}_{\theta}(T)\right] \varrho(\theta)d\theta  
 & \geq  \frac{\lambda_C^2 \cos^2(\alpha_\Acal) T  }{ n \mathscr{I}_{\ve}\int_\Theta  z(\theta)^\top \E_{\theta}^\pi\left[V_T\right] z(\theta) \varrho(\theta) \: d\theta  + n \trace\{\mathscr{J}(\varrho)\}}.
\label{van_trees_ineq_fin_integral_form}
\end{align}
Since supremum dominates expectation, we get
\begin{align}
\sup_{\theta \in \Theta}\E_{\theta}^\pi \left[\mathscr{R}_{\theta}(T) \right]
 & \geq   \frac{\lambda_C^2 \cos^2(\alpha_\Acal)  c_3 T}{ n\mathscr{I}_{\ve} \sup_{\theta \in \Theta}  z(\theta)^\top \E_{\theta}^\pi\left[V_T\right] z(\theta)   + n\trace \left\{ \mathscr{J}(\varrho)\right\}},\label{van_trees_ineq_fin} 
\end{align}
Finally, $z(\theta) \perp a^\star(\theta)$ by construction. Appealing to Theorem \ref{inference_regret_connaction}, we obtain
\begin{align}
\sup_{\theta \in \Theta}\E_{\theta}^\pi \left[\mathscr{R}_{\theta}(T) \right] \geq  \frac{\lambda_C^2 \cos^2(\alpha_\Acal) c_3^2 \;T}{ n \mathscr{I}_{\ve} \sup_{\theta \in \Theta}\E_{\theta}^\pi \left[\mathscr{R}_{\theta}(T) \right] + n c_3\trace \left\{ \mathscr{J}(\varrho)\right\}}. 
\end{align}
This yields $\sup_{\theta \in \Theta}  \E_{\theta}^\pi\left[\mathscr{R}_{\theta}(T)\right] \gtrsim   \sqrt{{T}/{\mathscr{I}_{\ve}}}$
for sufficiently large $T$.
\end{proof}
Theorems \ref{inference_regret_connaction} and \ref{theorem_lower_bound} provide a deep connection between inference and control. The close relation of $[\lambda_{\min}(V_t)]^{-1}$ and the inference of $\theta^\star$, shown by \cite{improved_algorithms_for_stochastic_bandits}, allows us to informally state Theorem \ref{inference_regret_connaction} as  
\begin{align}
\text{Inference Quality}\lesssim  \text{Cumulative Regret},
\label{eq:informal.1.repeat}
\end{align}
restated from \eqref{eq:informal.1}.
In other words, one cannot achieve high inference quality (with small confidence ellipsoids) without compromising performance from a control policy, whose expected cost is $\E_{\theta^\star}^\pi\left[\mathscr{R}_{\theta^{\star}}(T)\right]$. On the other hand, \eqref{van_trees_ineq_fin_integral_form} within the proof of Theorem \ref{theorem_lower_bound} implies a complementary result--if a policy does not guarantee a good-enough inference, then it cannot be regret-wise order-optimal. 
To see why, for directions $z(\theta^\star)$ orthogonal to $a^\star(\theta^\star)$, \eqref{van_trees_ineq_fin_integral_form} gives
\begin{align}
\E^{\varrho,\pi}\left[\mathscr{R}_{\theta^{\star}}(T) \right]\E^{\varrho}[z(\theta^\star)^\top \E^\pi_{\theta^\star}[V_T] z(\theta^\star)], \gtrsim T, \label{connection_1}
\end{align}
where $\E^{\varrho}[\cdot] = \int_{\theta^\star} \cdot d\varrho(\theta^\star)$ averages over the prior $\varrho$. If $\pi$ incurs a sublinear expected regret, then $a_t\sim \pi$ must ultimately align with $a^\star(\theta^\star)$. Since $V_T$ adds contributions from $a_t a_t^\top$, the sublinear regret accumulation then implies
$\mathring{a}^\star(\theta^\star)^\top\E^\pi[ V_T] \mathring{a}^\star(\theta^\star) \gtrsim T$ for any $\theta^\star \in \Theta$. 
Theorem \ref{inference_regret_connaction}, combined with the sublinear nature of the regret under $\pi$, then implies that the eigenvector corresponding to the minimum eigenvalue of $\E^\pi[ V_T]$ is roughly orthogonal to $a^\star(\theta^\star)$ for large $T$. In effect, we deduce from \eqref{connection_1} that
\begin{align}
\E^{\varrho,\pi}\left[\mathscr{R}_{\theta^{\star}}(T) \right]\E^{\varrho}[\lambda_{\min}(\E^{\pi}[V_T])] \gtrsim T
\label{eq:E.R.Lambda.T}
\end{align}
for any algorithm $\pi$ with sublinear expected regret. Said informally,
\begin{align}
\left( \text{Cumulative Regret}\right) \left(\text{Inference Quality}\right) \gtrsim T,
\label{eq:informal.2.repeat}
\end{align}
as restated from \eqref{eq:informal.2} in the Bayesian sense. Notice that \eqref{eq:informal.1.repeat} holds over each sample path and choice of the underlying parameter, while \eqref{eq:informal.1.repeat} holds in expectation over a prior of parameter choices. Notwithstanding the possible role of the prior, these two statements together roughly indicate that optimal regret can at most be $\Ocal(\sqrt{T})$, for which one must have $\lambda_{\min}(\E^{\pi}[V_T]) \sim \Omega(\sqrt{T})$--a sentiment that resonates with the findings of \cite{rich_action_spaces} under the condition that the boundary of the action set has a locally constant Hessian around the optimal action. No regret-wise order-optimal policy can guarantee better than $\Omega(\sqrt{T})$ inference quality. This observation runs counter to the conclusions made in \cite[Section 8]{rich_action_spaces} on unit $L^p$ norm-balls, a topic we discuss in detail in Section \ref{sec:Lp.balls}.

\section{Emprical Comparison with Other Algorithms}
\label{sec:numerics}

We considered two experimental setups to evaluate the performance of four algorithms for linear bandits. Namely, they are Linear Upper Confidence Bound (LinUCB) from \cite{context_bandits_ucb}, Thompson Sampling (TS) from \cite{linear_thompson_sampling_revisited}, Forced Exploration for Linear Bandit Problems (FEL) from \cite{AbbasiYadkori2009ForcedExplorationBA}, and our TRAiL algorithm on action sets defined as $\{x \in \Rset^n: x^\top \Delta_g x \leq 1\}$ for a matrix $ \Delta_g$. Such sets include both spherical sets considered in Section \ref{sec:numerics.sphere} and ellipsoidal sets in Section \ref{sec:numerics.ellipsoid}.
For TS, we employed the frequentist approach of \cite{linear_thompson_sampling_revisited} that avoids the need to compute Bayesian updates of classic TS in \cite{thompson_sampling_for_contextual_bandits_with_linear_payoffs}. The pseudo-codes of these algorithms are included in Appendix \ref{sec:appendix_pseudocodes}.

The FEL algorithm from \cite{AbbasiYadkori2009ForcedExplorationBA} separates the exploration and exploitation steps. During exploration, one needs a policy that ensures $\E[a_t a_t^\top]$ is full rank.  Since the experiments are conducted using ellipsoidal action sets, we devised such a strategy using the eigendirections of the ellipses. For an ellipsoidal $\Acal$ defined by the function $x^\top \Delta_g x \leq 1$,  we calculate an eigenbasis for $\Delta_g$ denoted as $\{e_i\}_{i = 1}^n$ with the eigenvalues $\{\lambda_i\}_{i = 1}^n$. Randomly sampling $\{e_i\}_{i = 1}^n$ and playing $e_i/\sqrt{\lambda_i}$ assures that $\E[a_t a_t^\top] \succeq 1/\lambda_{\min}(\Delta_g)$ and that $a_t \in \Acal$. 

The projection step in TRAiL can be computationally intensive, particularly when the structure of the action space is complex. However, in our experiments, we employed ellipsoidal sets for the action space, which allowed us to efficiently compute $\textrm{proj}_{\Acal}$ using \cite[Section 8]{projection_onto_ellipsoids}. We remark that even when the action space is more general than being ellipsoidal, TRAiL can be adapted with a minor adjustment to the projection step. We anticipate that our analysis remains valid if this step is replaced by finding the nearest point in the action set such that the difference from the perturbed action aligns with $\hat{\theta}_t$.

\subsection{Experiments on the Unit Sphere}
\label{sec:numerics.sphere}
In the initial setup, the action set $\Acal$ was defined as the unit sphere ($\Delta_g = I_{10}$) in $\Rset^{10}$ with $\ve_t \sim \Ncal(0,0.1)$. We selected $\theta^\star$ uniformly at random from the unit sphere of radius $1$. We tuned the hyperparameters  $\lambda^{\textrm{TS}}, \lambda^{\textrm{UCB}}, c^{\textrm{FEL}}, D^{\textrm{TRAiL}}$\footnote{In Section \ref{sec:numerics}, the hyperparameters $D$ and $\lambda$ of TRAiL are denoted as $D^{\text{TRAiL}}$ and $\lambda^{\text{TRAiL}}$, respectively.}, but chose $\lambda^{\textrm{TRAiL}} = 0.01$ without tuning. We performed a grid search using the  expected cumulative regret on these parameters and selected those that resulted in the lowest mean + standard deviation evaluated over 10 runs for $T = 1000$. The grid search  resulted in  $\lambda^{\textrm{TS}} =0.01 $, $\lambda^{\textrm{UCB}} = 0.01$, $c^{\textrm{FEL}} = 0.4$, $D^{\textrm{TRAiL}} = 0.1$. The tuning process for $D^{\textrm{TRAiL}}$ is visualized in Figure \ref{circle_sub_fig_1}, where the results are generated over $30$ runs of the TRAiL algorithm. Until $D^{\textrm{TRAiL}}$ reaches $ 0.1$, decreasing it reduces regret. Beyond a certain threshold, however, it leads to insufficient exploration and hurts performance. 

After selecting the hyperparameters, we  sampled $20$ new $\theta^\star$'s lying on the unit-sphere and compared the algorithms. TRAiL outperformed all other algorithms followed by FEL and UCB as Figure \ref{circle_sub_fig_2} reveals. We repeated the experiment with a higher error variance where $\ve_t \sim \Ncal(0,1)$. The grid-search yielded $\lambda^{\textrm{TS}} =0.2$, $\lambda^{\textrm{UCB}} = 0.1$, $c^{\textrm{FEL}} = 0.6$, $D^{\textrm{TRAiL}} = 0.5$. According to Figure \ref{circle_sub_fig_3}, TRAiL performed better than other algorithms.

\begin{figure}[!ht]
    \centering
    \begin{subfigure}[b]{0.3\textwidth}
        \centering
        \includegraphics[width=\textwidth]{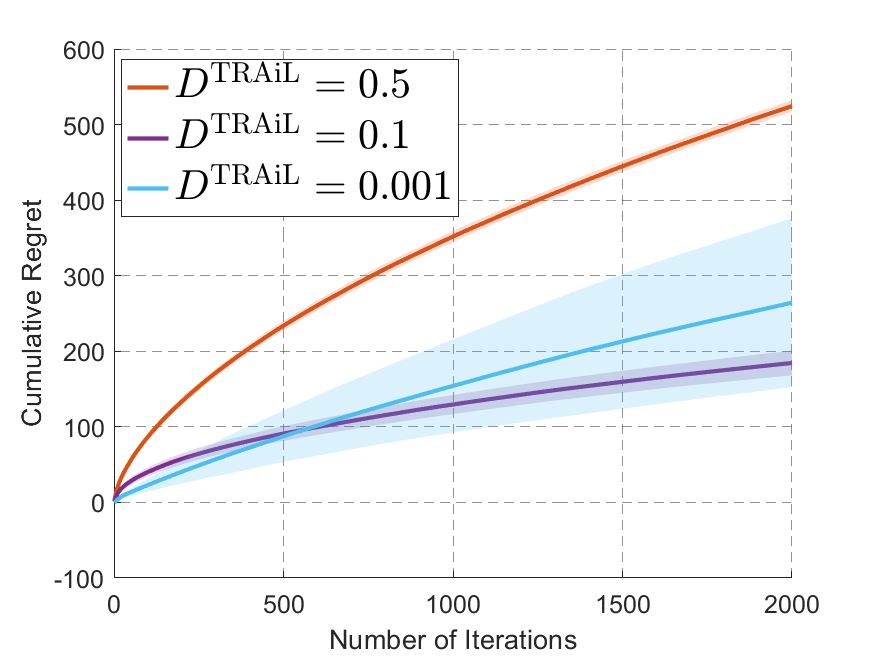}
        \caption{Effect of $D^{\textrm{TRAiL}}$}
        \label{circle_sub_fig_1}
    \end{subfigure}
    \hfill
    \begin{subfigure}[b]{0.3\textwidth}
        \centering
        \includegraphics[width=\textwidth]{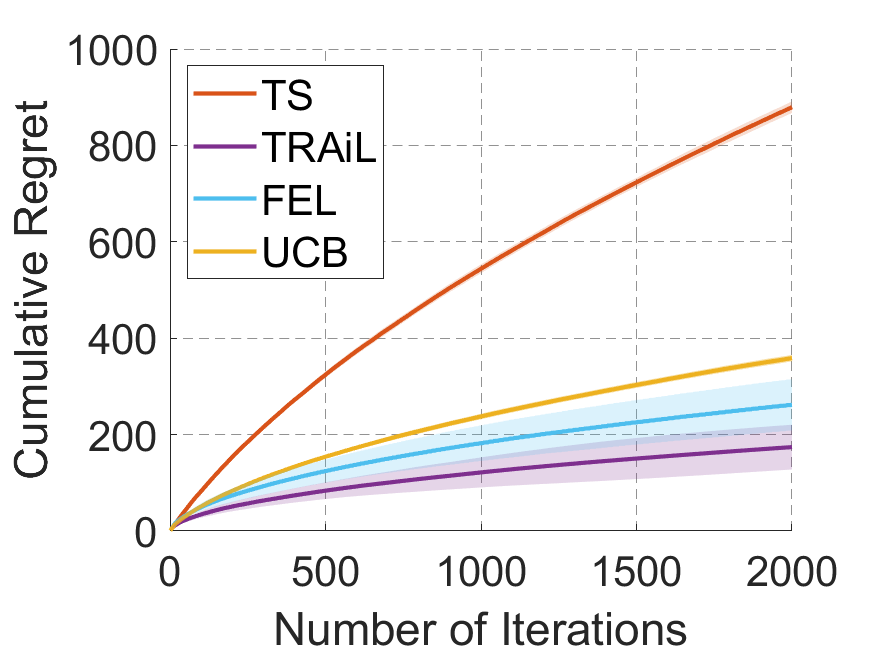}
        \caption{$\ve_t \sim \Ncal(0,0.1)$}
        \label{circle_sub_fig_2}
    \end{subfigure}
    \hfill
    \begin{subfigure}[b]{0.3\textwidth}
        \centering
        \includegraphics[width=\textwidth]{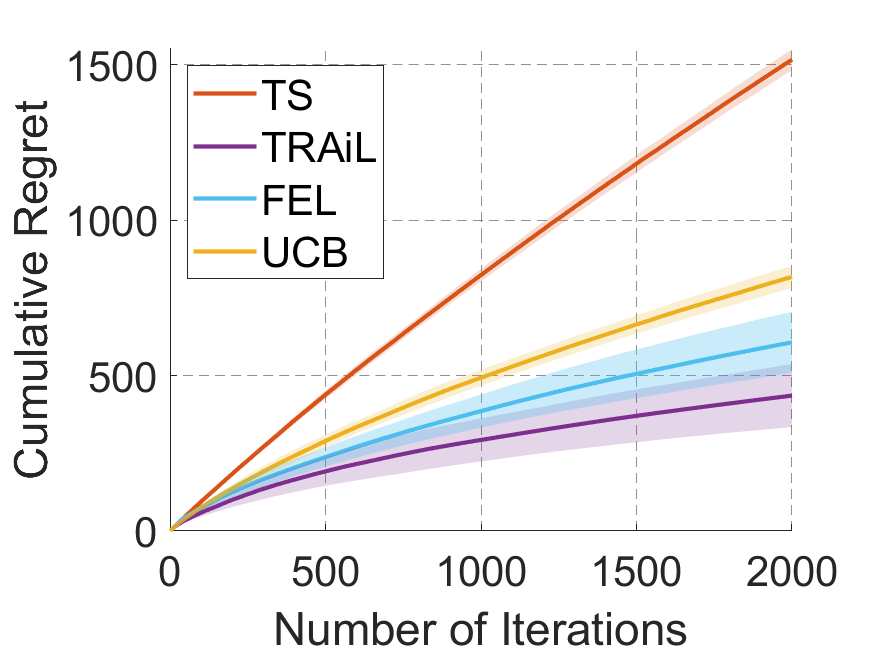}
        \caption{$\ve_t \sim \Ncal(0,1)$}
        \label{circle_sub_fig_3}
    \end{subfigure}
    \caption{Comparing FEL, TS, UCB, and TRAiL on a spherical action set in $\Rset^{10}$ where the area that lies in one standard deviation from the means of the curves are shaded with their respective colors; (a) shows the effect of $D^{\textrm{TRAiL}}$ on regret, (b)-(c) plots  regret with reward error variances of $0.1$ and $1$, respectively.}
    \label{fig:three_subfigures_sphere}
\end{figure}

\subsection{Experiments on Randomly Sampled Ellipsoids}
\label{sec:numerics.ellipsoid}
For the final experiment, we set $\ve_t=0.1$, but significantly expanded the complexity of the action set. We sampled random ellipsoidal action spaces in $\Rset^{20}$ as follows. We generated $\Delta_g = I_n + \frac{1}{n}\sum_{i = 1}^n \delta_i x_i x_i^\top$ where $x_i \sim \Ncal(0, I_n)$ and $\delta_i \sim 10^{-(1 + v_i)}$ for $v_i \sim \textrm{Uniform}[0,1]$. The hyper parameters were tuned using 10 trials and the algorithms were tested with 20 new action sets, sampled independently from the given distribution. The training process produced $\lambda^{\textrm{TS}} = 0.15$, $c^{\textrm{FEL}} = 0.3$, $D^{\textrm{TRAiL}} = 0.3$. UCB was excluded from this comparison due to the considerable computational demands of the algorithm in high-dimensional settings, which rendered the tuning process computationally burdensome. TRAiL outperformed TS and FEL in terms of regret, as evidenced by Figure \ref{sub_fig_2}. The same conclusion remained true when we repeated the experiment with action sets in $\Rset^{100}$, as Figure \ref{sub_fig_3} shows. For $n = 100$, the training process produced $\lambda^{\textrm{TS}} = 0.01$, $c^{\textrm{FEL}} = 0.07$, $D^{\textrm{TRAiL}} = 0.03$.

We remark that TRAiL outperformed FEL when the condition numbers of the ellipsoids $\lambda_{\max}(\Delta_g)/\lambda_{\min}(\Delta_g)$ were larger. This is expected, given that FEL's fixed exploration strategy is more sensitive to the composition of the action set--variations in eigendirections directly influence the degree of perturbation introduced into the system.  Consequently, significant size discrepancies among action sets can lead to unnecessary exploration with FEL.

\begin{figure}[h]
    \centering
    \begin{subfigure}[b]{0.3\textwidth}
        \centering
        \includegraphics[width=\textwidth]{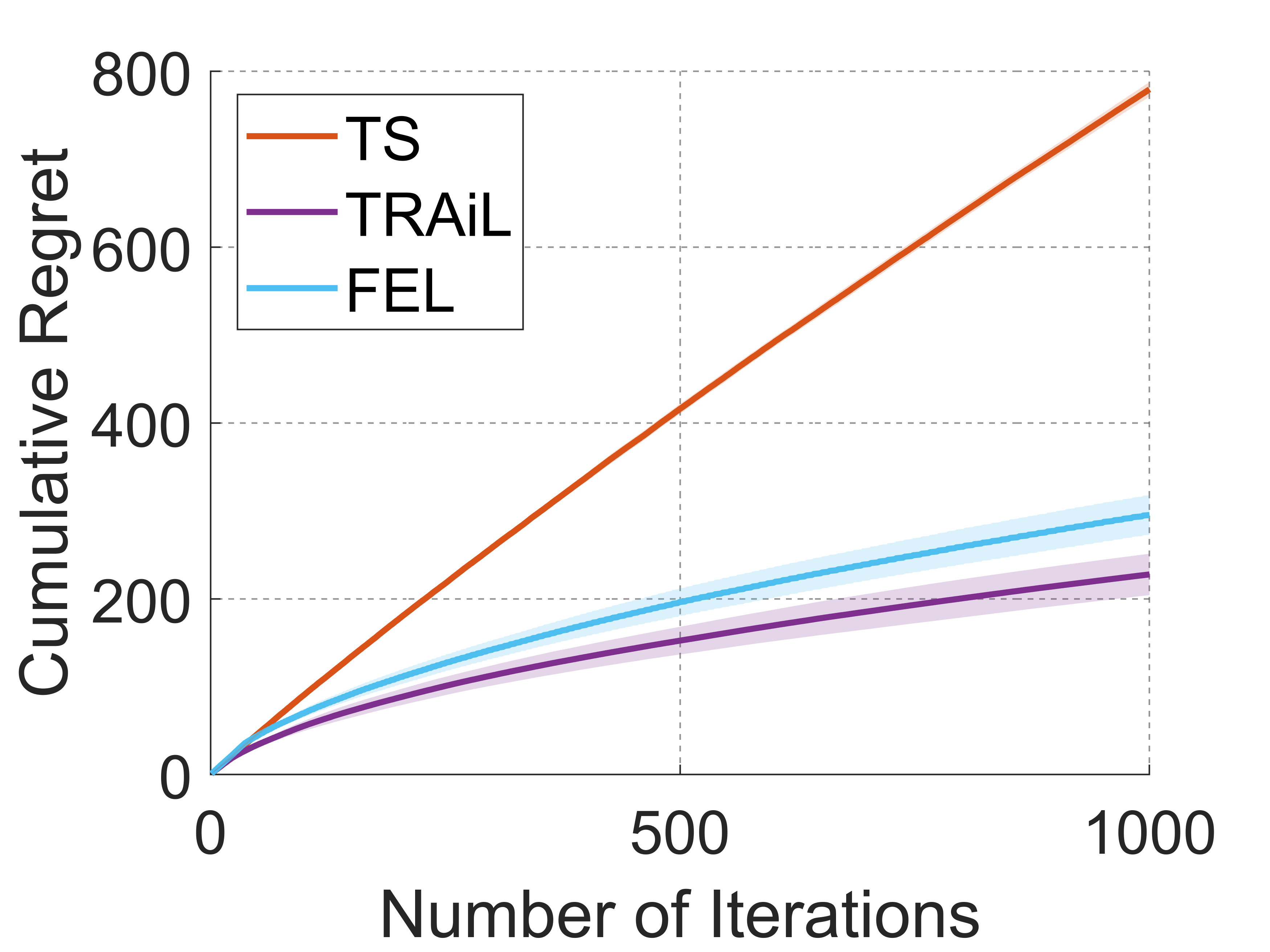}
        \caption{$n = 20$}
        \label{sub_fig_1}
    \end{subfigure}
    \hfill
    \begin{subfigure}[b]{0.3\textwidth}
        \centering
        \includegraphics[width=\textwidth]{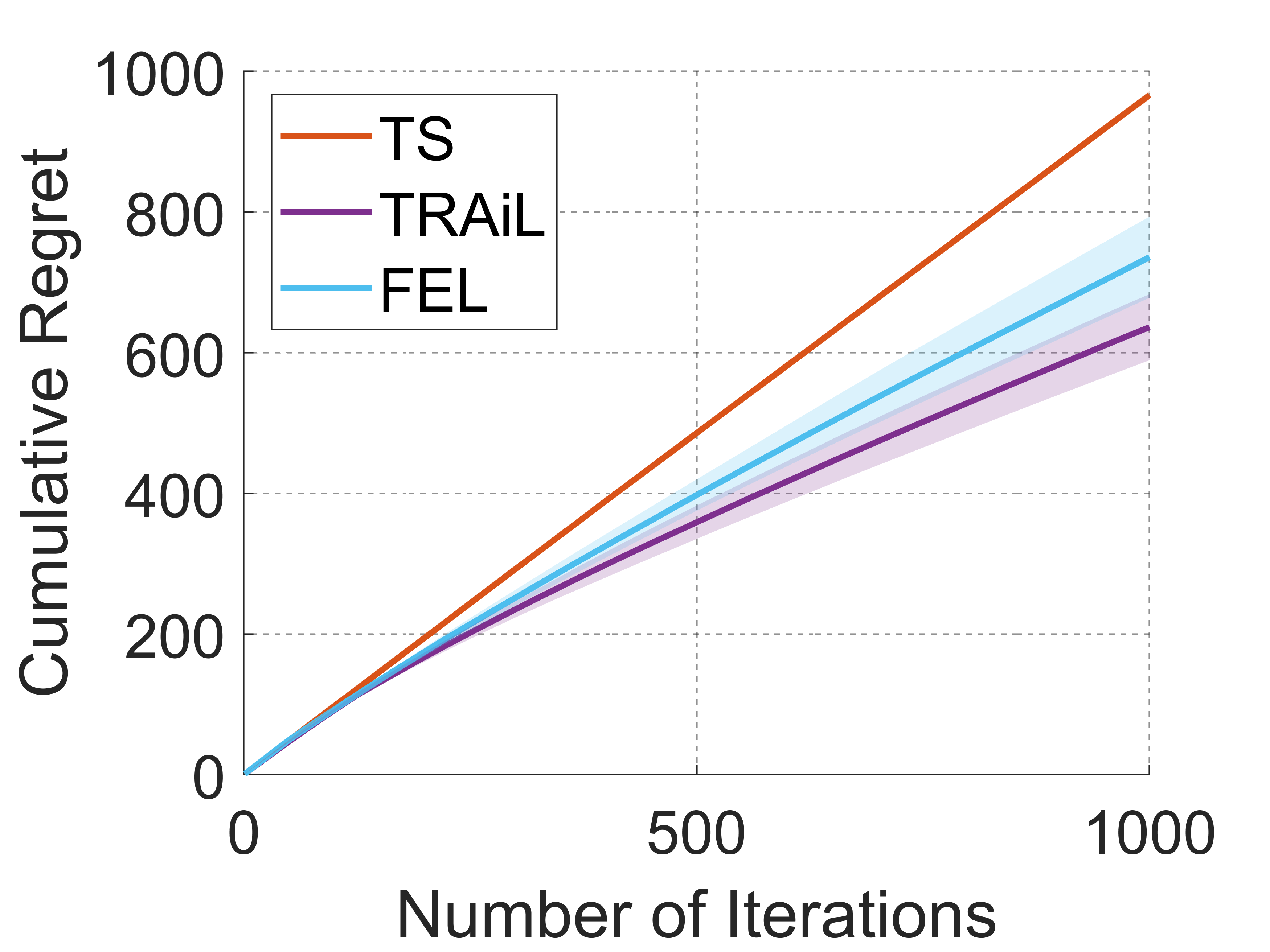}
        \caption{$n = 100$}
        \label{sub_fig_2}
    \end{subfigure}
    \hfill
    \begin{subfigure}[b]{0.3\textwidth}
        \centering
        \includegraphics[width=\textwidth]{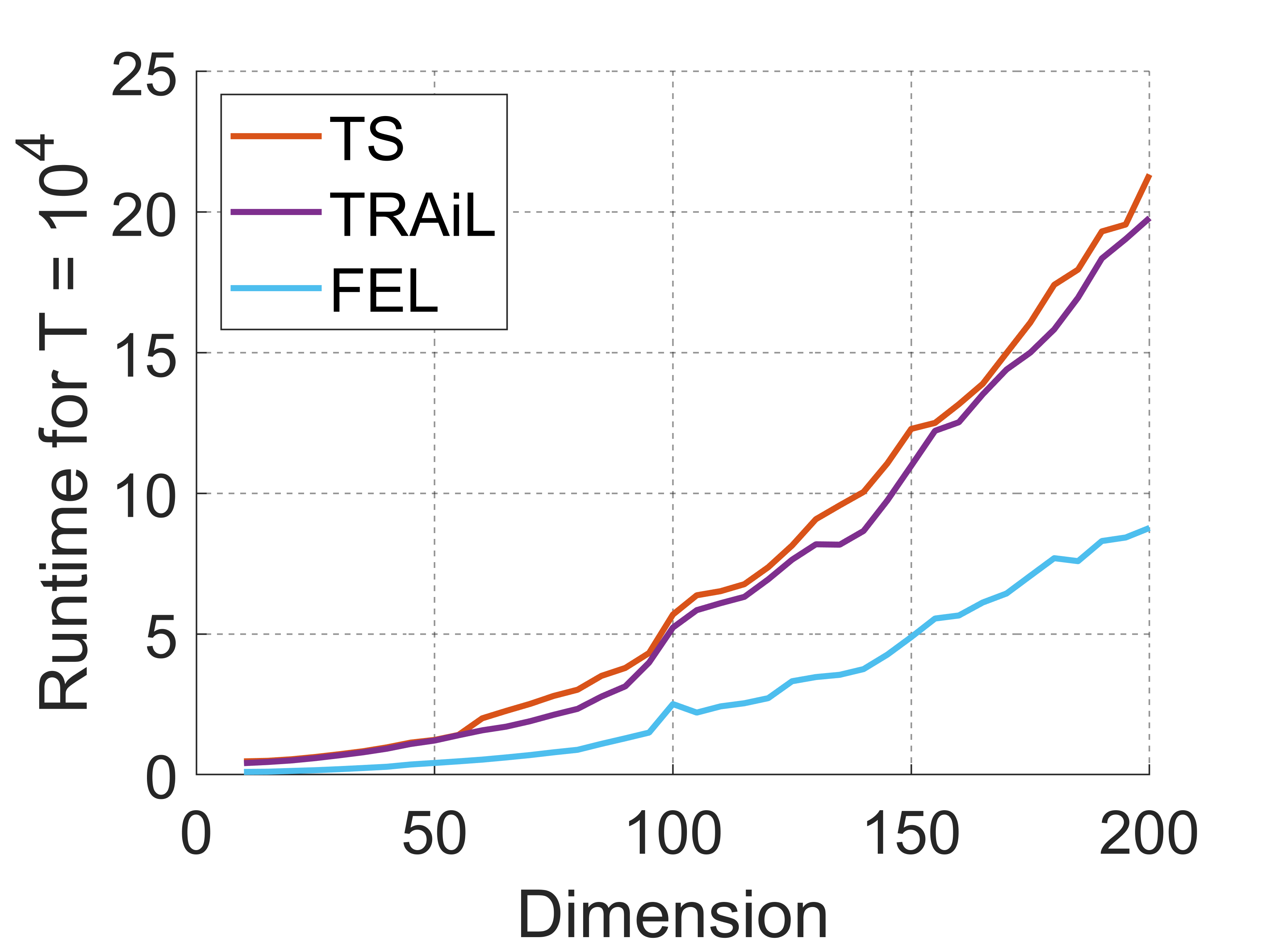}
        \caption{Average run-times}
        \label{sub_fig_3}
    \end{subfigure}
    \caption{Comparing FEL, TS, and TRAiL on randomly selected ellipsoidal action sets where $\ve_t \sim \Ncal(0,0.1)$; (a), (b) show the regret curves for these algorithms where the the action sets are in $\Rset^{20}$ and $\Rset^{100}$, respectively, (c) provides a speed comparison in seconds with the dimension of the action space varied between 10 and 200 and averaged over 5 runs with $T = 10^4$.}
    \label{fig:three_subfigures_ellipsoid}
\end{figure}

To gauge the computational speeds, we varied the action set dimension between $10$ and $200$ and generated ellipsoidal actions sets using the same sampling procedure. We ran $5$ trials for each of these algorithms with $T = 10^4$. The average run-times of the algorithms in seconds are shown in Figure \ref{sub_fig_3}. FEL proved to be the most computationally efficient algorithm, aligning with our expectations. This is due to its infrequent exploration over time, reducing computational demands, with the primary computational burden being the identification of $a^\star(\hat{\theta}_t)$, a task shared by all algorithms. TRAiL also demonstrated slightly higher computational efficiency compared to TS. These experiments were conducted on a computer with a Intel(R) Xeon(R) CPU E3-1225 v5 processor, 16 GB RAM, on MATLAB R2021a.

\section{The Curious Case of the \texorpdfstring{$L^p$}{Lp} Unit-Norm Balls}
\label{sec:Lp.balls}
We concluded Section \ref{sec:lowerBoundRegret} with the thumb rule that good inference is crucial for good performance. We begin this section with a toy \emph{counterexample} to this principle that motivates the study of unit $L^p$-norm balls as action sets of the form $\Acal^p:=\{a \in \Rset^2 | \vnorm{a}_p \leq 1\}$. This section is inspired by the insightful experiments conducted in \cite[Section 8]{rich_action_spaces} and allows us to deepen the connection between inference and performance in linear bandits on convex sets with smooth surfaces.

Consider the action set $\Acal^\infty$. This set is \emph{not} the sublevel set of a strongly convex function, but it serves to illustrate the notion of control relevance of inference. For \emph{all} $\theta^\star \in \Rset^2_{+}$ save the origin, the optimal action is $a^\star(\theta^\star) = (1,1)$, the upper-right corner of $\Acal^\infty$. Hence, even with an $\Ocal(1)$-deviation in the inference of $\theta^\star$, an algorithm can identify the optimal arm. How fast should $\lambda_{\min}(V_t)$ grow to deduce $\theta^\star$ up to an $\Ocal(1)$ deviation? Lemma \ref{lemma:probability_Et} comes to the rescue to answer that question. Specifically,  
\begin{align}
\vnorm{\hat{\theta}_t - \theta^\star} \leq \frac{\rho_{t}(1/T)}{\sqrt{\lambda_{\min}(V_t)}} \approx \frac{\sqrt{\log T}}{\sqrt{\lambda_{\min}(V_t)}}.
\end{align}
In other words, $\lambda_{\min}(V_t) \gtrsim {\log (T)}$ allows a crude (up to $\Ocal(1)$) inference of $\theta^\star$ with high probability, and such inference is sufficient for zeroing out further regret. Hence, better than $\Ocal(1)$ inference becomes irrelevant to control performance.

$\Acal^\infty$ has a nonsmooth boundary. Authors of \cite{rich_action_spaces} carry the intuition of control-irrelevance of inference to its smooth counterpart $\Acal^{10}$. With a regret-optimal algorithm, they demonstrate that for $\frac{1}{\log t}\log(\lambda_{\min} (V_t)) \to 0.4$ within 5000 iterations, indicating that $\lambda_{\min}(V_T) \ngtr \Omega(\sqrt{T})$. In the sequel, we argue why this example does \emph{not} contradict our conclusion in Section \ref{sec:lowerBoundRegret}, but provides deep insights into the interplay between inference/control metrics and the geometry of the action set.

The inequality in \eqref{eq:E.R.Lambda.T} can be stated as
\begin{align}
\E^{\varrho,\pi}\left[\mathscr{R}_{\theta^{\star}}(t) \right]
\geq \frac{c_N t}{\mathscr{I}_\varepsilon \E^{\varrho}[\lambda_{\min}(\E^{\pi}_{\theta^\star}[V_t])] + \mathscr{J}(\varrho)}.
\end{align}
For $\varepsilon_t \sim \Ncal(0, \sigma_\varepsilon^2)$ and $\varrho = \Ncal((1,1), \sigma_{\theta^\star}^2 I)$, we have $\mathscr{I}_\varepsilon = \sigma_\varepsilon^{-2}$ and $\mathscr{J}(\varrho) \simeq \sigma_{\theta^\star}^{-2}$. 
With ``small'' $\sigma_{\theta^\star}$, the prior becomes concentrated around $(1,1)$ and the inequality implies
\begin{align}
\E^{\pi}\left[\mathscr{R}_{\theta^{\star}}(t) \right] \E^{\pi}_{\theta^\star}[\lambda_{\min}(V_t)]
\approx \E^{\pi}\left[\mathscr{R}_{\theta^{\star}}(t) \right] \lambda_{\min}(\E^{\pi}_{\theta^\star}[V_t])
\geq \sigma_\varepsilon^{2} c_N t,
\label{eq:er.ei.t}
\end{align}
provided $\lambda_{\min}(\E^{\pi}_{\theta^\star}[V_t]) \gg \mathscr{J}(\varrho)/\mathscr{I}_\varepsilon = \sigma_\varepsilon^2/\sigma_{\theta^\star}^2$ for sufficiently large $t$ for all $\theta_\star$ close enough to $(1,1)$. When $\pi$ has sublinear regret, $a_t \to a^\star(\theta^\star)$ and we expect $V_t$ to concentrate under $\pi$, allowing us to exchange $\lambda_{\min}$ and expectation. Also, $\E^\pi$ can now be estimated via its sample average. 
Assume that $\E^{\pi}\left[\mathscr{R}_{\theta^{\star}}(t) \right] \sim t^{e_R}$ and $ \E^{\pi}_{\theta^\star}[\lambda_{\min}(V_t)] \sim t^{e_I}$ for large $t$. Then, \eqref{eq:er.ei.t} indicates that we would expect $e_R + e_I \geq 1$.

\begin{figure}[!t]
    \centering
    \begin{subfigure}[b]{0.4\textwidth}
        \centering
        \includegraphics[width=\textwidth]{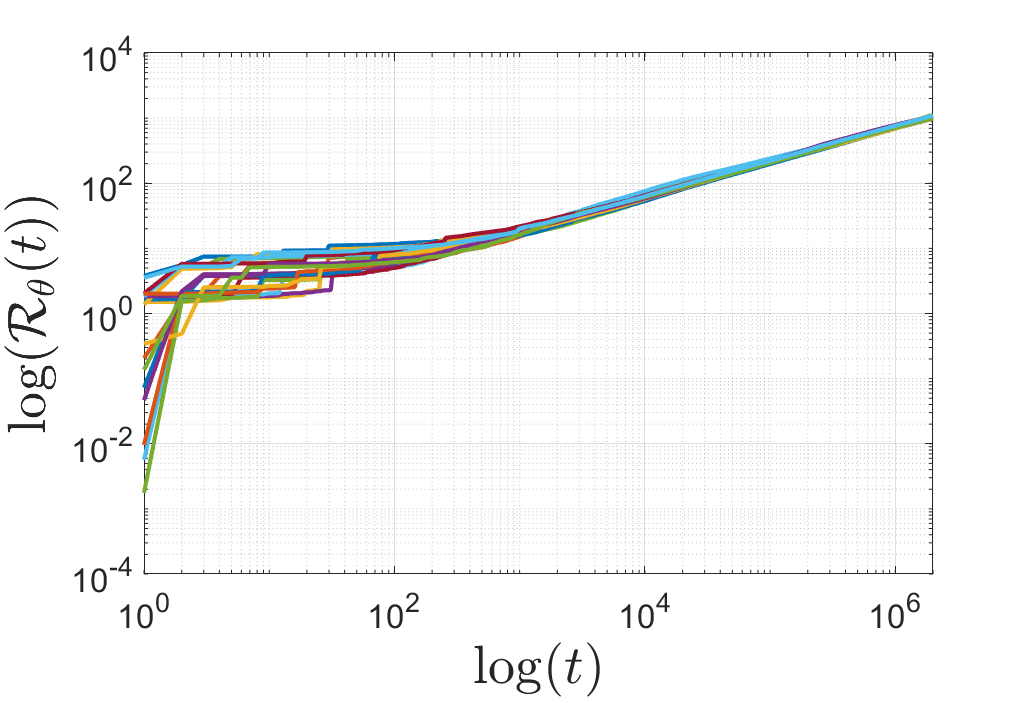}
        \caption{Regret}
        \label{fig:regret.performance}
    \end{subfigure}
    \begin{subfigure}[b]{0.4\textwidth}
        \centering
        \includegraphics[width=\textwidth]{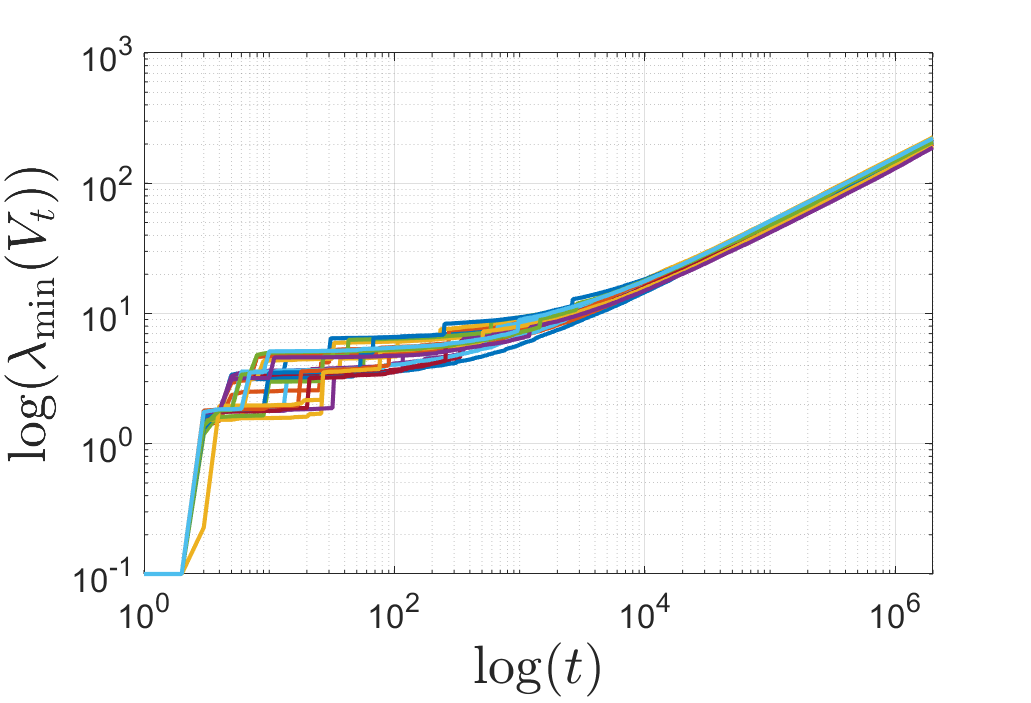}
        \caption{Inference}
        \label{fig:inference.quality}
    \end{subfigure}
    \\
    \begin{subfigure}[b]{0.4\textwidth}
        \centering
        \includegraphics[width=\textwidth]{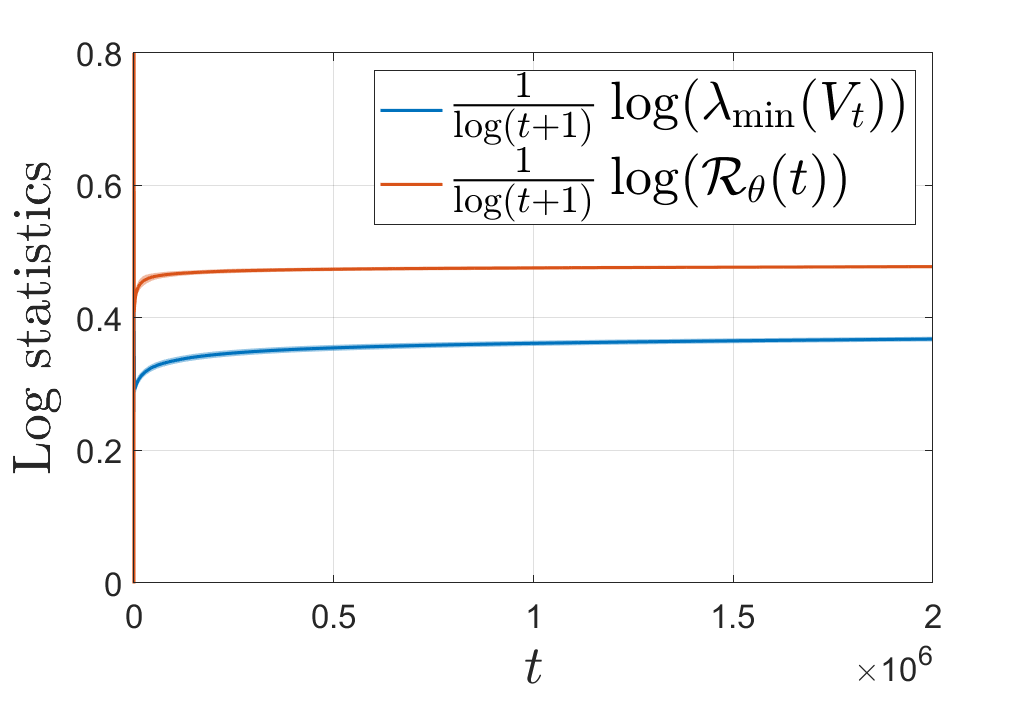}
        \caption{Log Statistics}
        \label{fig:log.stats}
    \end{subfigure}
    \begin{subfigure}[b]{0.4\textwidth}
        \centering
        \includegraphics[width=\textwidth]{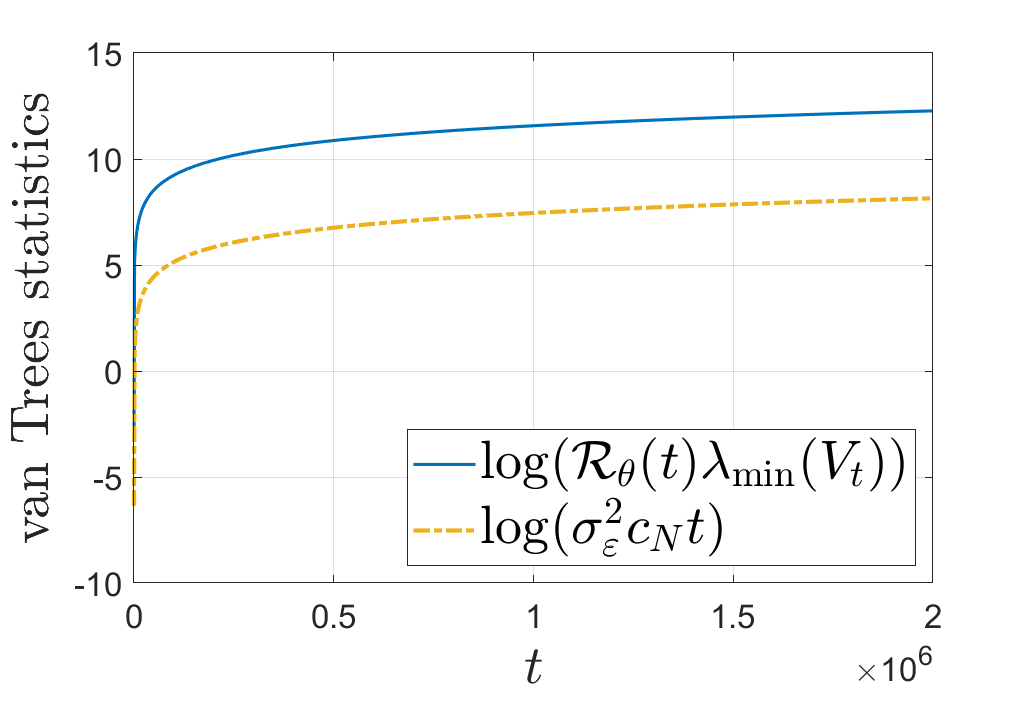}
        \caption{Verifying \eqref{eq:er.ei.t}}
        \label{fig:van.tree}
    \end{subfigure}
    \caption{The performance of BayesTS, given in \cite{linear_thompson_sampling_revisited}, on $\Acal^{10}$ with  $\sigma_\varepsilon^{2} = 0.1$ and $\sigma_{\theta^\star}^{2} = 0.01$, over 20 runs. (a) plots the progress of regret, (b) plots the progress of inference equality, measured via $\lambda_{\min}(V_t)$, (c) plots the log-ratio of regret and inference quality with time, and (d) shows the validity of \eqref{eq:er.ei.t}.}
    \label{fig:TS.inf.reg}
\end{figure}

We simulated the Bayesian Thompson Sampling (BayesTS) as given in  \cite{linear_thompson_sampling_revisited}, which can be produced from Algorithm \ref{alg:thompson_sampling} by changing step \ref{step_beta} to  $\beta_t' \gets  1$, for $T=10^5$ in $n=2$ dimensions, reported over 20 runs. The results are included in Figure \ref{fig:TS.inf.reg}. Indeed, the plot of $\log(\lambda_{\min} (V_t))/{\log (t+1)}$ in Figure \ref{fig:log.stats}\footnote{We plot the mean with standard deviations, but the deviations are too small to notice.} suggests  $e_I \approx 0.37$, where a similar behavior has been noted in \cite{rich_action_spaces}. However, a close inspection of the log-log plot of $\lambda_{\min} (V_t)$ in Figure \ref{fig:inference.quality} reveals that the slope is not quite constant, but shows a slow but steady increase. To estimate $e_I$ at higher $t$'s, we regressed $\log(\lambda_{\min} (V_t))$ against $\log t$ over the last $\frac{1}{100}$-fraction of the iterates to obtain $e_I = 0.49$. Thus, reading off  $\log(\lambda_{\min} (V_t))/{\log (t+1)}$ too early  from a graph such as Figure \ref{fig:log.stats} can be misleading about the asymptotic rate of inference quality. Regressing regret similarly yields $e_R \sim 0.52$, that is close to the terminal value of $\log(\mathscr{R}_{\theta^{\star}}(t))/\log(t + 1) = 0.48$. This match is expected, given that the slope of the regret with time on the log-log plot in Figure seems virtually constant in Figure \ref{fig:regret.performance}. Regret starts showing its asymptotic behavior earlier than $\lambda_{\min}(V_t)$.
With the estimation, we indeed verify $e_R + e_I \approx 1.01 > 1$, as conjectured by our derivation. With the sample set $\hat{\Theta}$ and a collection of sub-sampled time points $\Tcal$, we estimated,
\begin{align}
    \lambda_C \approx \min_{\theta_\star \in \hat{\Theta}} \lambda_1(\nabla a^\star(\theta^\star)), 
    \;
    \cos(\alpha_\Acal) \approx \min_{\theta_\star \in \hat{\Theta}}  \mathring{a}^\star(\theta_\star)^\top \mathring{\theta}^\star,
    \;
    c_3 \approx \min_{\theta_\star \in \hat{\Theta}} \min_{t\in\Tcal} \frac{\mathscr{R}_{\theta^\star}(t)}{\lambda_{\min}(V_t)},
\end{align}
that gave $c_N = \lambda_C^2 \cos^2(\alpha_\Acal) c_3/n \approx 0.01$. With these values, Figure \ref{fig:van.tree} demonstrates the validity of \eqref{eq:er.ei.t}. This inequality is predicated on $\lambda_{\min}(\E^{\pi}[V_t]) \gg \sigma_\varepsilon^2/\sigma_{\theta^\star}^2 = 10$, which holds for $\lambda_{\min}(V_t)$ beyond $t\sim10^3$. This analysis testifies that even though our claim about the product of inference quality and regret holds \emph{on average} in a Bayesian setting, concentrated priors often allow a sharper analysis.

Next, we dig deeper into the increase in the slope of $\log(\lambda_{\min}(V_t))$ over time in Figure \ref{fig:inference.quality}. Treating inference as \emph{currency} for better control, this increase in $\log (\lambda_{\min}(V_t))$ admits an interesting interpretation. In the short-run, $\Acal^{10}$ with $\theta^\star=(1,1)$ offers low-cost control efficacy. Cost leniency disappears, however, with stringent  demands on regret. In a sense, the return-on-investment relationship changes for higher returns. In what follows, we provide a geometric explanation behind this phenomenon.

\begin{figure}[!ht]
    \centering
    \includegraphics[width=0.7\linewidth]{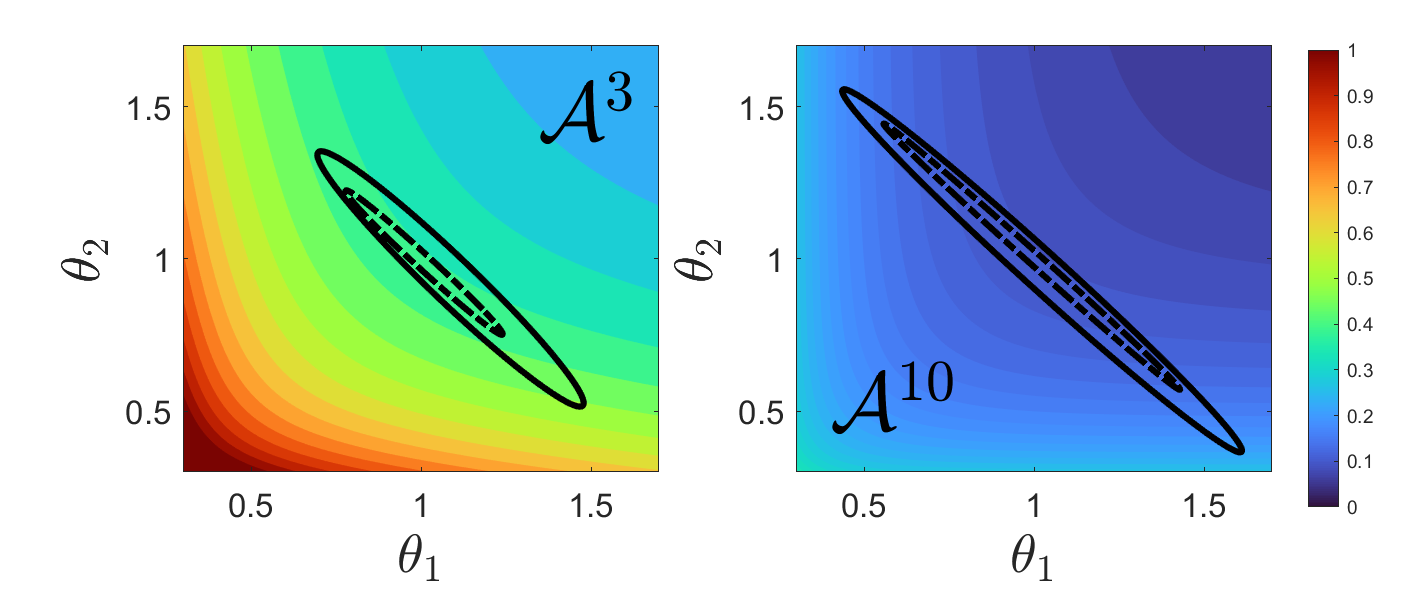}
    \caption{Confidence ellipsoids, defined at Lemma \ref{lemma:probability_Et} with $a_{\max} = \sqrt{2}$, for the progress of BayesTS with $\theta^\star = (1,1)$ at (\protect\solidblackline)$T_1=10^3$ iterations and (\protect\dashedblackline)$T_2=10^4$ iterations are overlaid on the heatmap of  
    $\lambda_1(\nabla a^\star(\theta^\star))$.}
    \label{fig:nabla.a}
\end{figure}

The heatmaps of $\lambda_1(\nabla a^\star(\theta^\star))$ for $\Acal^3$ have much larger values than that for $\Acal^{10}$, as evidenced by Figure \ref{fig:nabla.a}. Here, $\lambda_1$ computes the only nonzero eigenvalue of $\nabla a^\star(\theta^\star)$. This quantity captures how well changes in $\theta^\star$ affect the reward-maximizing action. The smaller it is, the more insensitive the best action becomes to the variation in $\theta^\star$. For two actions $a_1 = a^\star(\theta_1^\star)$ and $a_2 = a^\star(\theta_2^\star)$ that are close on the surface of $\Acal$, we expect the distance of $\theta_1^\star$ and $\theta_2^\star$ to be larger when $\lambda_1(\nabla a^\star(\theta^\star))$ is smaller. As a result, one expects a larger area of $\theta^\star$ to remain \emph{consistent} with observations upon playing actions in $\A^\star$ near the optimal action $a^\star(\theta^\star)$. Confidence ellipsoids, obtained via RLS estimation, encode this consistence. Indeed, these ellipsoids in Figure \ref{fig:nabla.a} are larger for $\Acal^{10}$ than they are for $\Acal^3$ at different points in time. More precisely, the major axis of the ellipsoid is larger for $\Acal^{10}$ than that for $\Acal^{3}$, and vice versa for the minor axis. To see why, consider starting the algorithm from the same sized ellipsoids for both $\Acal^{3}$ and $\Acal^{10}$. Given the discussion above, $a^\star(\theta^\star)$ for all $\theta^\star$ in these ellipsoids is a smaller set of points on the surface of $\Acal^{10}$ than they are on $\Acal^{3}$. As a result, algorithms that exploit RLS estimates such as BayesTS will result in more of the actions $a_t$'s aligned with $a^\star(\theta^\star)$ earlier in $\Acal^{10}$ than in $\Acal^{3}$. 
Since $V_t = \lambda I + \sum_{s = 1}^t a_t a_t^\top$, the minor axis will start aligning with the direction of $a^\star(\theta^\star)$ faster for $\Acal^{10}$ than for $\Acal^{3}$. However, this also means that these actions close to $a^\star(\theta^\star)$ will glean information much more slowly along the major axis. Given our earlier discussion, a larger set of $\theta^\star$'s remains consistent with observations and actions close to $a^\star(\theta^\star)$ for $\Acal^{10}$ than for $\Acal^3$ and this happens along the major axis--a direction in which such actions glean less information in. This explains the difference in shapes of the confidence ellipsoids in $\Acal^3$ and $\Acal^{10}$, and the difference increases in the short-run. The size of the major axis being $\lambda_{\min}(V_t)$, this difference leads to lower qualities of inference for $\Acal^{10}$ than for $\Acal^3$. Since actions are aligning faster with the optimum on $\Acal^{10}$ than on $\Acal^3$, regret drops faster on $\Acal^{10}$ than $\Acal^3$. These two facts together give rise to an illusion of good performance with poor inference in the short-run.

\begin{figure}[!ht]
    \centering
    \includegraphics[width=0.3\linewidth]{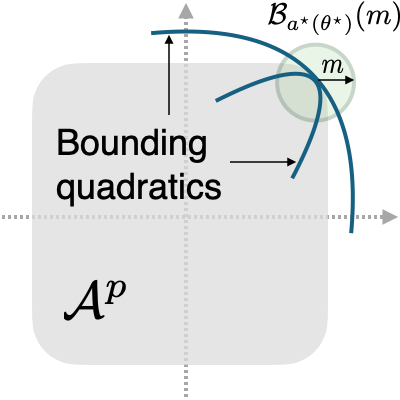}
    \caption{The geometry of $\Acal^p$ and how it locally behaves similar to a convex set that resembles the sub-level set of a strongly convex function.}
    \label{fig:Ap_ball}
\end{figure}
Why does this relationship change in the long-run? For a control policy  with sublinear regret, control actions must eventually align enough with the optimal action. These actions remain within a small ball of the optimal action, portrayed as $\Bcal_{a^\star(\theta^\star)}(m)$ in Figure \ref{fig:Ap_ball}. As long as the function describing the action set is \emph{locally strongly convex}, within this ball, our geometric analysis of Section \ref{sec:geometric_properties} applies. Consequently, the superlinear lower bound on the product of inference quality and cumulative regret on average from \eqref{eq:er.ei.t} in Section \ref{sec:lowerBoundRegret} precludes bettering regret without paying the cost of good inference in the long-run.

\section{Conclusions and Future Work}
\label{sec:conc}
Linear bandits define a well-studied class of decision problems, where the stochastic reward is linear in the action that can vary over a continuous set. Opposed to the multi-arm counterpart, this problem offers the ability to garner information about the unknown governing parameter even with sub-optimal ``arms'' or actions, if done correctly. In this paper, we studied a forced exploration algorithm for linear bandits with convex compact action sets. The algorithm adds a noise along the tangent plane of the reward-maximizing action based on the current point estimate of the unknown parameter and then projects it back to the action set. Our analysis guaranteed fast  estimation of the unknown parameter with high probability that translated to an $\Ocal(\sqrt{T}\log(T))$ expected regret in high probability. Further analysis of the problem domain revealed that this upper bound was tight up to $\log$ factors. In particular, we provided an $\Omega(\sqrt{T})$ lower bound for linear bandits with a wide variety of action sets and noise processes. The lower bound analysis also revealed interesting connections between optimal rates for inference and control in linear bandits. We dug deep into examples in the literature where regret-optimal performance was achieved with low inference quality. We provided a geometric  argument to establish that such phenomenon is essentially an illusion that is short-lived. The asymptotic regime ties them to the same rate for minimax order-optimal regret.

Adaptive control literature on linear bandits is extensive. Perhaps the most well-studied algorithms include UCB and TS. In future work, we want to understand the commonalities among these algorithms and forced exploration with an aim to unify their analyses. A second line of future work includes extension of our work to adaptive control of dynamical systems, such as fully or partially-observed linear time-invariant systems and Markov decision processes. Furthermore, we wish to extend the forced exploration setting and the lower-bound analysis to varying noise processes as in \cite{log_minmax_regret}. Finally, of interest to us, are applications of our methodology to cooperative multi-agent learning contexts in unknown environments.

\section*{Acknowledgments}
The work was partially supported by grants from the National Science Foundation, specifically NSF CPS 2038775 and NSF ECCS 23-49418. 
{The authors would like to thank Nikolai Matni from the University of Pennsylvania and John R. Birge from the University of Chicago for insightful discussions.}

\bibliography{adaptive}

\begin{thebibliography}{40}
\providecommand{\natexlab}[1]{#1}
\providecommand{\url}[1]{\texttt{#1}}
\expandafter\ifx\csname urlstyle\endcsname\relax
  \providecommand{\doi}[1]{doi: #1}\else
  \providecommand{\doi}{doi: \begingroup \urlstyle{rm}\Url}\fi

\bibitem[Abbasi-Yadkori et~al.(2009)Abbasi-Yadkori, Antos, and Szepesv{\'a}ri]{AbbasiYadkori2009ForcedExplorationBA}
Yasin Abbasi-Yadkori, Andr{\'a}s Antos, and Csaba Szepesv{\'a}ri.
\newblock Forced-exploration based algorithms for playing in stochastic linear bandits.
\newblock In \emph{COLT Workshop on On-line Learning with Limited Feedback}, volume~92, page 236, 2009.

\bibitem[Abbasi-yadkori et~al.(2011)Abbasi-yadkori, P\'{a}l, and Szepesv\'{a}ri]{improved_algorithms_for_stochastic_bandits}
Yasin Abbasi-yadkori, D\'{a}vid P\'{a}l, and Csaba Szepesv\'{a}ri.
\newblock Improved algorithms for linear stochastic bandits.
\newblock In J.~Shawe-Taylor, R.~Zemel, P.~Bartlett, F.~Pereira, and K.Q. Weinberger, editors, \emph{Advances in Neural Information Processing Systems}, volume~24. Curran Associates, Inc., 2011.
\newblock URL \url{https://proceedings.neurips.cc/paper_files/paper/2011/file/e1d5be1c7f2f456670de3d53c7b54f4a-Paper.pdf}.

\bibitem[Abeille and Lazaric(2017)]{linear_thompson_sampling_revisited}
Marc Abeille and Alessandro Lazaric.
\newblock Linear {T}hompson sampling revisited.
\newblock In Aarti Singh and Jerry Zhu, editors, \emph{Proceedings of the 20th International Conference on Artificial Intelligence and Statistics}, volume~54 of \emph{Proceedings of Machine Learning Research}, pages 176--184. PMLR, 20--22 Apr 2017.
\newblock URL \url{https://proceedings.mlr.press/v54/abeille17a.html}.

\bibitem[Agrawal and Goyal(2013)]{thompson_sampling_for_contextual_bandits_with_linear_payoffs}
Shipra Agrawal and Navin Goyal.
\newblock Thompson sampling for contextual bandits with linear payoffs.
\newblock In \emph{Proceedings of the 30th International Conference on Machine Learning (ICML)}, pages 127--135. JMLR.org, 2013.

\bibitem[Banerjee and Gopalan(2023)]{regret_for_all_balls}
Debangshu Banerjee and Aditya Gopalan.
\newblock On the minimax regret for linear bandits in a wide variety of action spaces.
\newblock \emph{ArXiv}, abs/2301.03597, 2023.
\newblock URL \url{https://api.semanticscholar.org/CorpusID:255569807}.

\bibitem[Banerjee et~al.(2023)Banerjee, Ghosh, Ray~Chowdhury, and Gopalan]{rich_action_spaces}
Debangshu Banerjee, Avishek Ghosh, Sayak Ray~Chowdhury, and Aditya Gopalan.
\newblock Exploration in linear bandits with rich action sets and its implications for inference.
\newblock In Francisco Ruiz, Jennifer Dy, and Jan-Willem van~de Meent, editors, \emph{Proceedings of The 26th International Conference on Artificial Intelligence and Statistics}, volume 206 of \emph{Proceedings of Machine Learning Research}, pages 8233--8262. PMLR, 25--27 Apr 2023.
\newblock URL \url{https://proceedings.mlr.press/v206/banerjee23b.html}.

\bibitem[Chu et~al.(2011)Chu, Li, Reyzin, and Schapire]{contextual_bandits_with_ucb}
Wei Chu, Lihong Li, Lev Reyzin, and Robert~E. Schapire.
\newblock Contextual bandits with linear payoff functions.
\newblock In \emph{Proceedings of the 14th International Conference on Artificial Intelligence and Statistics (AISTATS)}, volume~15, pages 208--214. JMLR.org, 2011.

\bibitem[Deshpande and Montanari(2012)]{recomendation_systems}
Yash Deshpande and Andrea Montanari.
\newblock Linear bandits in high dimension and recommendation systems.
\newblock In \emph{2012 50th Annual Allerton Conference on Communication, Control, and Computing (Allerton)}, pages 1750--1754, 2012.
\newblock \doi{10.1109/Allerton.2012.6483433}.

\bibitem[Gill and Levit(1995)]{GillLevit1995}
Richard~D. Gill and Boris~Y. Levit.
\newblock Applications of the van {T}rees inequality: A {B}ayesian {C}ram\'{e}r-{R}ao bound.
\newblock \emph{Bernoulli}, 1\penalty0 (1/2):\penalty0 59--79, Mar. - Jun. 1995.

\bibitem[Golub(1973)]{golub1973some}
Gene~H Golub.
\newblock Some modified matrix eigenvalue problems.
\newblock \emph{SIAM review}, 15\penalty0 (2):\penalty0 318--334, 1973.

\bibitem[Greenewald et~al.(2017)Greenewald, Tewari, Murphy, and Klasnja]{ph}
Kristjan~H. Greenewald, Ambuj Tewari, Susan~A. Murphy, and Predrag~V. Klasnja.
\newblock Action centered contextual bandits.
\newblock \emph{Advances in neural information processing systems}, 30:\penalty0 5973--5981, 2017.
\newblock URL \url{https://api.semanticscholar.org/CorpusID:29904027}.

\bibitem[Harrison et~al.(2012)Harrison, Keskin, and Zeevi]{incomplete_learning_example}
J.~Michael Harrison, N.~Bora Keskin, and Assaf Zeevi.
\newblock Bayesian dynamic pricing policies: Learning and earning under a binary prior distribution.
\newblock \emph{Management Science}, 58\penalty0 (3):\penalty0 570--586, 2012.
\newblock URL \url{https://ssrn.com/abstract=2389764}.

\bibitem[He et~al.(2022)He, Zhang, and Zhang]{estimation_regret}
Jiahao He, Jiheng Zhang, and Rachel~Q. Zhang.
\newblock A reduction from linear contextual bandit lower bounds to estimation lower bounds.
\newblock In Kamalika Chaudhuri, Stefanie Jegelka, Le~Song, Csaba Szepesvari, Gang Niu, and Sivan Sabato, editors, \emph{Proceedings of the 39th International Conference on Machine Learning}, volume 162 of \emph{Proceedings of Machine Learning Research}, pages 8660--8677. PMLR, 17--23 Jul 2022.
\newblock URL \url{https://proceedings.mlr.press/v162/he22e.html}.

\bibitem[Horn and Johnson(2012)]{horn2012matrix}
Roger~A Horn and Charles~R Johnson.
\newblock \emph{Matrix analysis}.
\newblock Cambridge university press, 2012.

\bibitem[Huang et~al.(2017)Huang, Lattimore, Gy{{\"o}}rgy, and Szepesv{{\'a}}ri]{JMLR:v18:17-079}
Ruitong Huang, Tor Lattimore, Andr{{\'a}}s Gy{{\"o}}rgy, and Csaba Szepesv{{\'a}}ri.
\newblock Following the leader and fast rates in online linear prediction: Curved constraint sets and other regularities.
\newblock \emph{Journal of Machine Learning Research}, 18\penalty0 (145):\penalty0 1--31, 2017.
\newblock URL \url{http://jmlr.org/papers/v18/17-079.html}.

\bibitem[Kato(1980)]{kato1980perturbation}
Tosio Kato.
\newblock \emph{Perturbation Theory for Linear Operators}, volume 132 of \emph{Grundlehren der mathematischen Wissenschaften}.
\newblock Springer-Verlag, Berlin, 1980.
\newblock ISBN 978-3-540-07558-6.

\bibitem[Kaufmann et~al.(2021)Kaufmann, Koolen, and Garivier]{cooperative_MAMAB_ucb}
Emilie Kaufmann, Wouter~M Koolen, and Aurelien Garivier.
\newblock Cooperative multi-agent bandits with heavy tails.
\newblock \emph{IEEE Transactions on Information Theory}, 67\penalty0 (3):\penalty0 1860--1875, 2021.
\newblock \doi{10.1109/TIT.2020.3035751}.

\bibitem[Kerdreux et~al.(2021)Kerdreux, d'Aspremont, and Pokutta]{projection_free_optimization_on_uniformly_convex_sets}
Thomas Kerdreux, Alexandre d'Aspremont, and Sebastian Pokutta.
\newblock Projection-free optimization on uniformly convex sets.
\newblock In Arindam Banerjee and Kenji Fukumizu, editors, \emph{Proceedings of The 24th International Conference on Artificial Intelligence and Statistics}, volume 130 of \emph{Proceedings of Machine Learning Research}, pages 19--27. PMLR, 13--15 Apr 2021.
\newblock URL \url{https://proceedings.mlr.press/v130/kerdreux21a.html}.

\bibitem[Keskin and Zeevi(2014)]{keskin2014dynamic}
N.~Bora Keskin and Assaf Zeevi.
\newblock Dynamic pricing with an unknown demand model: Asymptotically optimal semi-myopic policies.
\newblock \emph{Operations Research}, 62\penalty0 (5):\penalty0 1142--1167, 2014.
\newblock \doi{10.1287/opre.2014.1302}.
\newblock Columbia Business School Research Paper No. 14-30, Available at SSRN: \url{https://ssrn.com/abstract=2389721} or \url{http://dx.doi.org/10.2139/ssrn.2389721}.

\bibitem[Kirschner et~al.(2020)Kirschner, Lattimore, Filippi, and Seeger]{IDS_for_linear_partial_monitoring}
Johannes Kirschner, Tor Lattimore, Sarah Filippi, and Matthias Seeger.
\newblock Information directed sampling for linear partial monitoring.
\newblock In \emph{Advances in Neural Information Processing Systems (NeurIPS)}, pages 1--12. Curran Associates, Inc., 2020.

\bibitem[Landgren et~al.(2016)Landgren, Srivastava, and Leonard]{MAMAB_ucb}
Pontus Landgren, Vivek Srivastava, and Naomi~Ehrich Leonard.
\newblock Multi-agent multi-armed bandits with limited communication.
\newblock In \emph{Proceedings of the 55th IEEE Conference on Decision and Control (CDC)}, pages 1792--1797. IEEE, 2016.

\bibitem[Lattimore and Szepesvári(2020)]{bandit_algorithms_book}
Tor Lattimore and Csaba Szepesvári.
\newblock \emph{Bandit Algorithms}.
\newblock Cambridge University Press, 2020.

\bibitem[Li et~al.(2010)Li, Chu, Langford, and Schapire]{context_bandits_ucb}
Lihong Li, Wei Chu, John Langford, and Robert~E. Schapire.
\newblock A contextual-bandit approach to personalized news article recommendation.
\newblock In \emph{Proceedings of the 19th International Conference on World Wide Web}, WWW '10, page 661–670, New York, NY, USA, 2010. Association for Computing Machinery.
\newblock ISBN 9781605587998.
\newblock \doi{10.1145/1772690.1772758}.
\newblock URL \url{https://doi.org/10.1145/1772690.1772758}.

\bibitem[Lumbreras and Tomamichel(2024)]{log_minmax_regret}
Josep Lumbreras and Marco Tomamichel.
\newblock Linear bandits with polylogarithmic minimax regret.
\newblock In Shipra Agrawal and Aaron Roth, editors, \emph{Proceedings of Thirty Seventh Conference on Learning Theory}, volume 247 of \emph{Proceedings of Machine Learning Research}, pages 3644--3682. PMLR, 30 Jun--03 Jul 2024.
\newblock URL \url{https://proceedings.mlr.press/v247/lumbreras24a.html}.

\bibitem[Martinez-Rubio et~al.(2019)Martinez-Rubio, Kanade, and Valko]{distributed_cooperative_decision_making_ucb}
Daniel Martinez-Rubio, Varun Kanade, and Michal Valko.
\newblock Distributed cooperative decision-making in multiarmed bandits: Frequentist and bayesian algorithms.
\newblock In \emph{2019 IEEE 58th Conference on Decision and Control (CDC)}, pages 6142--6149. IEEE, 2019.
\newblock \doi{10.1109/CDC40024.2019.9029256}.

\bibitem[Martinez-Rubio et~al.(2021)Martinez-Rubio, Jang, Wang, Koppel, and Ribeiro]{MYS_wind_farm}
David Martinez-Rubio, Taein Jang, Zaiwei Wang, Alec Koppel, and Alejandro Ribeiro.
\newblock Multi-agent thompson sampling for bandit applications with sparse neighbourhood structures.
\newblock In \emph{Proceedings of the 38th International Conference on Machine Learning (ICML)}, pages 7647--7656. PMLR, 2021.

\bibitem[Molinaro(2022)]{curvature_offline_online_optimization}
Marco Molinaro.
\newblock Curvature of feasible sets in offline and online optimization.
\newblock \emph{Mathematics of Operations Research}, abs/2002.03213, February 2022.
\newblock URL \url{https://www.microsoft.com/en-us/research/publication/curvature-of-feasible-sets-in-offline-and-online-optimization/}.

\bibitem[Mussi et~al.(2020)Mussi, Metelli, and Restelli]{online_advertising}
Marco Mussi, Alberto~Maria Metelli, and Marcello Restelli.
\newblock Dynamical linear bandits.
\newblock In \emph{Proceedings of the 37th International Conference on Machine Learning}, pages 7156--7166. PMLR, 2020.

\bibitem[Pope(2008)]{projection_onto_ellipsoids}
Stephen~B. Pope.
\newblock Algorithms for ellipsoids.
\newblock Technical Report FDA-08-01, Sibley School of Mechanical \& Aerospace Engineering, Cornell University, Ithaca, New York, February 2008.

\bibitem[Rusmevichientong and Tsitsiklis(2010)]{rusmevichientong2010linearlyparameterizedbandits}
Paat Rusmevichientong and John~N. Tsitsiklis.
\newblock Linearly parameterized bandits.
\newblock \emph{Mathematics of Operations Research}, 35\penalty0 (2):\penalty0 395--411, 2010.
\newblock ISSN 0364-765X.
\newblock \doi{10.1287/moor.1100.0446}.

\bibitem[Russo and Van~Roy(2014{\natexlab{a}})]{learning_to_optimize_with_IDS}
Daniel Russo and Benjamin Van~Roy.
\newblock Learning to optimize via information-directed sampling.
\newblock In \emph{Proceedings of the 27th Annual Conference on Learning Theory (COLT)}, pages 1--38. JMLR.org, 2014{\natexlab{a}}.

\bibitem[Russo and Van~Roy(2014{\natexlab{b}})]{russo2014learning}
Daniel Russo and Benjamin Van~Roy.
\newblock Learning to optimize via posterior sampling.
\newblock \emph{Mathematics of Operations Research}, 39\penalty0 (4):\penalty0 1221--1243, 2014{\natexlab{b}}.
\newblock ISSN 0364-765X.
\newblock \doi{10.1287/moor.2014.0650}.

\bibitem[Russo and Van~Roy(2016)]{an_information_theoretic_analysis_of_TS}
Daniel Russo and Benjamin Van~Roy.
\newblock An information-theoretic analysis of thompson sampling.
\newblock \emph{Journal of Machine Learning Research (JMLR)}, 17\penalty0 (1):\penalty0 2442--2471, 2016.

\bibitem[Schubert(2021)]{schubert2021triangle}
Erich Schubert.
\newblock A triangle inequality for cosine similarity.
\newblock In \emph{International Conference on Similarity Search and Applications}, pages 32--44. Springer, 2021.

\bibitem[Stern and Birge(2020)]{stern2020dynamic}
Matt Stern and John~R Birge.
\newblock Dynamic learning in strategic pricing games.
\newblock \emph{Available at SSRN 3579123}, 2020.

\bibitem[Tropp(2011)]{tropp}
Joel Tropp.
\newblock {Freedman's inequality for matrix martingales}.
\newblock \emph{Electronic Communications in Probability}, 16\penalty0 (none):\penalty0 262 -- 270, 2011.
\newblock \doi{10.1214/ECP.v16-1624}.
\newblock URL \url{https://doi.org/10.1214/ECP.v16-1624}.

\bibitem[Vershynin(2018)]{vershynin2018high}
Roman Vershynin.
\newblock \emph{High-dimensional probability: An introduction with applications in data science}, volume~47.
\newblock Cambridge university press, 2018.

\bibitem[Zhu and Modiano(2018)]{network_routing}
Ruihao Zhu and Eytan~H. Modiano.
\newblock Learning to route efficiently with end-to-end feedback: The value of networked structure.
\newblock \emph{ArXiv}, abs/1810.10637, 2018.
\newblock URL \url{https://api.semanticscholar.org/CorpusID:53030121}.

\bibitem[Ziemann and Sandberg(2021)]{ziemann2021uninformative}
Ingvar Ziemann and Henrik Sandberg.
\newblock On uninformative optimal policies in adaptive lqr with unknown b-matrix.
\newblock In \emph{Learning for Dynamics and Control}, pages 213--226. PMLR, 2021.

\bibitem[Ziemann and Sandberg(2024)]{new_lqr_result_arxiv}
Ingvar Ziemann and Henrik Sandberg.
\newblock Regret lower bounds for learning linear quadratic gaussian systems, 2024.
\newblock URL \url{https://arxiv.org/abs/2201.01680}.

\end{thebibliography}

\appendix
\section{Algorithms Compared in Section \ref{sec:numerics}}
\label{sec:appendix_pseudocodes}

The pseudocodes for LinUCB, TS\footnote{We adjust the factor in $\beta_t'$ with $a_{\max}^2$ to incorporate action sets with $a_{\max} \geq 1$. }, and FEL are included here.

\begin{algorithm}
\caption{LinUCB Algorithm, adapted from \cite{improved_algorithms_for_stochastic_bandits, context_bandits_ucb}}
\label{alg:action_space_sampling}
\begin{algorithmic}[1]
\State \textbf{Data:} $\hat{\theta}_1$, T, $\lambda^{\textrm{UCB}}$
\State  $V_0\gets\lambda^{\textrm{UCB}} I_n$  

\For{$t = 1, \dots, T$}
      
    \State  $\beta_t \gets M \sqrt{n \log(T (1 + t a_{\max}^2/\lambda^{\textrm{UCB}}))} + \sqrt{\lambda^{\textrm{UCB}}}\theta_{\max}$
    \State $a_t\gets\arg \max_{a \in \Acal} \, a^\top \hat{\theta}_t +  \beta_t \vnorm{a}_{V_{t-1}^{-1}}$
    \State Observe  $Y_t$ 
    \State  $V_t\gets V_{t-1} + a_t a_t^\top$
    \State  $\hat{\theta}_{t + 1}\gets V_{t}^{-1} \sum_{s = 1}^t a_s Y_{s}$
\EndFor

\end{algorithmic}
\end{algorithm}

\begin{algorithm}
\caption{TS Algorithm, adapted from \cite{linear_thompson_sampling_revisited}}
\label{alg:thompson_sampling}
\begin{algorithmic}[1]
\State \textbf{Data:} $\hat{\theta}_1$, T, $\lambda^{\textrm{TS}}$
\State $V_0 \gets \lambda^{TS} I$
\For{$t = 1, \dots, T$}
    \State $\beta_t' \gets M \sqrt{2 \log(4T^2 (1 + a_{\max}^2 t /\lambda^{\textrm{TS}})^{n/2})} + \sqrt{\lambda^{\textrm{TS}}}\theta_{\max}$ \label{step_beta}
    \State $\tilde{\theta}_t\gets\hat{\theta}_t +  \beta_t' V_{t-1}^{-1/2} \eta_t$, where $\eta_t \sim \Ncal(0, I_n)$
    \State $ a_t \gets  \arg \max_{a\in \Acal} a^\top \tilde{\theta}_t$
    \State Observe $Y_{t}$
    \State $V_t\gets V_{t-1} + a_t a_t^\top$
    \State $\hat{\theta}_{t+1} \gets V_{t}^{-1} \sum_{s = 1}^{t} a_s Y_{s}$

\EndFor
\end{algorithmic}
\end{algorithm}

\begin{algorithm}
\caption{FEL Algorithm, adapted from \cite{AbbasiYadkori2009ForcedExplorationBA}}
\label{alg:fel_algorithm}
\begin{algorithmic}[1]
\State \textbf{Data:} $\hat{\theta}_1$, T
\State $V_0 \gets I$, $f_0 \gets 0$
\State Find an eigenbasis $\{e_i\}_{i = 1}^n$ for $\Delta_g$ with associated eigenvalues $\{\lambda_i\}_{i = 1}^n$.
\For{$t = 1, \ldots, T$}
    \If{$f_{t-1} < c^{\textrm{FEL}} n \sqrt{t}$} \text{: Exploration phase}
        \State  $a_t \gets e_i/\sqrt{\lambda_i}$, where $i \sim \operatorname{Uniform}\{1,...,n\}$
        \State  Observe $Y_t$
        \State $V_t \gets V_{t-1} + a_t a_t^\top$
        \State $\hat{\theta}_{t + 1} \gets (V_t)^{-1} \sum_{s = 1}^{t} a_s Y_{s}$
        \State $f_t \gets f_{t-1} + 1$
    \Else \text{: Exploitation phase}
        \State $a_t \gets a^\star(\hat{\theta}_t)$
        \State Observe $Y_t$
        \State Keep previous values: $V_t \gets V_{t-1}$, $\hat{\theta}_{t+1} \gets \hat{\theta}_{t}$, $f_t \gets f_{t-1}$
    \EndIf
\EndFor
\end{algorithmic}
\end{algorithm}
\newpage
\section{High-Probability Guarantees for the FEL Algorithm}\label{sec:FEL_hp_proof}
Our matrix-martingale based analysis yields high-probability inference and regret guarantees for the FEL algorithm in \cite{AbbasiYadkori2009ForcedExplorationBA}.
For ellipsoidal action sets, the FEL algorithm is given in Appendix \ref{sec:appendix_pseudocodes}; see the original paper for general action sets. The following derivation illustrates the power of our proof technique.

The FEL algorithm updates $V_t$ only when it explores and keeps it constant when it exploits. Let $\Tcal^\textrm{exp}$ be the collection of times up to time $T$ when it explores. The design of the algorithm is such that $|\{ s \leq t | s \in \Tcal^\textrm{exp} \}| \sim \Ocal(\sqrt{t})$.
For $t \in \Tcal^\textrm{exp}$, $a_t$ is sampled from a pre-determined distribution  such that $\lambda_{\min}(\E_{t-1}[a_t a_t^\top]) \geq \lambda_\textrm{exp} > 0$. Thus, we have
\begin{align}
    \lambda_{\min}(\E[V_t]) \geq \sum_{s = 1}^t \lambda_{\min}(\E_{s-1}[a_s a_s^\top]) \mathds{1}_{s \in \Tcal^\textrm{exp}} \gtrsim \lambda_{\textrm{exp}} \sqrt{t}.
    \label{eq:E.Vt.FEL}
\end{align}
The per-step matrix variation satisfies
\begin{align}
    \vnorm{\Ex{\left(a_s a_s^\top - \Ex{a_s a_s^\top}\right)^2} } 
    \begin{cases}
        \leq 16 a_{\max}^4, & {s \in \Tcal^\textrm{exp}},
        \\
        = 0, & \text{otherwise}.
    \end{cases}
    \label{eq:var.Vt.FEL}
\end{align}
Collectively, \eqref{eq:E.Vt.FEL}--\eqref{eq:var.Vt.FEL} yield the growth rates in Lemmas \ref{lemma_lower_bound_for_variation} and \ref{lemma_upper_bound_for_variation}, implying that the analysis of Lemma \ref{lemma:probability_Ct} carries over for the FEL algorithm. 

To obtain the high probability regret bound, we modify the proof of Theorem \ref{theorem_regret} as follows. Recall that the proof starts by showing \eqref{rewriting_the_action}, restated below. 
    \begin{align}
    \begin{aligned}
    \vnorm{a_{t } - a^\star(\theta^\star)}^2
    & \leq 2\vnorm{a^\star(\hat{\theta}_t) - a^\star(\theta^\star)}^2 + 2\vnorm{\sum_{i = 2}^n\nu_t^i \mu_t^i}^2.
    \end{aligned}
    \label{rewriting_the_action.replica}
    \end{align}
The left-hand-side of the above inequality is the per-step regret of TRAiL. The first summand on the right-hand side captures the result of playing the reward-maximizing action $a^\star(\hat{\theta}_t)$ and the second term is the result of the perturbations of TRAiL. In the proof of Theorem \ref{theorem_regret}, we bounded the first summand on the right hand side obtained from $V_t$'s that satisfy the inference guarantee, and hence, applies to the FEL algorithm. Thus, collectively all the exploitation steps and the $\Ocal(\sqrt{T})$ exploration steps incur at most $\Ocal(\sqrt{T})$ expected regret growth with high probability.

\end{document}